\documentclass[a4paper,final,3p,reqno]{elsarticlem}

\usepackage{amssymb}
\usepackage{amsmath,amsthm}
\usepackage{enumerate}
\usepackage{comment}
\usepackage{times}
\usepackage[british]{babel}
\usepackage{mathtools}
\usepackage{amssymb}
\usepackage{comment}
\usepackage{microtype}
\usepackage{natbib}
\usepackage{enumitem}
\usepackage{nicefrac}
\usepackage{hyperref,url}
  \hypersetup{
    bookmarksnumbered
  }
  \hypersetup{
    breaklinks=false,
    pdfborderstyle={/S/U/W 0},
    citebordercolor=.235 .702 .443,
    urlbordercolor=.255 .412 .882,
    linkbordercolor=.804 .149 .149,
  }
  \hypersetup{
    pdfauthor={},
    pdftitle={},
    pdfkeywords={}
  }
\usepackage{tikz}
\usetikzlibrary{matrix}

\newcommand{\pspace}{\varOmega}
\newcommand{\rspace}{\mathcal{X}}
\newcommand{\reals}{\mathbb{R}}

\newcommand{\nats}{\mathbb{N}}

\newcommand{\values}{\mathcal{X}}
\newcommand{\valuesy}{\mathcal{X}_z}
\newcommand{\valuesn}{\mathcal{X}}
\newcommand{\pr}{P}
\newcommand{\lpr}{{\underline{\pr}}}
\newcommand{\upr}{{\overline{\pr}}}
\newcommand{\apr}{Q}

\newcommand{\nex}{E}
\newcommand{\lnex}{{\underline{\nex}}}
\newcommand{\unex}{{\overline{\nex}}}

\newcommand{\gambles}{\mathcal{L}}
\newcommand{\acts}{\mathcal{H}}
\newcommand{\actsy}{\mathcal{H}_z}
\newcommand{\actsn}{\mathcal{H}}

\newcommand{\domain}{\mathcal{K}}

\newcommand{\desirs}{\mathcal{D}}

\newcommand{\rdesirs}{\mathcal{R}}

\newcommand{\rdesirsc}[1]{\mathcal{R}\rfloor#1}
\newcommand{\sdiff}{\mathcal{A}}
\newcommand{\cone}{\mathcal{C}}

\newcommand{\partit}{\mathcal{B}}

\newcommand{\set}[2]{\left\{#1\colon#2\right\}}

\newcommand{\norm}[1]{\lVert#1\rVert}

\newcommand{\solp}{\mathcal{M}}

\renewcommand{\epsilon}{\varepsilon}
\newcommand{\eps}{\epsilon}
\newcommand{\cluc}[1]{\overline{#1}}  

\renewcommand{\iff}{\Leftrightarrow}

\DeclareMathOperator{\posi}{posi}
\DeclareMathOperator{\aff}{aff}

\DeclareMathOperator{\dom}{\vartriangleright}
\DeclareMathOperator{\ndom}{\ntriangleright}

\newtheorem{theorem}{Theorem}
\newtheorem{lemma}[theorem]{Lemma}
\newtheorem{proposition}[theorem]{Proposition}
\newtheorem{corollary}[theorem]{Corollary}
\theoremstyle{remark}
\newtheorem{remark}{Remark}
\newtheorem{example}{Example}
\newtheorem{definition}{Definition}

\begin{document}

\begin{frontmatter}

\title{Desirability and the birth of incomplete preferences}

\author[uni1]{Marco Zaffalon\corref{cor1}}
\ead{zaffalon@idsia.ch}

\cortext[cor1]{Corresponding author}

\author[uni2]{Enrique Miranda}
\ead{mirandaenrique@uniovi.es}

\address[uni1]{Istituto Dalle Molle di Studi sull'Intelligenza Artificiale (IDSIA). Lugano (Switzerland)}

\address[uni2]{University of Oviedo, Dep. of Statistics and Operations Research. Oviedo (Spain)}

\begin{abstract}
We establish an equivalence between two seemingly different
theories: one is the traditional axiomatisation of incomplete
preferences on horse lotteries based on the mixture independence
axiom; the other is the theory of desirable gambles (bounded random
variables) developed in the context of imprecise probability, which
we extend here to make it deal with vector-valued gambles. The
equivalence allows us to revisit incomplete preferences from the
viewpoint, and with the tools, of desirability and through the
derived notion of coherent lower previsions (i.e., lower expectation
functionals). On this basis, we obtain new results and insights: in
particular, we show that the theory of incomplete preferences can be
developed assuming only the existence of a worst act---no best act
is needed---, and that a weakened Archimedean axiom suffices too;
this axiom allows us also to address some controversy about the
regularity assumption (that probabilities should be positive---they
need not), which enables us also to deal with uncountable
possibility spaces;  we show that it is always possible to extend in
a minimal way a preference relation to one with a worst act, and yet
the resulting relation is never Archimedean, except in a
trivial case; we show that the traditional notion of state
independence coincides with the notion called \emph{strong
independence} in imprecise probability (stochastic independence in
the case of complete preferences)---this leads us to give much a
weaker definition of state independence than the traditional one; we
rework and uniform the notions of complete preferences, beliefs,
values; we argue that Archimedeanity does not capture all the problems that can be
modelled with sets of expected utilities and we provide a new notion
that does precisely that. Perhaps most importantly, we argue
throughout that desirability is a powerful and natural setting to
model, and work with, incomplete preferences, even in the case of
non-Archimedean problems. This leads us to suggest that
desirability, rather than preference, should be the primitive notion
at the basis of decision-theoretic axiomatisations.

\end{abstract}


\end{frontmatter}

\section{Introduction}
\subsection*{Strand one} It seems natural to found a theory of rational decision making on the notion of preference; after all, what is deciding other than choosing between alternatives?

This must have been the idea behind the early works on the subject,
starting from von Neumann and Morgenstern's \cite{neumann1947}, to
the analytical framework of  Anscombe and Aumann \cite{anscombe1963}
and to that of Savage \cite{savage1954}. They show that a rational
decision maker, let us call him Thomas,\footnote{A homage to Bayes.} can be regarded as an agent
with beliefs about the world, in the form of probabilities, and
values of consequences, in the form of utilities. And moreover that
Thomas can be regarded as taking decisions by maximising his expected
utility. This view has had a tremendous impact in many fields of
research, not last on Bayesian statistics, which some see to draw
its justification from Savage's work.

Yet, many authors, including von Neumann, Morgenstern and Anscombe
themselves, have soon recognised that it is not realistic to assume
that Thomas can always compare alternatives; in some cases he may just
not know which one to prefer---even only because he lacks
information about them. In this case we talk of \emph{incomplete},
or \emph{partial}, \emph{preferences}. Axiomatisations of rational
decision making with incomplete preferences came much later, though,
through the works of Bewley \cite{bewley1986}, Seidenfeld et al.
\cite{seidenfeld1995}, Nau \cite{nau2006}, and, more recently, Ok et
al. \cite{ok2012}, and Galaabaatar and Karni \cite{galaabaatar2013}.
These works build upon the analytical framework of Anscombe and
Aumann so as to represent rational (or \emph{coherent}) preferences
through sets of expected utilities; the disagreeing decisions these
may lead to account for the incomplete nature of preferences.

The picture that comes out of these works is not entirely clear. The axioms employed are not always the same. This is the case of the continuity axiom, called \emph{Archimedean}, which is necessary to obtain a representation in terms of expected utilities; but it is also the case for the \emph{state-independence} axioms that enable one to decompose the set of expected utilities into separate sets of probabilities and utilities. The cardinality of the spaces involved also changes in different works, and the treatment of infinite spaces turns out to be quite technical while not void of problems.

Moreover, the overall impression is that directly stretching the axioms of Anscombe and Aumann, so as to deal with incomplete preferences, shows some limits, and that clarity risks getting lost in the process.

\subsection*{Strand two}
In parallel, other researchers were getting, through another path, to a theory of \emph{imprecise probability}. It started from de Finetti's interpretation of expectation as a subject's fair price for a \emph{gamble}, that is, a bounded random variable. Let us call this subject Matilda, or Thilda, to stress that she is Thomas' counterpart. De Finetti's next bright move was to deduce probability by imposing a single axiom of self-consistency on Thilda's fair prices for different gambles \cite{finetti1937}. Smith suggested that de Finetti's approach could be extended to account for the case that probabilities are indeterminate or not precisely specified \cite{smith1961}. Williams made Smith's ideas precise, by giving an axiomatisation that is again based on a notion of self-consistency, which is called \emph{coherence} since then \cite{williams1975}.

It is important to stop for a moment on this notion, which will be
also central to the present work. Williams developed his theory
starting from a setup more primitive than de Finetti's. Rather than
asking Thilda to assess her prices for gambles, he only requires
Thilda to state whether a gamble is \emph{desirable} (or
\emph{acceptable}) for her,\footnote{These are also called
\emph{favourable gambles} in \cite{seidenfeld1990}.} in the sense
that she would commit herself to accept whatever reward or loss it
will eventually lead to. The core notion in Williams' theory is then
a set of so-called \emph{desirable gambles}. One such set is said
\emph{coherent} when it satisfies a few axioms of rationality. Lower
and upper expectations, which are called \emph{previsions} in
Williams' theory, and their properties, are derived from the set of
gambles, and are shown to be equivalent to sets of probabilities.
Eventually, we can also recover de Finetti's theory as a special
case from sets of gambles called \emph{maximal}, or \emph{complete}.
But the important point is that coherent sets of desirable gambles
can be conditioned, marginalised, extended to bigger spaces, and so
on, without ever needing to talk about (sets of) probabilities. This
is even more remarkable as coherent sets of desirable gambles are
more expressive than probabilities; for instance, we can condition a
set on an event of zero probability without making any conceptual or
technical issue arise. In fact, sets of probabilities are equivalent
to a special type of desirability, the one made of so-called
\emph{strictly desirable} gambles. We take care to introduce
desirability and coherent lower previsions in a self-contained way,
and as pedagogically as possible in a research paper, in
Section~\ref{sec:intro-clp}. We do so since we are aware that it is
not a formalism that is as well known as that of preference
relations (these are briefly introduced in
Section~\ref{sec:IntroPrefs}).

Williams' fundamental work went largely unnoticed until Walley used it as the basis for his theory of imprecise probability \cite{walley1991}. At the very essence, Walley's theory can be regarded as Williams' theory with an additional axiom to account for conglomerability. \emph{Conglomerability} \cite{finetti1931} is a property that a finitely additive model may or may not satisfy, and that makes a difference when the space of possibilities is infinite. In fact, if the possibility space is finite, William's and Walley's theory essentially coincide.

Walley's theory has been influential, originating a number of further developments as well as specialisations and applications (see \cite{augustin2014} for a recent collection of related works). Most importantly, along the years, the core notion of desirability underlying both Williams' and Walley's theories, has resisted thorough analysis and has proven to be a very solid and general foundation for a behavioural theory of uncertainty. On the other hand, let us note that both Williams' and Walley's theories are developed for the case of a linear utility scale, and the utility itself is assumed precisely specified.

\subsection*{Equivalence}
What stroke us at first is that the same mathematical structure is at the basis of both the axiomatisations of incomplete preferences and the theory of desirability in imprecise probability: that of convex cones.

In the former case, cones are made of scaled differences of horse
lotteries. If we call $\pspace$ the possibility space and $\values$
the set of possible outcomes, or prizes, then a \emph{horse
lottery}, or \emph{act}, $p:\pspace\times\values\rightarrow[0,1]$ is
a nested, or compound, lottery that returns a simple lottery
$p(\omega,\cdot)$ on $\values$ for each possible realisation of
$\omega\in\pspace$. In this paper we take $\values$ to be
finite---whereas $\pspace$ is unconstrained---, so $p(\omega,\cdot)$
is simply a mass function on $\values$. The convex cone associated
with a (strict, as we take it) coherent preference relation $\succ$
between horse lotteries is given by
$\cone\coloneqq\{\lambda(p-q):\lambda>0,p\succ q\}$. Note that
$\cone$ is made of objects that are neither acts nor
preferences---strictly speaking.

On the other hand, a coherent set of desirable gambles turns out to be again a convex cone, this time made of gambles. In the traditional formulation of desirability, a gamble is a bounded function $f$ from the space of possibility $\pspace$ to the real numbers $\reals$; $f(\omega)$ represents the (possibly negative) reward one gets, in a linear utility scale, in case $\omega\in\pspace$ occurs.

Obviously, the two representations are very close. Cannot they just be the same? The answer is yes: in Section~\ref{sec:equiv} we show that desirability cones and cones arising from incomplete preferences are equivalent representations. Two conditions must be met for this to be the case:
\begin{itemize}
\item One is very specific: the preference relation must have a worst act, that is, an act $w$ such that $p\succ w$ for all $p\neq w$. This is almost universally assumed in the literature, together with the presence of a best act. In this paper we need not have the best act, the worst act suffices to develop all the theory. Moreover, we show that we can assume without loss of generality that the worst act $w$ is degenerate on an element $z\in\values$ for all $\omega\in\pspace$. In other words, that set $\values$ has a worst outcome $z$. Whence we denote by $z$ both the worst outcome and the worst act.  Accordingly, we shall use $\valuesy$ to denote a set of prizes with worst outcome and $\valuesn$ a set of prizes without it (or something that is not specified).

\item The second condition is that desirability must be extended to make it deal with vector-valued
gambles. As we shall see in Remark~\ref{rem:interpret1}, this is
surprisingly simple to do. It is enough to define a gamble as a
function $f:\pspace\times\values\rightarrow\reals$. The
interpretation is that once $\omega\in\pspace$ occurs, gamble $f$
returns a vector of rewards, one for each different type of prize in
$\values$. And yet, we eventually, and mathematically, treat the
gamble \emph{as if} the possibility space were the product
$\pspace\times\values$. Therefore the desirability axioms are
unchanged, they are simply applied to gambles defined on
$\pspace\times\values$. All the theory is applied unchanged to
gambles on the product space. The consequence are, however,
stunning.
\end{itemize}

\noindent They follow in particular because then we can go from a problem of incomplete preferences to an equivalent one of desirability by simply dropping $z$ from $\values_z$ and hence from all acts. Immediately, cone $\cone$ becomes in that way a coherent set $\rdesirs$ of desirable gambles on $\pspace\times\values$. And we can go the other way around: we start from a coherent set of desirable gambles on $\pspace\times\values$ and by appending a worst outcome to $\values$, and to all the acts, we create an equivalent cone $\cone$, with the associated coherent preference relation $\succ$. What is striking is that the equivalence holds also for the corresponding notions of Archimedeanity: an Archimedean preference relation originates a coherent set of strictly desirable gambles, and vice versa. And in desirability there are well-consolidated tools to derive, and work with, lower previsions (expectations); these tools now become available to the decision-theoretic investigator.

On the other hand, our standpoint is that the fundamental tool to model incomplete preferences is cone $\rdesirs$: it is more expressive that the sets of probabilities we can derive from it: it allows us to model and work with non-Archimedean problems besides Archimedean ones. And we can do this with well-established tools that apply directly to the cone, without any need to go through a probability-utility representation. This remarkably widens the set of problems we can handle.

\subsection*{Offsprings}

A variety of results and insights follow from the mentioned equivalence. We list them in the order they appear in the paper.

\begin{itemize}
\item We show in Section~\ref{sec:equiv} that the traditional Archimedean axiom conflicts with the possibility to represent a maximally uninformative, or \emph{vacuous}, relation. Thanks to the worst outcome, we define a \emph{weak Archimedean} condition, which is like the original one but restricted to a subset of \emph{non-trivial} preferences. This solves the problem and still suffices to obtain a representation in terms of sets of expected utilities.

Weak Archimedeanity enables us also to address the long-term controversy about the \emph{regularity assumption}, that is, whether or not a probabilistic model should be allowed to assign probability zero to events: probabilities can be zero in our model and yet we can have a meaningful representation of a strict partial order in terms of a set of expected utilities.

Finally, weak Archimedeanity allows the probabilistic models derived in our representation to be meaningfully defined on uncountable sets, unlike in the case of the traditional Archimedean axiom. This is remarked in Section~\ref{sec:interpretation}.

\item Still in Section~\ref{sec:interpretation} we discuss how desirable gambles might be taken as the primitive notion at the basis of preference, with an opportune interpretation.

We illustrate also how our representation in terms of expectations
is naturally based on objects that we can interpret as joint
probability-utility functions (e.g., joint mass functions in the
finite case). This enables us to use tools from probability theory
to deal with, and reinterpret, operations with utility, in a uniform
way.

\item Given that the presence of the worst act is important for this paper, we investigate in Section~\ref{sec:z} whether one can extend, in a least-committal way, a coherent preference relation that has no worst act, into one that has it. It turns out that it is always possible to find such a minimal extension, but there is a catch: the extension is never (weakly) Archimedean irrespectively of the preference relation we start from---apart from the trivial empty relation; moreover, that the notion of minimal extension is ill-defined for Archimedean extensions: given any Archimedean extension, it is always possible to find a smaller one.

Then we dive deeper into this problem, and define a \emph{strong} Archimedean condition, which is indeed stronger than the traditional one. We show that the weak, strong and traditional Archimedean conditions are all essentially\footnote{For the weak one this is only partly true.} equivalent in case a relation has a worst outcome. When it has not, and restricting the attention to the case of finite spaces (in particular finite $\pspace$), we show that the strong Archimedean condition leads to an Archimedean extension and that the traditional does not suffice. Moreover, we show that strong Archimedeanity is equivalent to the topological openness of cone $\cone$.

\item Starting from Section~\ref{sec:decompose}, we assume that the worst outcome exists and we take advantage of the desirability representation of incomplete preferences to revisit a number of traditional notions.

Initially we discuss the cases of state dependence and independence directly for the case of desirable gambles. We show that there are much weaker (and arguably, more intuitive) notions than the traditional ones we can employ to model state independence.

Something similar happens with the case of complete preferences. The definition in the case of desirability is straightforward and of great generality.

\item In Section~\ref{sec:lprevs}, we see what happens of these notions when we move down to the level of coherent lower previsions, that is, sets of expected utilities. Also in this case, we provide weaker notions than the traditional ones that are very direct.

For the case of complete preferences, we give a number of equivalent conditions to impose them, thus also simplifying the traditional conditions, and showing that we can use the very same condition both for the case of complete beliefs and incomplete values and for the opposite one (the so-called \emph{Knightian uncertainty}).

\item In Section~\ref{sec:si=sp} we consider two axioms used in the literature to impose state independence in the case of a multiprior expected multiutility representation. After quite an involved analysis, we show that imposing those axioms is equivalent to model state independence with sets of expectations using the \emph{strong product}. In other words, for complete preferences, state independence simply turns out to be stochastic independence in our setting; and to a set of stochastically independent models in the case of incomplete preferences. In fact, when we say, as above, that there are weaker ways to model state independence, this is because it is well known with sets of probabilities that there are much weaker ways to model irrelevance and independence than the strong product.

\item In Section~\ref{sec:archimedes} we argue that the Archimedean condition is inadequate to capture all the problems that can be tackled using sets of expected utilities. In particular, there are problems that can be modelled using collections of sets of expected utilities that are not Archimedean according to the common definition. We give a new definition of \emph{full Archimedeanity} that captures all and only the problems that can be expressed with collections of sets of expected utilities, and of which the Archimedean condition is a special case.

\end{itemize}

\subsection*{Comparing}

In Section~\ref{sec:relatedWs} we compare our work with some previous ones that have dealt with incomplete preferences.

We consider the work by Nau \cite{nau2006} in the light of the connection we make to desirability, showing how some of his notions map into ours and vice versa. Nau's work is based on weak preference relations; this also gives us the opportunity to discuss how the present approach can be adapted to that case.

The work by Galaabaatar and Karni \cite{galaabaatar2013} is particularly interesting to compare as it has been a source of inspiration for ours, and because it is actually quite close in spirit, even though it misses the connection to desirability.

Finally, we consider the work by Seidenfeld, Schervish and Kadane
\cite{seidenfeld1995}. This work is interesting to compare
especially for the different type of setting it is based on,
compared to ours and to the former ones. In particular, they use a
special type of Archimedean condition with a topological structure.
The paper is also based on quite technical mathematical tools.
Moreover, they work with sets of expected utilities that need not be
open or closed, thus gaining generality compared to the former
approaches. Interestingly, they also consider the problem of
extending a preference relation to a worst (and a best) act and come
to conclusions that seem to be clashing with ours.

We clarify the difference with Seidenfeld et al.'s work by mapping their concepts in our language of desirability. By doing so we show that there is no contradiction between their results and ours, and we argue that the type of generality they get to can be achieved more naturally and simply using convex cones of gambles rather than sets of expected utilities.

\subsection*{Summarising}
In summary, we present a very general approach to axiomatise, and work with, incomplete preferences, which to us appears simpler and with a great potential to clarify previous notions and to unify them under a single viewpoint. It is based on a shift of paradigm: regarding desirability as the underlying and fundamental concept at the basis of preference.

There are limitations in our current work, like the finiteness of $\values$, and other challenges left to address. We comment on these and other issues in the Conclusions. The Appendix collects the proofs of our results.

\section{Desirability and coherent lower
previsions}\label{sec:intro-clp}
\subsection{Foundations of desirability}\label{sec:foundations}
Let $\pspace$ denote the set of possible outcomes of an experiment. In this paper we let the cardinality of $\pspace$ be general, so $\pspace$ can be infinite. We call $\pspace$ the \emph{space of possibilities}. A \emph{gamble} $f:\pspace\rightarrow\reals$ is a bounded, real-valued, function of $\pspace$. It is interpreted as an uncertain reward in a linear utility scale: in particular, $f(\omega)$ is the amount of utiles a subject receives if $\omega\in\pspace$ eventually happens to be the outcome of the experiment. Let us name this subject Thilda.

We can model Thilda's uncertainty about $\pspace$ through the set of gambles she is willing to accept. We say also that those are her \emph{acceptable} or \emph{desirable gambles} (we use the two terms interchangeably). Accepting a gamble $f$ means that Thilda commits herself to receive $f(\omega)$ whatever $\omega$ occurs. Since $f(\omega)$ can be negative, Thilda can lose utiles and hence the desirability of a gamble depends on Thilda's beliefs about $\pspace$.

Denote by $\gambles(\pspace)$ the set of all the gambles on
$\pspace$ and by
$\gambles^+(\pspace)\coloneqq\{f\in\gambles(\pspace):f\gneq0\}$ its
subset of the \emph{positive gambles}: the non-negative non-zero
ones (the set of negative gambles is similarly given by
$\{f\in\gambles(\pspace):f\lneq0\})$. We denote these sets also by
$\gambles$ and $\gambles^+$, respectively, when there can be no
ambiguity about the space involved. Thilda examines a set of gambles
$\apr\subseteq\gambles$ and comes up with the subset $\domain$ of
the gambles in $\apr$ that she finds desirable. How can we
characterise the rationality of the assessments represented by
$\domain$?

We can follow the procedure adopted in similar cases in logic, where first of all we need to introduce a notion of deductive closure: that is, we first characterise the set of gambles that Thilda must find desirable as a consequence of having desired $\domain$ in the first place. This is easy to do since Thilda's utility scale is linear, whence those gambles are the positive linear combinations of gambles in $\domain$:
\begin{equation*}
  \posi(\domain)\coloneqq\set{\sum_{j=1}^r\lambda_jf_j}{f_j\in\domain,\lambda_j>0,r\ge1}.
\end{equation*}
On the other hand, we must consider that any gamble in $\gambles^+$ must be desirable as well, given that it may increase Thilda's utiles without ever decreasing them. Stated differently, the set $\gambles^+$ plays the role of the tautologies in logic. This means that the actual deductive closure we are after is given by the following:
\begin{definition}({\bf Natural extension for gambles}) Given a set $\domain$ of desirable gambles, its \emph{natural extension} $\rdesirs$ is the set of gambles given by
\begin{equation}
\rdesirs\coloneqq\posi(\domain\cup\gambles^+).\label{eq:posi-nex}
\end{equation}
\end{definition}
\noindent Note that $\rdesirs$ is the smallest convex cone that includes $\domain\cup\gambles^+$.

The rationality of the assessments is characterised through the natural extension by the following:
\begin{definition}({\bf Avoiding partial loss for gambles}) A set $\domain$ of desirable gambles is said to \emph{avoid partial loss} if $0\notin\rdesirs$.
\end{definition}
\noindent This condition is the analog of the notion of consistency
in logic. The irrationality of a natural extension that incurs
partial loss depends on the fact (as it is possible to show) that it
must contain a negative gamble $f$, that is, one that cannot
increase utiles and can possibly decrease them. In contradistinction, a set
that avoids partial loss does not contain negative gambles.

There is a final notion that is required to make a full logical theory of desirability. This is the logical notion of a theory, that is, a set of assessments that is consistent and logically closed, in the sense that the consistent assessments coincide with their deductive closure in the examined domain $\apr$. This means that Thilda is fully aware of the implications of her assessments on other assessments in $\apr$. The logical notion of a theory goes in desirability under the name of coherence:
\begin{definition}[{\bf Coherence for gambles}]\label{def:coh-R-zero}
Say that $\domain$ is \emph{coherent relative to $\apr$} if
$\domain$ avoids partial loss and
$\apr\cap\rdesirs\subseteq\domain$ (and hence
$\apr\cap\rdesirs=\domain$). In case $\apr=\gambles$ then we simply say that $\domain$ is \emph{coherent}.
\end{definition}
\noindent This definition alone, despite its conceptual simplicity, makes up all the theory of desirable gambles: in principle, every property of the theory can be derived from it. Moreover, the definition gives the theory a solid logical basis and in particular guarantees that the inferences one draws are always coherent with one another. At the same time, the theory is very powerful: as we have seen, it can be defined on any space of possibility $\pspace$ and any domain $\apr\subseteq\gambles$ (in this sense, it is not affected by measurability problems); and, as we shall make precise later on in this section, it can handle both precise and imprecise assessments, as well as model both Archimedean and non-Archimedean problems.

Sets of desirable gambles are uncertainty models and as such we can
define a notion of conditioning for them. As usual, we consider an
event $B\subseteq\pspace$. We adopt de Finetti's convention to
denote by $B$ both the subset of $\pspace$ and its indicator
function $I_B$ (that equals one in $B$ and zero elsewhere). Using
this convention, we can multiply $B$ and a gamble $f$ obtaining the
new gamble $Bf$ given by
\begin{equation*}
Bf(\omega)=
\begin{cases}
f(\omega)&\text{ if }\omega\in B\\
0&\text{ otherwise}
\end{cases}
\end{equation*}
for all $\omega\in\pspace$. Recall the interpretation of a gamble as an uncertain reward. Since gamble $Bf$ cannot change Thilda's wealth outside $B$, we can as well interpret $Bf$ as a gamble that is called off unless the outcome $\omega$ of the experiment belongs to $B$: we say that $Bf$ is a gamble \emph{contingent}, or \emph{conditional}, on $B$. This leads to the following:
\begin{definition}[{\bf Conditioning for gambles}]\label{def:condR} Let $\rdesirs$ be a coherent set of desirable gambles on $\gambles$ and $B$ be a non-empty subset of $\pspace$. The set of desirable gambles conditional on $B$ derived from $\rdesirs$ is defined as $\rdesirs|B\coloneqq\{f\in\rdesirs:f=Bf\}$.
\end{definition}
\noindent $\rdesirs|B$ is a set of desirable gambles coherent
relative to $\apr\coloneqq\{f\in\gambles:f=Bf\}$. Note that there is
a natural correspondence between $Bf$ and the restriction of $f$ to
$B$, whence we can also put $\rdesirs|B$ in relation with
$\rdesirsc{B}\coloneqq\{f_B\in\gambles(B):Bf_B\in\rdesirs|B\}$, and
show that this is a coherent set of desirable gambles in
$\gambles(B)$. The point here is that $\rdesirs|B$ and
$\rdesirsc{B}$ are equivalent representations of the conditional
set; yet, there can be some mathematical convenience in using one
over the other depending on the situation at hand. For these
reasons, and in the attempt to avoid a cumbersome notation, from now
on---with few exceptions---we shall use notation $\rdesirs|B$ for
the conditional set, even though, on occasions, what we shall
actually mean and use is $\rdesirs\rfloor B$. This abuse should not
be problematic as the specific set we shall use will be clear from
the context. For analogous reasons, in the following we shall
sometimes abuse terminology by just saying that $\rdesirs|B$ is
coherent, without specifying `relative to $\apr$'.

Finally, it is useful to consider the operation of marginalisation for a coherent set of desirable gambles.

\begin{definition}[{\bf Marginalisation for gambles}]\label{def:margR} Let $\rdesirs$ be a coherent set of desirable gambles in the product space $\pspace\times\pspace'$, where $\pspace,\pspace'$ are two logically independent sets. The marginal set of desirable gambles on $\gambles(\pspace\times\pspace')$ induced by $\rdesirs$ is defined as $\rdesirs_{\pspace}\coloneqq\{f\in\rdesirs:(\forall\omega\in\pspace)(\forall\omega_1',\omega_2'\in\pspace')f(\omega,\omega_1')=f(\omega,\omega_2')\}$. The marginal on $\gambles(\pspace')$ is defined analogously.
\end{definition}
\noindent $\rdesirs_{\pspace}$ is a coherent set of desirable
gambles relative to
$\apr_\pspace\coloneqq\{f\in\gambles(\pspace\times\pspace'):(\forall\omega\in\pspace)(\forall\omega_1',\omega_2'\in\pspace')f(\omega,\omega_1')=f(\omega,\omega_2')\}$. 
$\apr_\pspace$ collects all gambles that depend only on elements of
$\pspace$; we also say they are the $\pspace$-measurable gambles.
Since $\rdesirs_\pspace$ is made of $\pspace$-measurable gambles, we
can establish a correspondence between $\rdesirs_{\pspace}$ and
$\rdesirs_{\pspace}'\coloneqq\{g\in\gambles(\pspace):(\exists
f\in\rdesirs)(\forall\omega\in\pspace)(\forall\omega'\in\pspace')g(\omega)=f(\omega,\omega')\}$.
$\rdesirs_{\pspace}'$ is a coherent set of gambles in
$\gambles(\pspace)$. Similarly to the discussion made above in the
case of conditioning, there is no real difference in representing
the marginal information via $\rdesirs_{\pspace}$ or
$\rdesirs_{\pspace}'$, so from now on we shall stick to notation
$\rdesirs_\pspace$ for the marginal set, even when we shall actually
mean $\rdesirs_{\pspace}'$.

\subsection{Coherent sets of desirable gambles and coherent lower previsions}
When we restrict the attention to the case where $\apr=\gambles$, coherence can by characterised by four simple conditions:
\begin{theorem} $\rdesirs$ is a \emph{coherent set of desirable gambles} on $\gambles$ if and only if it satisfies the following conditions:
\begin{enumerate}[label=\upshape D\arabic*.,ref=\upshape D\arabic*]
\item\label{D1} $\gambles^+ \subseteq \rdesirs$ \emph{[Accepting Partial Gains]};
\item\label{D2} $0\notin \rdesirs$ \emph{[Avoiding Null Gain]};
\item\label{D3} $f \in \rdesirs, \lambda>0 \Rightarrow \lambda f \in \rdesirs$ \emph{[Positive Homogeneity]};
\item\label{D4} $f,g \in \rdesirs \Rightarrow f+g \in \rdesirs$ \emph{[Additivity]}.
\end{enumerate}
\end{theorem}
\noindent This result goes back to Williams \cite{williams1975} and
Walley \cite{walley1991} (for a recent proof,
see~\cite[Proposition~2]{miranda2010c}.) It shows, somewhat more
explicitly than Definition~\ref{def:coh-R-zero}, that a coherent set
of desirable gambles is a convex cone (\ref{D3},~\ref{D4}) that
excludes the origin (\ref{D2}) and that contains the positive
gambles (\ref{D1}).

A coherent set of desirable gambles implicitly defines a probabilistic model for $\pspace$. The way to see this is to consider gambles of the form $f-\mu$, where $\mu$ is a real value used as a constant gamble here, and $f$ is any gamble. Say that Thilda is willing to accept the gamble $f-\mu$. We can reinterpret this by saying that she is willing to buy gamble $f$ at price $\mu$. Focusing on the supremum price for which this happens leads us to the following:
\begin{definition}[{\bf Coherent lower and upper previsions}]\label{def:clps} Let $\rdesirs$ be a coherent set of desirable gambles on $\gambles$. For all $f\in\gambles$, let
\begin{equation}
\lpr(f)\coloneqq\sup\{\mu\in\reals:f-\mu\in\rdesirs\};\label{eq:lprR}
\end{equation}
it is called the \emph{lower prevision} of $f$. The conjugate value given by $\upr(f)\coloneqq-\lpr(-f)$ is called the \emph{upper prevision} of $f$. The functionals $\lpr,\upr:\gambles\rightarrow\reals$ are respectively called a \emph{coherent lower prevision} and a \emph{coherent upper prevision}.
\end{definition}
\noindent It is not difficult to see that an upper prevision can be written also as
\begin{equation*}
\upr(f)\coloneqq\inf\{\mu\in\reals:\mu-f\in\rdesirs\},\label{eq:uprR}
\end{equation*}
which makes it clear that it is Thilda's infimum selling price for gamble $f$. That buying and selling prices for some goods usually do not coincide is a matter of fact in real problems; this shows that the ability to represent such a situation is important if we aim at doing realistic applications of probability. In case they do coincide, instead, what we get are linear previsions:
\begin{definition}[{\bf Linear prevision}] Let $\rdesirs$ be a coherent set of desirable gambles on $\gambles$ and $\lpr,\upr$ the induced coherent lower and upper previsions. If $\lpr(f)=\upr(f)$ for some $f\in\gambles$, then we call the common value the \emph{prevision} of $f$ and we denote it by $\pr(f)$. If this happens for all $f\in\gambles$ then we call the functional $\pr$ a linear prevision.
\end{definition}

Now we would like to give a word of caution and of clarification: before proceeding, it should be crystal clear that linear previsions are nothing else but expectations. In particular, they are expectations with respect to the probability that is the restriction of $\pr$ to events: the probability of an event $B\subseteq\pspace$ is just $\pr(I_B)=\pr(B)$. Traditionally, one takes probability as the primitive concept, which is created on top of some structure such as a $\sigma$-algebra, and then computes the expectation of a measurable gamble (i.e., a measurable bounded random variable) $f$. Here instead probability is derived from expectation that is derived from a coherent set of desirable gambles. Among other advantages of this approach, one is that we do not need structures as $\sigma$-algebras to do our analysis since the probability underlying $\pr$ can be finitely additive; whence one can in principle also compute the prevision (expectation) of a non-measurable gamble. In this sense $\pr$ is more general and fundamental than a traditional expectation; it is actually an expectation in de Finetti's sense; it seems worth remarking this by using de Finetti's name for it: that of prevision, and to use symbol $\pr$ to denote it. Besides this, using symbol $\pr$ is mathematically accurate once probability is defined as the prevision of indicator functions, and it avoids us to keep on switching, in a needless way, between symbols $\nex$ and $\pr$.

In turn, coherent lower and upper previsions are just lower and upper expectation functionals. Consider a coherent lower prevision $\lpr$. We can associate it with a set of probabilities by considering all the linear previsions that dominate $\lpr$:
\begin{equation*}
 \solp(\lpr)\coloneqq\{\pr \text{ linear prevision}:(\forall f\in\gambles)
 \pr(f)\geq\lpr(f)\},
\end{equation*}
that turns out to be closed\footnote{In the weak* topology, which is
the smallest topology for which all the evaluation functionals given
by $f(\pr)\coloneqq \pr(f)$, where $f \in\gambles$, are continuous.}
and convex. Since each linear prevision is in a one-to-one
correspondence with a finitely additive probability, we can regard
$\solp(\lpr)$ also as a set of probabilities, which is called a \emph{credal set}. Moreover, $\lpr$ is the lower envelope of the previsions in $\solp(\lpr)$:
\begin{equation}
\lpr(f)=\inf\{\pr(f):\pr\in\solp(\lpr)\}.\label{eq:lowenv}
\end{equation}
In fact, coherent lower previsions are in a one-to-one correspondence with closed and convex sets of probabilities, such as $\solp(\lpr)$. The coherent upper prevision $\upr$ is, not surprisingly, the upper envelope of the same previsions; as a consequence, it follows that $\lpr(f)\leq\upr(f)$ for all $f\in\gambles$. In any case, and even if it is convenient sometimes to work with coherent upper previsions, let us remark that it is enough to work with coherent lower previsions in general, thanks to the conjugacy relation between them.

It is well known that a functional $\lpr:\gambles(\pspace)\rightarrow\reals$ is a coherent lower prevision if and only if it satisfies the following three conditions for all $f\in\gambles(\pspace)$ and all real $\lambda>0$:
\begin{enumerate}[label=\upshape C\arabic*.,ref=\upshape C\arabic*]
\item\label{C1} $\lpr(f)\geq \inf f$;
\item\label{C2} $\lpr(\lambda f)=\lambda \lpr(f)$;
\item\label{C3} $\lpr(f+g)\geq \lpr(f)+\lpr(g)$. \end{enumerate}
(When condition~\ref{C3} holds with equality for every $f,g\in\gambles(\pspace\times\values)$, then $\lpr$ is actually a linear prevision.) Therefore one can understand these three conditions also as the axioms of coherent lower previsions, thus disregarding the more primitive notion of desirability. Still, it is useful to know that coherent lower previsions are in a one-to-one correspondence with a special class of desirable gambles:
\begin{definition}[{\bf Strict desirability}] A coherent set of gambles $\rdesirs$ is said \emph{strictly desirable} if it satisfies the following condition:
\begin{enumerate}[label=\upshape D0.,ref=\upshape D0]
\item\label{D0} $f\in\rdesirs\setminus\gambles^+\Rightarrow (\exists\delta>0)f-\delta\in \rdesirs$ [Openness].
\end{enumerate}
\end{definition}
\noindent Strict desirability\footnote{A note of caution to prevent confusion in the reader: the adjective `strict' denotes two unrelated things in desirability and in preferences. In preferences it characterises irreflexive (i.e., non-weak) relations, while in desirability it formalises an Archimedean condition as it will become clear in Section~\ref{sec:equiv}. We are keeping the same adjective in both cases for historical reasons and given that there should be no possibility to create ambiguity by doing so.} is a condition of openness: it means that the part of cone represented by $\rdesirs\setminus\gambles^+$ does not contain the topological border. By an abuse of terminology, $\rdesirs$ is said to be open too.

The correspondence mentioned above holds in particular because we
can start with a coherent lower prevision $\lpr$
satisfying~\ref{C1}--\ref{C3} and induce the following set of
desirable gambles:
\begin{equation}
\rdesirs\coloneqq\{f\in\gambles:f\gneq0\text{ or }\lpr(f)>0\}.\label{eq:Rlpr}
\end{equation}
This set is coherent and in particular strictly desirable and moreover it induces $\lpr$ through Eq.~\eqref{eq:lprR}.

\subsection{Conditional lower previsions and non-Archimedeanity}\label{sec:regular}
We have just seen that coherent lower previsions and coherent sets
of strictly desirable gambles are equivalent models. This means also
that coherent sets of desirable gambles are more general than
coherent lower previsions, given that the general case of
desirability does not impose any constraint on the topological
border (such as openness). We can rephrase this by saying that
coherent sets of desirable gambles can model also non-Archimedean
problems, that is, problems that cannot be modelled by probabilities
(and in particular through a coherent lower prevision).

There are two main avenues where non-Archimedeanity can show up in desirability and both are related to gambles with zero prevision. The first has to do with the much debated problem about the way to deal with conditioning in case of zero-probability events. To see this, it is useful first to define a conditional coherent lower prevision.
\begin{definition}[{\bf Conditional coherent lower and upper previsions}]\label{def:cclps} Let $\rdesirs$ be a coherent set of desirable gambles on $\gambles(\pspace)$ and $B$ a non-empty subset of $\pspace$. For all $f\in\gambles(\pspace)$, let
\begin{equation}
\lpr(f|B)\coloneqq\sup\{\mu\in\reals:B(f-\mu)\in\rdesirs\}\label{eq:cclprR}
\end{equation}
be the \emph{conditional lower prevision} of $f$ given $B$. The conjugate value given by $\upr(f|B)\coloneqq-\lpr(-f|B)$ is called the \emph{conditional upper prevision} of $f$. The functionals $\lpr(\cdot|B),\upr(\cdot|B):\gambles(\pspace)\rightarrow\reals$ are respectively called a \emph{conditional coherent lower prevision} and a \emph{conditional coherent upper prevision}.
\end{definition}
\noindent Note that Eq.~\eqref{eq:lprR} is the special case of Eq.~\eqref{eq:cclprR} obtained when $B=\pspace$, whence for all matters we can stick to Eq.~\eqref{eq:cclprR} as the general procedure to obtain (conditional) coherent lower previsions from coherent sets of desirable gambles. Note, on the other hand, that $\lpr(f|B)=\sup\{\mu\in\reals:B(f-\mu)\in\rdesirs\}=\sup\{\mu\in\reals:B(f-\mu)\in\rdesirs|B\}=\sup\{\mu\in\reals:f_B-\mu\in\rdesirs\rfloor B\}\eqqcolon\lpr(f_B)$; here we have denoted by $f_B\in\gambles(B)$ the restriction of $f$ to $B$ and by $\lpr(f_B)$ its unconditional lower prevision obtained from set $\rdesirs\rfloor B$. The equality of the two lower previsions implies that $\lpr(\cdot|B)$ is equivalent to a set of probabilities and that it satisfies conditions similar to~\ref{C1}--\ref{C3} for all $f\in\gambles(\pspace)$ and all real $\lambda>0$:
\begin{enumerate}[label=\upshape CC\arabic*.,ref=\upshape CC\arabic*]
\item\label{CC1} $\lpr(f|B)\geq \inf_ Bf$;
\item\label{CC2} $\lpr(\lambda f|B)=\lambda \lpr(f|B)$;
\item\label{CC3} $\lpr(f+g|B)\geq \lpr(f|B)+\lpr(g|B)$, \end{enumerate}
 in addition to the condition, specific to the conditional case, that $\lpr(B|B)=1$ (this could be removed by formulating everything using $\lpr(f_B)$ rather than $\lpr(f|B)$). As in the unconditional case, one could take these four requirements as axioms of coherent conditional lower previsions, thus disregarding desirability. And also in this case, if we do start from a coherent conditional lower prevision $\lpr(\cdot|B)$, we can then induce its associated set of strictly desirable gambles through
\begin{equation}\label{eq:cond-from-lpr}
 \rdesirs|B\coloneqq\{f \in\gambles^+: f=Bf\}\cup\{B(f-(\lpr(f|B)-\eps)): \eps>0, f\in\gambles\}.
\end{equation}
\noindent Note that, according to Definition~\ref{def:condR}, $\rdesirs|B$ is made of gambles that are zero outside $B$. For the rest, the above expression simply states what is desirable under $\lpr(\cdot|B)$: either the positive gambles or the net gains originated by buying a gamble $f$ at price $\lpr(f|B)-\eps$, which Thilda regards as convenient since the price is less than her supremum acceptable one.

Eq.~\eqref{eq:cclprR} tells us how to create coherent conditional
lower previsions from a coherent set of desirable gambles. If we
apply it in particular to a coherent set of strictly desirable
gambles $\rdesirs$, then, thanks to its equivalence to a coherent
lower prevision $\lpr$, we obtain a conditioning rule defined
directly for coherent lower previsions:
\begin{definition}[{\bf Conditional natural extension}]\label{def:condnatex} Let $\lpr$ be a coherent lower prevision and $B$ a non-empty subset of $\pspace$. The \emph{conditional natural extension} of $\lpr$ given $B$ is the real-valued functional
\begin{equation}\label{eq:cond-natural-extension}
 \lpr(f|B)\coloneqq\begin{cases}
  \inf_{B} f &\text{ if } \lpr(B)=0, \\
  \min\{\pr(f|B): \pr\geq\lpr\} &\text{ otherwise}
  \end{cases}
\end{equation}
defined for every $f\in\gambles(\pspace)$, where $\pr(f|B)\coloneqq\pr(Bf)/\pr(B)$ is a conditional linear prevision defined by Bayes' rule.
\end{definition}
\noindent In other words, $\lpr(f|B)$ is obtained by conditioning
all the linear previsions in $\solp(\lpr)$ by Bayes' rule, when
$\lpr(B)>0$, and then taking their lower envelope. When $\lpr(B)=0$,
$\lpr(f|B)$ is instead \emph{vacuous}, and the interval
$[\lpr(f|B),\upr(f|B)]$ is equal to $[\inf_B f,\sup_B f]$ for all
$f\in\gambles$, whence it is completely uninformative about Thilda's beliefs when the conditioning event has zero lower
probability. This is just a limitation of an Archimedean model such
as a coherent lower prevision.

In contrast, it is known that the conditional lower prevision
$\lpr(f|B)$ obtained from a coherent set of non-strictly desirable
gambles can be informative, and actually for every pair
$\lpr,\lpr(\cdot|B)$ with $\lpr(B)=0$ that are coherent with each
other in Walley's sense, we can find a coherent set of desirable
gambles that induces them both (see~\cite[Appendix~F4]{walley1991}).
This is to say that conditioning with events of zero probability
does not pose any problem in the framework of desirability. This
happens thanks to the rich modelling capabilities offered by the
border of the cone, which is excluded from consideration in the case
of strictly desirable sets. Note that there are many common
situations that we would like to model where $B$ is assigned zero
lower probability and posterior beliefs are not vacuous: just think
of a bivariate normal density function over $\reals\times\reals$; it
assigns zero probability to each real number but conditional on a
real number it is again Gaussian, whence non-vacuous. These cases
fall in the area of general desirability.

The previous question, related to conditioning on an event of
probability zero, has illustrated the first type of
non-Archimedeanity that a coherent set of desirable gambles can
address. Still, it is possible to model the same case through
probabilities: the key is to use a collection of coherent lower
previsions as the basic modelling tool, such as the pair
$\lpr,\lpr(\cdot|B)$, rather than a single unconditional one.
However, there is another, somewhat purer, type of
non-Archimedeanity that cannot be modelled by collections either and
that can instead be modelled through desirability.  Here is an
example (taken from \cite[Example~13]{zaffalon2013a}):
\begin{example}\label{ex:non-fullA}
Two people express their beliefs about a fair coin using coherent
sets of desirable gambles. The possibility space
$\pspace\coloneqq\{h,t\}$, represents the two possible outcomes of
tossing the coin, i.e., heads and tails. For the first person, the
desirable gambles $f$ are characterised by $f(h)+f(t)> 0$; for the
second person, a gamble $f$ is desirable if either $f(h)+f(t)> 0$ or
$f(h)=-f(t)<0$. Call $\rdesirs_1$ and $\rdesirs_2$ the set of
desirable gambles for the first and the second person, respectively.
It can be verified that both sets are coherent. Moreover, they
originate the same unconditional and conditional lower previsions through Eqs.~\eqref{eq:lprR} and~\eqref{eq:cclprR}.
In the unconditional case we obtain $\lpr(f)=\frac{f(h)+f(t)}{2}$;
this corresponds, correctly, to assigning probability $\frac{1}{2}$ to both
heads and tails. In the conditional case, we again correctly obtain
that each person would assign probability 1 to either heads or tails
assuming that one of them indeed occurs: $\lpr(f|\{h\})=f(h),
\lpr(f|\{t\})=f(t)$. This exhausts the conditional and unconditional
lower previsions that we can obtain from $\rdesirs_1$ and
$\rdesirs_2$, given that $\pspace$ has only two elements. It follows
that $\rdesirs_1$ and $\rdesirs_2$ are indistinguishable as far as
probabilistic statements are concerned. But now consider the gamble
$f\coloneqq(-1,1)$, which yields a loss of 1 unit of utility if the
coin lands heads and a gain of 1 unit otherwise: whereas $f$ is not
desirable for the first person, it is actually so for the second.
This distinction of the two persons' behaviour cannot be achieved
through probabilities---and in fact gamble $f$ lies in the border of each of the two sets, given that $\lpr(f)=0$.
$\lozenge$
\end{example}
\noindent The same example can be rephrased in the language of
preferences (see~\cite[Example~10]{miranda2010c}). It shows that
coherent sets of desirable gambles can determine a preference also
when the lower expectation of the related gamble, in the case above
$(-1,1)$, is zero, which is a clear case of a non-Archimedean
preference. Again, the extra expressive power of general
desirability compared to strict desirability is made possible by the
modelling capabilities offered by the border of the involved cones.

All the discussion above on non-Archimedeanity shows in particular
that with coherent sets of desirable gambles we need not enter the
controversy as to whether or not we should use the \emph{regularity
assumption}, which prescribes that probabilities of possible events
should be positive. It is instructive to track the origin of this
assumption: it goes back to an important article by Shimony
\cite{shimony1955}. In the language of this paper, Shimony
argued---correctly, in our view--- that de Finetti's framework could
lead to the questionable non-acceptance of a positive gamble in case
zero probabilities where present, given that the prevision
(expectation) of such a gamble could be zero. This led Shimony and a
number of later authors, among whom Carnap and Skyrms, to advocate
strengthening de Finetti's theory by requiring regularity. But it
has originated also much controversy given the very constraining
nature of regularity on probabilistic models. Between requiring
regularity or dropping the acceptability of positive gambles, it
eventually emerged a third option: that of using non-Archimedean
models, which can keep both desiderata together. Unfortunately, this
idea has not been the subject of much development in mainstream
probability. But there are signs that something is changing and that
there could be a renewed interest in non-Archimedean models (see for
instance Pedersen's very recent interesting work
\cite{pedersen2014}). In this light, it is remarkable that Williams
has elegantly solved these problems by desirability as long as 40
years ago: in fact, by including the positive gambles in the natural
extension, no matter what Thilda's assessments are, as in
Eq.~\eqref{eq:posi-nex}, we make sure that they, and their
implications, are always desirable; and we do this without
compromising the presence of zero probabilities, therefore not
requiring regularity. What we get is a theory, very much in the
spirit of de Finetti's, that is very powerful, so much that it can
smoothly deal with non-Archimedeanity too.

An important difference between de Finetti's theory and desirability is that the former is developed for precise probabilistic assessments.  We can restrict desirability to precise models by an additional simple axiom of maximality:
\begin{definition}[{\bf Maximal coherent set of gambles}] A coherent set of desirable gambles $\rdesirs$ is called \emph{maximal} if
\begin{equation*}
 (\forall f\in\gambles\setminus\{0\}) f\notin\rdesirs\Rightarrow-f\in\rdesirs.
\end{equation*}
\end{definition}
\noindent Requiring maximality is a tantamount to assuming that Thilda
has complete preferences. The logic counterpart of maximality is
also called the completeness of a theory. It is interesting to
consider that logic has discovered long ago the inevitability of
incomplete theories, after G\"{o}del's celebrated theorem, and this
has led logicians to eventually appreciate their modelling power.
Mainstream (Bayesian) probability and statistics, on the other hand,
for the most part seem to be still stuck on precise probabilistic
models; and yet these are complete logical theories too.

Geometrically, a coherent maximal, or complete, set of desirable gambles $\rdesirs$ is a cone degenerated into a hyperplane. It induces, via Eq.~\eqref{eq:posi-nex}, a linear prevision $\pr$; this, in turn, induces through Eq.~\eqref{eq:Rlpr} a coherent set of strictly desirable gambles that corresponds to the interior of $\rdesirs$.\footnote{Remember that  by an abuse of terminology we say that a coherent set of strictly desirable gambles is open; for this same reason we refer to the union of $\gambles^+$ with the interior of $\rdesirs\setminus\gambles^+$ as `the interior' of $\rdesirs$.} Therefore there is a one-to-one correspondence between the interiors of maximal coherent sets and linear previsions. This connects again to the question of Archimedeanity. A maximal coherent set of gambles is a richer model than the associated coherent lower prevision, given that the former can profit from the border of the cone. Therefore, for example, it can yield a non-vacuous conditional linear prevision even when the conditioning event has precise zero probability; in contrast, the linear prevision that corresponds to the interior of the set will lead to a vacuous conditional model in that case.

\subsection{Conglomerability and marginal extension}\label{sec:congmet}
Finally, it is important for this paper to say something also about conglomerability. In fact, conditions~\ref{D1}--\ref{D4} essentially make up Williams' theory of desirability \cite{williams1975}. The competing theory by Walley \cite{walley1991} adds them a fifth condition that depends on the choice of a partition $\partit$ of the possibility space:
\begin{enumerate}[label=\upshape D5.,ref=\upshape D5]
\item\label{D5} $(\forall B\in\partit)Bf\in\rdesirs\cup\{0\} \Rightarrow f\in\rdesirs\cup\{0\}$ [Conglomerability],\footnote{Note that we should call it $\partit$-conglomerability, but we drop $\partit$ given that in this paper we shall always consider only one partition of the space.}
\end{enumerate}
This axiom follows from additivity when $\partit$ is finite. The rationale behind~\ref{D5} is that if Thilda is willing to accept gamble $f$ conditional on $B$, and this holds for all $B\in\partit$, then she should also be willing to accept $f$ unconditionally.

Despite the innocuous-looking nature of~\ref{D5}, conglomerability has originated almost a century-long controversy after de Finetti discovered it \cite{finetti1930}. De Finetti described conglomerability in the case of previsions and it can be shown that~\ref{D5} reduces to such a formulation when we induce a linear prevision from a coherent and conglomerable set of desirable gambles. The controversy concerns whether or not conglomerability should be imposed as a rationality axiom. De Finetti rejects this idea; others, like Walley, support it. In some recent work we have shown that there are cases where conglomerability is necessary for a probabilistic theory to make coherent inferences in time \cite{zaffalon2013a}. In any case, it is not our intention to enter the controversy in this paper and actually we shall try to avoid having to deal with questions of conglomerability as much as possible.

The natural extension of a set of gambles $\domain$ that avoids partial loss is the smallest superset that is coherent according to~\ref{D1}--\ref{D4}. We can proceed similarly in the case of conglomerability:
\begin{definition}[{\bf Conglomerable natural extension}] The \emph{conglomerable natural extension} of a set of gambles $\domain$, in case it exists, is the smallest coherent and conglomerable superset according to~\ref{D1}--\ref{D5}.
\end{definition}
\noindent The conglomerable natural extension plays the role of the deductive closure for a conglomerable theory of probability. Although the natural extension is easy to compute
\cite{walley1991}, this is not necessarily the case for the
conglomerable natural extension \cite{miranda2012,miranda2015}.
However, it will be simple in the context of this paper.

In fact, we shall use conglomerability jointly with a special type of hierarchical information: we shall consider a marginal coherent set of gambles $\rdesirs_{\pspace}$ and conditional coherent sets $\rdesirs|{\{\omega\}}$ for all $\omega\in\pspace$. It is an interesting fact then that the conglomerable natural extension of the given marginal and conditional information always exists and can easily be represented as follows (for a proof, see \cite[Proposition~29]{miranda2012}):
\begin{proposition}[{\bf Marginal extension}]\label{prop:met} Let $\rdesirs_{\pspace}$ be a marginal coherent set of gambles  in $\gambles(\pspace\times\pspace')$ and $\rdesirs|{\{\omega\}}$, for all $\omega\in\pspace$, be conditional coherent sets. Let
\begin{equation*}
 \rdesirs|\pspace\coloneqq\{h\in\gambles(\pspace\times\pspace'):(\forall
 \omega\in\pspace)
 h(\omega,\cdot)\in\rdesirs|{\{\omega\}}\cup\{0\}\}\setminus\{0\}
\end{equation*}
be a set that conglomerates all the conditional information along the partition $\{\{\omega\}\times\pspace':\omega\in\pspace\}$ of $\pspace\times\pspace'$ that we denote, by an abuse of notation, as $\pspace$ too. Then the $\pspace$-conglomerable natural extension of $\rdesirs_{\pspace}$ and $\rdesirs|{\{\omega\}}$ ($\omega\in\pspace$) is called their \emph{marginal extension} and is given by
\begin{equation*}
\hat{\rdesirs}\coloneqq\{g+h:
g\in\rdesirs_{\pspace}\cup\{0\},
h\in\rdesirs|\pspace\cup\{0\}\}\setminus\{0\}.
\end{equation*}

\end{proposition}

The marginal extension is a generalisation of the law of total expectation to desirable gambles. It was initially defined for lower previsions in \cite[Section~7.7.2]{walley1991} and later extended to deal with more than two spaces in \cite{miranda2007}; in the previous form for desirable gambles it has appeared in \cite[Section~7.1]{miranda2012}. Representing marginal extension through desirability allows one to take advantage of the increased expressiveness of the model; we can for instance condition our marginal extension on an event with zero (lower or upper) probability and obtain an informative model.

As we have said, the marginal extension can be defined also for lower previsions. To this end, we first need to introduce a way to conglomerate the conditional lower previsions defined over a partition of $\pspace$:
\begin{definition}[{\bf Separately coherent conditional lower prevision}]\label{def:sc-clp} Let $\partit$ be a partition of $\pspace$ and $\lpr(\cdot|B)$ a coherent lower prevision conditional on $B$ for all $B\in\partit$. Then we call
\begin{equation*}
\lpr(\cdot|\partit)\coloneqq\sum_{B\in\partit} B
\lpr(\cdot |B)
\end{equation*}
a \emph{separately coherent conditional lower prevision}.
\end{definition}
\noindent For every gamble $f$,
$\lpr(f|\partit)$ is the gamble on $\pspace$ that equals $\lpr(f|B)$ for $\omega\in B$; so it is a $\partit$-measurable gamble.

Secondly, we introduce the notion of marginalisation for coherent lower previsions similarly to the case of desirability:
\begin{definition}[{\bf Marginal coherent lower prevision}]\label{def:marg-clp} Let $\lpr$ be a coherent lower prevision on $\gambles(\pspace\times\pspace')$. Then the $\pspace$-marginal coherent lower prevision it induces is given by
\begin{equation*}
\lpr_{\pspace}(f)\coloneqq\lpr(f)
\end{equation*}
for all $f\in\gambles(\pspace\times\pspace')$ that are $\pspace$-measurable. The definition of $\lpr_{\pspace'}$ is analogous.
\end{definition}
\noindent In other words, the $\pspace$-marginal is simply the restriction of $\lpr$ to the subset of gambles in $\gambles(\pspace\times\pspace')$ that only depend on elements of $\pspace$. For this reason, and analogously to the case of desirability, we can represent the $\pspace$-marginal in an equivalent way also through the corresponding lower prevision $\lpr_{\pspace}'$ defined on $\gambles(\pspace)$. In the following we shall not distinguish between $\lpr_\pspace$ and $\lpr_\pspace'$ and rather use the former notation in both cases.

We are ready to define the marginal extension:
\begin{definition}[{\bf Marginal extension for lower previsions}] Consider the possibility space $\pspace\times\pspace'$ and its partition $\{\{\omega\}\times\pspace':\omega\in\pspace\}$. We shall denote this partition by $\pspace$ and its elements by $\{\omega\}$, with an abuse of notation. Let $\lpr_{\pspace}$ be a marginal coherent lower prevision and let $\lpr(\cdot|\pspace)$ be a separately coherent conditional lower prevision on $\gambles(\pspace\times\pspace')$. Then the marginal extension of $\lpr_{\pspace},\lpr(\cdot|\pspace)$ is the lower prevision $\lpr$ given by
\begin{equation*}
\lpr(f)\coloneqq\lpr_{\pspace}(\lpr(f|\pspace))
\end{equation*}
for all $f\in\gambles(\pspace\times\pspace')$.
\end{definition}
\noindent The marginal extension is the least-committal coherent lower prevision with marginal $\lpr_{\pspace}$ that is coherent with $\lpr(\cdot|\pspace)$, in the sense that there is a coherent and conglomerable set of desirable gambles $\rdesirs$ that induces both $\lpr$ and $\lpr(\cdot|\pspace)$ via Eq.~\eqref{eq:cclprR}.

\begin{remark}[{\bf On dynamic consistency}]
\label{rem:dc1}
It is useful to observe that the marginal extension is tightly related to the question of \emph{dynamic consistency} in decision problems (analogous considerations hold for marginal extension in the case of sets of desirable gambles). This is a concept originally highlighted by Machina \cite{machina1989} and also related to the work of Hammond \cite{hammond1988}. Loosely speaking, a decision problem is dynamically consistent if the optimal strategy does not change from the normal to the extensive form. In the context of this paper, an uncertainty model is understood as dynamically consistent if it coincides with the least-committal (that is, weakest) combination of the marginal and conditional information it induces.

For example, assume that we have a joint model on the product space $\pspace\times\pspace'$ given by a credal set $\solp$. We
derive from $\solp$ a set of marginal previsions $\solp_{\pspace}$ on
$\pspace$, and a family of sets of conditional previsions
$\{\solp(\cdot|\{\omega\}): \omega\in\pspace\}$, one for each element of $\pspace$. Then dynamic consistency means that we can recover
$\solp$ by taking the closed convex hull of set
\begin{equation}\label{eq:dynamic-consistency}
\{\pr_{\pspace}(\pr(\cdot|\pspace)):
 \pr_{\pspace}\in\solp_{\pspace}, (\forall \omega\in\pspace)\pr(\cdot|\{\omega\})
 \in\solp(\cdot|\{\omega\})\}.
\end{equation}
This seems to be what Epstein and Schneider called `rectangularity' in \cite{epstein2003b}. Note in particular that both the marginal linear prevision and each conditional linear prevision are free to vary in their respective sets in~\eqref{eq:dynamic-consistency} irrespectively of the other linear previsions; in other words, there are no `logical' ties between linear previsions in different credal sets. This is the essential feature that characterises a dynamically consistent model. In terms of lower previsions, if
we let $\lpr,\lpr_{\pspace},\lpr(\cdot|\pspace)$ be the coherent
lower previsions determined by the sets
$\solp,\solp_{\pspace},\{\solp(\cdot|\{\omega\}):\omega\in\pspace\}$ by
means of lower envelopes, as in~\eqref{eq:lowenv}, dynamic consistency means that we should
have $\lpr=\lpr_{\pspace}(\lpr(\cdot|\pspace))$, that is, $\lpr$
should correspond to a procedure of marginal extension.

Note that dynamic consistency, as described above, depends on the notion of `weakest combination' of marginal and conditional information. This notion may vary, thus yielding different dynamically consistent models, even though they induce the same marginal and conditional information. For instance, when independence enters the picture, the form of the weakest combination may depend on the notion of independence adopted. We shall discuss more about this point in Section~\ref{sec:si-clps} (see Remark~\ref{rem:si-indep}). $\lozenge$
\end{remark}

\section{Preference relations}\label{sec:IntroPrefs}

Let $\pspace$ denote, as before, the space of possibilities. In
order to deal with preferences, we introduce now another set
$\values$ of outcomes, or prizes. While the cardinality of $\pspace$
is unrestricted, in this paper we take $\values$ to be a finite set.
Moreover, we assume that all the pairs of element in
$\pspace\times\values$ are possible or, which is equivalent, that
$\pspace$ and $\values$ are logically independent.

The treatment of preferences relies on the basic notion of a horse lottery:
\begin{definition}[{\bf Horse lottery}]
We define a \emph{horse lottery} as a functional $p:\pspace\times\rspace$ such that $p(\omega,\cdot)$ is a probability mass function on $\rspace$ for all $\omega\in\pspace$.
\end{definition}

Let us denote by $\acts(\pspace\times\rspace)$ the set of all horse lotteries on $\pspace\times\rspace$. Horse lotteries will also be called \emph{acts} for short. In the following we shall use the notation $\acts$ for the set of all the acts in case there is no possibility of ambiguity. An act $p$ for which it holds that $p(\omega_1,\cdot)=p(\omega_2,\cdot)$ for all $\omega_1,\omega_2\in\pspace$ is called a \emph{von Neumann-Morgenstern lottery}; moreover, if such a $p(\omega,\cdot)$ is degenerate on an element $x\in\rspace$, then it is called a \emph{constant von Neumann-Morgenstern lottery} and is denoted by the symbol $x$: that is, $x(\omega,x)=1$ for all $\omega\in\pspace$.

A horse lottery $p$ is usually regarded as a pair of nested lotteries: at the outer level, the outcome $\omega\in\pspace$ of an experiment is employed to select the simple lottery $p(\omega,\cdot)$; this is used at the inner level to determine a reward $x\in\values$. Horse lotteries can be related to a behavioural interpretation through a notion of preference. The idea is that a subject, that this time we shall name Thomas, who aims at receiving a prize from $\values$, will prefer some acts over some others; this will depend on his knowledge about the experiment originating an $\omega\in\pspace$, as well as on his attitude towards the elements of $\values$. We consider the following well-known axioms of coherent preferences.

\begin{definition}[{\bf Coherent preference relation}]
A \emph{preference relation} $\succ$ over horse lotteries is a subset of $\acts\times\acts$. It is said \emph{coherent} if it satisfies the next two axioms:
\begin{enumerate}[label=\upshape A\arabic*.,ref=\upshape A\arabic*]
\item\label{A1} $(\forall p,q,r\in\acts)p\nsucc p\text{ and }p\succ q\succ r\Rightarrow p\succ r$ [Strict Partial Order];
\item\label{A2} $(p,q\in\acts)p\succ q \iff (p,q\in\acts)(\forall r\in\acts)(\forall\alpha\in(0,1])\alpha p+(1-\alpha)r\succ \alpha q+(1-\alpha)r$ [Mixture Independence].
\end{enumerate}
If also the next axiom is satisfied, then we say that the coherent preference relation is \emph{Archimedean}.
\begin{enumerate}[label=\upshape A\arabic*.,ref=\upshape A\arabic*]
\setcounter{enumi}{2}
\item\label{A3} $(p,q,r\in\acts)p\succ q\succ r\Rightarrow(\exists\alpha,\beta\in(0,1)) \alpha p+(1-\alpha)r\succ q\succ \beta p+(1-\beta)r$  [Archimedeanity].
\end{enumerate}
\end{definition}

Next, we recall a few results that we shall use in the paper. We omit their proofs, which are elementary.

\begin{proposition}\label{prop:nau} Suppose that for given $p,q,r,s\in\acts$ it holds that $\alpha(p-q)=(1-\alpha)(r-s)$ for some $\alpha\in(0,1)$.
Then for any coherent preference relation $\succ$ it holds that
\begin{equation*}
p\succ q \iff r\succ s.
\end{equation*}
\end{proposition}
\noindent Note that using the previous proposition one can also
easily show that
\begin{equation*}
p\succ q \text{ and } r\succ s \Rightarrow \alpha p + (1-\alpha)r
\succ \alpha q + (1-\alpha) s.\label{eq:A2-cons1}
\end{equation*}

Let the set of scaled differences of horse lotteries be defined by
\begin{equation}
\sdiff(\acts)\coloneqq\{\lambda(p-q):p,q\in\acts,\lambda>0\}.\label{eq:diffs-hl}
\end{equation}
We shall also denote this set simply by $\sdiff$ in the following,
if there is not a possibility of confusion. Preference
relation~$\succ$ is characterised by the following subset of
$\sdiff$:
\begin{equation}
\cone\coloneqq\{\lambda(p-q):p,q\in\acts,\lambda>0,p\succ q\},\label{eq:diffs-hl-prefs}
\end{equation}
as the next proposition remarks:
\begin{proposition}\label{prop:charC}Let $\succ$ be a coherent preference relation. Then for all $p,q\in\acts$ and $\lambda>0$, it holds that
\begin{equation*}
p\succ q\iff\lambda(p-q)\in\cone.
\end{equation*}
\end{proposition}
\noindent Moreover, $\cone$ has a specific geometrical structure:
\begin{proposition}\label{prop:C-convex-cone}
Let $\succ$ be a coherent preference relation. Then $\cone$ is a convex cone that excludes the origin; it is empty if and only if so is $\succ$.
\end{proposition}
\noindent It turns out that cones and coherent preference relations are just two ways of looking at the same thing.

\begin{proposition}\label{prop:one2one} There is a one-to-one correspondence between coherent preference relations and convex cones in $\sdiff\setminus\{0\}$.
\end{proposition}

\section{Equivalence of desirability and preference}\label{sec:equiv}
In this section we shall show how coherent preferences on acts can equivalently be represented through coherent sets of desirable gambles and vice versa. The key to establish the relation is the notion of worst outcome.

\subsection{Worst act, worst outcome}\label{sec:wawo}

It is customary in the literature to assume that a coherent preference relation comes with a best and a worst act, which are defined as follows:

\begin{definition}[{\bf Best and worst acts}] Let $\succ$ be a coherent preference relation and $p\in\acts$. The relation has a \emph{best act} $b$ if $b\succ p$ for all $p\neq b$ and, similarly, it has a \emph{worst act} $w$ if $p\succ w$ for all $p\neq w$.
\end{definition}
\noindent In this paper, however, we shall not be concerned with the best act.

A special type of worst act is one that is degenerate on the same element $z\in\values$ for all $\omega\in\pspace$ (it is therefore a constant von Neumann-Morgenstern lottery). We call it the worst outcome and we denote it by $z$ as well:

\begin{definition}[{\bf Worst outcome}] Let $\succ$ be a coherent preference relation. The relation has a \emph{worst outcome} $z\in\values$ if $p\succ z$ for all $p\neq z$.
\end{definition}

Now we wonder whether it is restrictive to assume that a coherent
relation has a worst outcome compared to assuming that it has a
worst act. To this end, we start by characterising the form of the
worst act:

\begin{lemma}\label{lem:w-degen}
Let $\succ$ be a coherent preference relation with worst act $w$. Then for each $\omega\in\pspace$, there is $x_{\omega}\in\values$ such that $w(\omega,x_\omega)=1$.
\end{lemma}
\noindent In other words, a worst act has to be surprisingly similar to a worst outcome: a worst act tells us that for each $\omega\in\pspace$ there is always an element $x_\omega\in\values$ that is the worst possible; the difference is that such a worst element is not necessarily the same for all $\omega\in\pspace$ as it happens with the worst outcome. But at this point it is clear that we can reformulate the representation in such a way that we can work with relation~$\succ$ as if it had a worst outcome (we omit the trivial proof):
\begin{proposition}\label{prop:w-z}
Consider a coherent preference relation $\succ$ with a worst act $w$. Let $z$ be any element in $\values$. Define the bijection $\sigma:\acts\rightarrow\acts$ that, for all $\omega\in\pspace$, does nothing if $x_\omega= z$ and otherwise swaps the probabilities of outcomes $x_\omega$ and $z$:
\begin{align*}
 \sigma(p)(\omega,x)\coloneqq\begin{cases}
  p(\omega,z) \quad &\text{ if } x=x_\omega, \\
  p(\omega,x_\omega) \quad &\text{ if } x=z, \\
  p(\omega,x) \quad &\text{ otherwise,}
  \end{cases}
\end{align*}
for all $p\in\acts$. Application $\sigma$ induces a relation $\succ_\sigma$ by:
\begin{align*}
p\succ_\sigma q\iff\sigma^{-1}(p)\succ\sigma^{-1}(q).
\end{align*}
Then it holds that $\succ_\sigma$ is a coherent preference relation for which $z$ is the worst outcome. Moreover, relation $\succ_\sigma$ is Archimedean if so is relation $\succ$.
\end{proposition}
\noindent In other words, if our original relation $\succ$ has a worst act $w$, we can always map it to a new relation, on the same product space $\pspace\times\values$, for which $z$ is the worst outcome and such that we can recover the original preferences from those of the new relation. This means that there is no loss of generality in assuming right from the start that $w=z$.

In the rest of the paper (with the exception of Section~\ref{sec:z}) we shall indeed assume that a coherent preference relation has a worst outcome. This will turn out to be enough to develop all the theory.

\begin{remark}[{\bf Notation for the worst outcome}]\label{rem:notationz} Given the importance of the worst outcome for this paper, it is convenient to define some notation that is tailored to it. In particular, when we want to specify that the set of acts contains $z$, then we shall denote it by $\valuesy$; otherwise we shall simply denote it by $\valuesn$. In the latter case, it can  either be that the set does not contain it or that the statements we do hold irrespectively of that. The distinction will be clear from the context. Note in particular that when the two sets are used together, the relation between them will always be that $\valuesy=\valuesn\cup\{z\}$. Moreover, we let $\actsy\coloneqq\acts(\pspace,\valuesy)$, besides the usual $\actsn\coloneqq\acts(\pspace,\valuesn)$, and we use them accordingly. $\lozenge$
\end{remark}

\begin{remark}[{\bf Assumption about the worst outcome}]\label{rem:assumptionz} Note, in addition, that if $\valuesy=\{z\}$, then it is necessary that $\actsy=\{z\}$ and hence the only possible relation on $\actsy\times\actsy$ is the empty one. We skip this trivial case in the paper by assuming throughout that $\valuesy$ contains at least two elements. $\lozenge$
\end{remark}

\subsection{Equivalence}\label{sec:equiv-prefs-gambles} When relation $\succ$ has the worst outcome, we can associate $\cone$ with another, equivalent, set. As it turns out, the new set will be one of
desirable gambles. To this end, we first define the projection
operator that drops the $z$-components from an act:
\begin{definition}[{\bf Projection operator}] Consider a set of outcomes $\valuesy$ that includes $z$. The \emph{projection operator} is the functional
\begin{equation*}
  \pi:\gambles(\pspace\times\valuesy)\rightarrow\gambles(\pspace\times\valuesn)
 \end{equation*}
defined by $\pi(h)(\omega,x)\coloneqq h(\omega,x)$, for all $(\omega,x)\in\pspace\times\valuesn$ and all $h\in\gambles(\pspace\times\valuesy)$.
\end{definition}
\noindent In this paper, we are going to use this operator to project horse
lotteries in $\actsy$, or scaled differences of them, into gambles on
$\pspace\times\valuesn$. Although $\pi$ is not injective in general,
both these restrictions, which we shall denote by $\pi_1,\pi_2$, are. As
a consequence, we shall denote by $\pi_1$ the restriction of the
projection operator to $\actsy$, and then define its inverse
\begin{align*}
 \pi_1^{-1}(f):\pspace\times\valuesy &\rightarrow \reals\\
 (\omega,x)&\mapsto \begin{cases}
   f(\omega,x) &\text{ if } x\neq z \\
   1-\sum_{x\in\valuesn} f(\omega,x) &\text{ if } x=z.
   \end{cases}
\end{align*}
Similarly, we shall denote by $\pi_2$ the restriction of $\pi$ to
$\sdiff(\actsy)$, and define its inverse as
\begin{align*}
 \pi_2^{-1}(f):\pspace\times\valuesy &\rightarrow \reals\\
  (\omega,x)&\mapsto \begin{cases}
   f(\omega,x) &\text{ if } x\neq z \\
   -\sum_{x\in\valuesn} f(\omega,x) &\text{ if } x=z.
   \end{cases}
\end{align*}

It is then an easy step to show the following:

\begin{proposition}\label{prop:rdesirs}
Let $\succ$ be a coherent preference relation on $\actsy\times\actsy$ and $\rdesirs$ be defined by
\begin{equation}
\rdesirs\coloneqq\{\lambda\pi(p-q):p\succ
q,\lambda>0\}=\pi(\cone).\label{eq:RfromC}
\end{equation}
Then:
\begin{itemize}
\item[(i)] $\rdesirs$ is a coherent set of desirable gambles on $\gambles(\pspace\times\valuesn)$.
\item[(ii)] $p-q\in\cone\iff\pi(p-q)\in\rdesirs.$
\end{itemize}
\end{proposition}

\begin{remark}[{\bf On preference through desirability}]\label{rem:interpret1} Despite this is a technically simple result, it has important consequences; it is worth stopping here one moment to consider them. Proposition~\ref{prop:rdesirs} gives us tools to analyse, and draw inferences about, preference by means of desirability. Nowadays desirability is well understood as a tool for uncertain modelling more primitive than probability. This offers us the opportunity to take a fresh new look at preference from the perspective of desirability. We shall exploit such an opportunity in the rest of the paper.

On another side, let us stress that it is important to correctly interpret the set of desirable gambles in~\eqref{eq:RfromC}. Set $\rdesirs$ is made of gambles from $\gambles(\pspace\times\values)$; but only $\pspace$ is actually a space of possibilities here, that is, the set made up of all the possible outcome of an uncertain experiment. This marks a significant departure from the traditional definition of a set of desirable gambles, which would require in this case that $\pspace\times\values$, and not just $\pspace$, was the space of possibilities. The way to interpret an element $f$ of $\rdesirs$ is then as a vector-valued gamble: for each $\omega\in\pspace$, $f(\omega,\cdot)$ is the vector of associated rewards. We shall detail this interpretation in Section~\ref{sec:interpretation}. In any case, we shall omit the specification `vector-valued' from now on and refer more simply to the elements of $\gambles(\pspace\times\values)$ as gambles. $\lozenge$
\end{remark}

The projection operator $\pi$ gives us also the opportunity to
define a notion of the preferences on which any rational subject should
agree and that for this reason we call objective:
\begin{definition}[{\bf Objective preference}] \label{def:dom}
Given acts $p,q\in\actsy$ we say that \emph{$p$ is objectively preferred to $q$} if
$\pi(p)\geq\pi(q)$ and $\pi(p)\neq\pi(q)$. We denote objective preference by $p\dom q$.
\end{definition}
\noindent The idea is that $p$ is objectively preferred to $q$ because the probability $p$ assigns to outcomes, except for the outcome $z$ (which is not one any subject actually wants), is always not smaller than that of $q$ while being strictly greater somewhere. An objective preference is indeed a preference:
\begin{proposition}\label{prop:domsucc}
Let $\succ$ be a coherent preference relation on $\actsy\times\actsy$ and $p,q\in\actsy$. Then $p\dom q\Rightarrow p\succ q$.
\end{proposition}
\noindent This proposition remarks in a formal way that objective preferences are those every rational subject expresses. These preferences are therefore always belonging to any coherent preference relation. If they are \emph{the only} preferences in the relation, then we call the relation \emph{vacuous} because Thomas is expressing only a non-informative, or trivial, type of preferences.

As surprising as it may seem, the vacuity of the coherent preference
relation turns out to be incompatible with the Archimedean axiom,
except in trivial cases:

\begin{proposition}\label{prop:vacuous-Arch}
Let $\succ$ be the vacuous preference relation on
$\actsy\times\actsy$, so that $p\succ z \iff p\neq z$. Then $\succ$
is Archimedean if and only if $|\pspace|=1$ and $|\valuesy|=2$.
\end{proposition}
\noindent It is important to remark that this incompatibility is not something we have to live with: for it is possible to define a weaker version of the Archimedean axiom that does not lead to such an incompatibility and that is based on restricting the attention to non-trivial preferences---that is, to non-objective ones. We can also simplify the axiom by focusing only on ternary relations such as $p\succ q\succ z$, which enables us also to skip the usual symmetrised version of the axiom. The result is the following:
\begin{definition}[{\bf Weak Archimedean condition}] Let $\succ$ be a coherent preference relation on $\actsy\times\actsy$. We say that the relation is \emph{weakly Archimedean} if it satisfies
\begin{enumerate}[label=\upshape wA\arabic*.,ref=\upshape wA\arabic*]
\setcounter{enumi}{2}
\item\label{wA3} ($p,q\in\actsy)p\succ q,  p\ndom q\Rightarrow(\exists\alpha\in(0,1))\alpha p+(1-\alpha)z\succ q$ [weak Archimedeanity].
\end{enumerate}
\end{definition}
\noindent The relation between the traditional Archimedean condition and the weak one is best seen through the following:
\begin{proposition}\label{prop:A3'-wA3} Let $\succ$ be a coherent preference relation on $\actsy\times\actsy$. Consider the following condition:
\begin{enumerate}[label=\upshape A\arabic*$'$.,ref=\upshape A\arabic*$'$]
\setcounter{enumi}{2}
\item\label{A3p} $(p,q,r\in\actsy)p\succ q\succ r,p\ndom q\Rightarrow(\exists\alpha\in(0,1)) \alpha p+(1-\alpha)r\succ q.$
\end{enumerate}
Then:
\begin{enumerate}
\item[(i)] Condition~\ref{A3p} implies that $$(p,q,r\in\actsy)p\succ q\succ r,q\ndom r\Rightarrow(\exists\beta\in(0,1)) q\succ\beta p+(1-\beta)r.$$
\item[(ii)] Relation $\succ$ satisfies~\ref{wA3} $\iff$ relation $\succ$ satisfies~\ref{A3p}.
\end{enumerate}
\end{proposition}
\noindent In other words,~\ref{wA3} retains the overall structure of the traditional condition~\ref{A3}, while focusing only on the non-trivial preferences. We shall discuss more about the weak Archimedean condition and its importance for the expressive power of a decision-theoretic representation in Section~\ref{sec:wA3-count}.

Next we detail the relation between the weak Archimedean condition for preferences and strict desirability for gambles.

\begin{theorem}\label{prop:rdesirs-archimedean}
Let $\succ$ be a coherent preference relation on $\actsy\times\actsy$ satisfying~\ref{wA3}. Then the corresponding set $\rdesirs$ obtained through~\eqref{eq:RfromC} is a coherent set of strictly desirable gambles.
\end{theorem}

The following result completes the picture: it turns out that
coherent sets of desirable gambles and coherent preferences are just
two equivalent representations, and this is true even when
Archimedeanity is taken into consideration.

\begin{theorem}\label{prop:rdesirsInv}
\begin{enumerate}
\item[(i)] Let $\sdiff(\actsy)$ be given by Eq.~\eqref{eq:diffs-hl}. For any gamble $f\in\gambles(\pspace\times\valuesy)$ it holds that
\begin{equation*}
f\in\sdiff(\actsy) \iff (\forall \omega\in\pspace)
\sum_{x\in\valuesy} f(\omega,x) = 0.
\end{equation*}

\item[(ii)] There is a one-to-one correspondence between elements of $\sdiff(\actsy)$ and gambles in $\gambles(\pspace\times\valuesn)$.

\item[(iii)] Let $\rdesirs$ be a coherent set of desirable gambles on
$\pspace\times\valuesn$. Then $\rdesirs$ defines a convex cone
$\cone$ that is equivalent to a coherent preference relation $\succ$
on $\actsy\times\actsy$ for which $z$ is the worst outcome.
Moreover, if $\rdesirs$ is a set of strictly desirable gambles, then
the relation is weakly Archimedean.
\end{enumerate}
\end{theorem}

\section{What happened?}\label{sec:interpretation}
It is useful to consider in retrospect what we have achieved so far.

\subsection{Desirability foundations of rational decision making}
The analysis in the previous sections can be summarised by the following:

\begin{theorem}\label{thm:equiv} There is a one-to-one correspondence between coherent preference relations  on $\actsy\times\actsy$ for which $z$ is the worst outcome and coherent sets of desirable gambles on $\pspace\times\valuesn$. Moreover the relation is weakly Archimedean if and only if the set of gambles is strictly desirable. In this case they give rise to the same representation in terms of a set of linear previsions.
\end{theorem}

In other words, from the mathematical point of view, it is completely equivalent to use coherent sets of desirable gambles on $\pspace\times\valuesn$ in order to represent coherent preference relations on $\actsy\times\actsy$. Or, to say it differently, we have come to realise that there is no need to distinguish two subjects, Thomas and Thilda, to make our exposition; accordingly, from now on we shall refer to the subject that makes the assessments more simply as `you'. That the two theories are just the same is even more the case if we consider that we can define a notion of preference directly on top of a coherent set of desirable gambles \cite[Section~3.7.7]{walley1991}:
\begin{definition}[{\bf Coherent preference relation for gambles}]\label{def:pref-gmb} Given a coherent set of desirable gambles $\rdesirs$ on $\gambles(\pspace\times\valuesn)$ and gambles $f,g$, we say that $f$ is preferred to $g$ in $\rdesirs$ if $f-g\in\rdesirs$, and denote this by $f\succ g$.
\end{definition}
\noindent In other words, $f\succ g$ if you are willing to give away $g$ in order to have $f$. 

It is trivial to show that the notion of preference we have just described is equivalent to the preference relation on horse lotteries that corresponds to $\rdesirs$: this means that $f\succ g$ if and only if $p\succ q$ for all $p,q\in\actsy$ such that $p-q=\pi(\lambda(f-g))$ for some $\lambda>0$. This is also the reason why we use the same symbol in both cases. 

All this suggests that we can think of establishing the foundations
of rational decision making using the desirability of gambles as our
primitive conceptual tool, rather than regarding it as derived from
preferences over horse lotteries. In this new conceptual setting,
you would be asked to evaluate which gambles you desires
in $\gambles(\pspace\times\valuesn)$ and then we would judge the
rationality of your assessments by checking whether or not the
related set of gambles is coherent according to~\ref{D1}--\ref{D4}.
This appears to be utterly natural and straightforward, and yet to
make this procedure possible, it is essential that the numbers
making up a gamble are given a clear interpretation, so as to put
you in the conditions to make meaningful and accurate assessments
of desirability.

\subsection{Interpretation}
 Remember that any gamble $f\in\gambles(\pspace\times\valuesn)$ is such that $f=\pi(\lambda(p-q))$, for some $p,q\in\actsy$ and $\lambda>0$. Therefore for all $\omega\in\pspace$ and $x\in\valuesn$, $f(\omega,x)$ is a number proportional to the increase (or decrease if it is negative) of the probability to win prize $x$ in state $\omega$ as a consequence of exchanging $q$ for $p$. We can use this idea to interpret $f(\omega,x)$ more directly, without having to rely on horse lotteries, in the following way.

Imagine that there is a simple lottery with possible prizes $x\in\valuesn$, all of them with the same, large, number of available tickets.\footnote{It is possible to assume that lottery tickets are infinitely divisible so as to cope with non-integer numbers; we shall neglect these technicalities in the description and just assume that their number is very large to make things simpler.} Tickets are sequentially numbered irrespectively of their type. Eventually only one number will be drawn according to a uniform random sampling; and this you know. Therefore the (objective, or physical) probability $\rho_x$ that you win prize $x$ is proportional to the number of $x$-tickets you own. This implies that your utility scale is linear in the number of $x$-tickets for all $x\in\valuesn$. You can increase, but also decrease, the number of your lottery tickets by taking gambles on an unrelated experiment, with possible outcomes in $\pspace$, about which you are uncertain: by accepting gamble $f\in\gambles(\pspace\times\values)$, you commit yourself to accept $f(\omega,\cdot)$ in case $\omega$ occurs; this is a vector proportional to the number of lottery tickets you will receive for each $x\in\values$ (this means also to have to give away $x$-tickets when $f(\omega,x)$ is negative).\footnote{Thanks to the linearity of your utility scale, we can assume without loss of generality that you initially own a positive number of $x$-tickets for all $x\in\valuesn$; for the same reason, the size of the positive proportionality constant is irrelevant. It follows that we can reduce the number of tickets you receive or give away so as to keep all your probabilities of winning prizes between 0 and 1.} In case you accept more than one gamble (and in any case finitely many of them), the reasoning is analogous; it is enough to take into account that the tickets received or given away will be proportional to the sum of the accepted gambles.

Stated differently, a gamble is interpreted as an uncertain reward in lottery tickets. Since lottery tickets eventually determine the probability to win prizes, such a probability is treated here as a currency. The idea of a `probability currency' appears to go back to Savage \cite{savage1954} and Smith \cite{smith1961} and has already been used also to give an interpretation of gambles in Walley's theory \cite[Section~2.2.2]{walley1991}. The difference between the traditional idea of probability currency and the one we give here rests on the possible presence of more than one prize, which is allowed in the present setting. That the traditional concept of probability currency can be extended to our generalised setting is ensured by Theorem~\ref{thm:equiv}.

With this interpretation in mind, it is possible to state and deal with a decision-theoretic problem only in terms of gambles on $\pspace\times\valuesn$. This has a number of implications:
\begin{itemize}
\item One is that we can profit now from all the theory already developed for desirable gambles in order to study decision-theoretic problems in quite a generality. We shall actually do so in the rest of the paper. \item Another is that we should be careful to interpret gambles properly, as we said already in Remark~\ref{rem:interpret1}; in particular, despite our gambles are on $\pspace\times\valuesn$, only $\pspace$ is the space of possibilities---which is something that marks a departure from the original theory of desirable gambles. The most important consequence here is that it is meaningless to think of the element of $\values$ as events that may occur; only elements of $\pspace$ can. As a consequence, we can still use conditioning as an updating rule, under the usual interpretations and restrictions; but we can do so only conditional on subsets of $\pspace$. Note, on the other hand, that in the current framework we can update both beliefs and values.
\item Finally, our results so far appear to have considerable implications for Williams' and Walley's behavioural theories of probability and statistics: in fact, the current reach of these, otherwise very powerful and well-founded, theories is limited to precise and linear utility. By extending gambles so as to make them deal with multiple prizes, we are de facto laying the foundations for their extension to imprecise non-linear utility. This should allow, with time, a whole new range of problems to be addressed in a principled way by those theories.
\end{itemize}

\subsection{A new way to represent utility (values)}

There is a feature of our proposed representation of decision-theoretic problems that is worth discussing even only to make the question explicit, as it seems to be different from past approaches.

Let us start by considering the simplest setup to make things very
clear: consider finite $\pspace$ (as well as $\values$, as usual)
and assume that the assessments eventually yield a coherent set of
strictly desirable gambles $\rdesirs$ corresponding to the interior
of a maximal set. This means that $\rdesirs$ is in a one-to-one
correspondence with a joint mass function $\Pi$ on
$\pspace\times\values$ that is the restriction of a linear prevision
to the elements of the product space. $\Pi$ is a joint mass function
in the sense that $\Pi(\omega,x)\geq0$ for all
$(\omega,x)\in\pspace\times\values$ and
$\sum_{(\omega,x)\in\pspace\times\values}\Pi(\omega,x)=1$. Now, if
we assume that $\Pi(\omega,x)>0$ for all
$(\omega,x)\in\pspace\times\values$, then $\Pi$ becomes an
equivalent representation of $\rdesirs$.

Moreover, we can rewrite $\Pi$ using total probability (marginal extension) as the product of the marginal mass function on $\pspace$ and of the conditional mass function on $\values$ given $\omega\in\pspace$:
\begin{equation}
\Pi(\omega,x)=\Pi(\omega)\Pi(x|\omega)\label{eq:Pi}
\end{equation}
for all $(\omega,x)\in\pspace\times\values$. Here the meaning of $\Pi(\omega)$ is that of the probability of $\omega$, whereas $\Pi(x|\omega)$ represents the utility of outcome $x$ conditional on the occurrence of $\omega$.

The important, and perhaps unconventional, point in all this discussion is that in our approach the utility function is just another mass function, that is, one such that $\Pi(x|\omega)\geq0,\sum_{x\in\values}\Pi(x|\omega)=1$ for all $\omega\in\pspace$. This means that the utility mass function is formally equivalent to a probability mass function, it is only our interpretation that changes: while the numbers in the probability function represent occurrence probabilities, in the utility function they represent mixture coefficients that allow us to compare vectors of (amounts of) outcomes. The more traditional view of the utilities is that of numbers between zero and one that do not need to add up to one.

An advantage of our representation is that we can directly exploit all the machinery and concepts already in place for probabilities also for utilities. For example, the notion of state independence trivially becomes probabilistic independence through~\eqref{eq:Pi}, and this happens whenever for all $x\in\values$ it holds that $$\Pi(x|\omega')=\Pi(x|\omega'')$$ for all $\omega',\omega''\in\pspace$, so that we are allowed to work with marginal mass functions for probability and utility. Instead, that state independence fails simply means that we have to work with the joint mass function $\Pi$, which appears to be a very natural way to address this case (all this will be detailed in Sections~\ref{sec:decompose} and~\ref{sec:si=sp}).

This feature extends over all the representation, not only in the specific case considered, as we can see by moving backwards toward greater generality. For instance, if we relax the precision requirement of the previous example, what we get is a representation of a decision-theoretic problem by a closed convex set $\solp$ of joint mass functions. State independence in this case will coincide with a notion of independence or irrelevance for credal sets among all the possible ones. When state independence holds, preferences will be formed by combining a marginal credal set for probabilities $\solp_\pspace$ with a marginal credal set for utilities $\solp_\values$. Completeness of utilities will mean that $\solp_\values$ contains only one mass function; and vice versa for completeness of probabilities.

These credal sets are in a one-to-one correspondence with coherent lower previsions. Whence, there will be a coherent lower prevision $\lpr(\cdot|\{\omega\})$ representing the lower envelope of the previsions that we can obtain from the credal set $\solp_{\values}(\cdot|\{\omega\})$ of the utilities conditional on $\{\omega\}$; that is, by formally treating utilities as probabilities and computing the related expectations. Relaxing the example further, we can consider an infinite (not necessarily countable) space $\pspace$ and in this case the credal set $\solp_{\pspace}$ will be made up of finitely additive probabilities.

What we find remarkable is also that we can easily obtain the quantities we need explicitly from sets of desirable gambles using established tools. In more traditional axiomatisations of preferences, separation theorems are usually aimed at obtaining a `proof of existence' of the linear functionals that represent probabilities and utilities, which then leads to a representation theorem. Our representation naturally appears to be more of a constructive nature, in the sense that we can actually compute the wanted numbers through~\eqref{eq:cclprR}; for instance, we can obtain the lower utility mass of $x\in\values$ conditional on $\omega\in\pspace$ simply as $$\lpr(I_{\{x\}}|\{\omega\})=\sup\{\mu\in\reals:I_{\{\omega\}}(I_{\{x\}}-\mu)\in\rdesirs\}.$$

The overall situation is very much alike when we finally move
backwards to $\rdesirs$. This will be in general a `joint' model for
beliefs and values. When we decompose it into separate models, there
will be a marginal set for beliefs $\rdesirs_{\pspace}$ and other
models for values, like $\rdesirs|\{\omega\}$, which will be
formally treated as usual as if they were models of conditional
beliefs instead of values. And this will allow us to exploit all the
tools and concepts already developed for desirability, like the
marginal extension, for instance. Yet, not all the concepts will be
alike: for instance, as we said already, our interpretation excludes
that the elements of $\values$ `occur', so that updating should be
done only conditional on subsets of $\pspace$.

Note that we are not claiming that all this is not possible to do in the traditional way to axiomatise preference. What we want to say is that our framework is naturally suited to represent decision-theoretic problems in which values (utilities) are formally treated as beliefs (probabilities); and that this helps dealing with these problems in a familiar way while exploiting well-established tools for beliefs also for values.

\subsubsection{Weak Archimedeanity, strict preference and the cardinality of spaces}\label{sec:wA3-count}

The discussion in the previous section helps us also to uncover a
question related to Condition~\ref{A3} in relation to strict
preference orderings. In fact,~\ref{A3} is conceived to eventually
lead to models that have to assign positive lower expected utility
(prevision) to all the elements of cone $\cone$, or, equivalently,
to all the gambles in $\rdesirs$. Remember that in the formulation
based on $\rdesirs$, beliefs and values naturally turn out to be
jointly represented through a set $\solp$ of linear previsions,
which are formally equivalent to a set of probabilities (in the
sense that not only probability but also utility is mathematically
represented through a probability function). This enables us to see
clearly that, under~\ref{A3} and strict preference, it is necessary
that $\pspace\times\values$ be countable: for $\rdesirs$ contains
the positive gambles, among which that are the indicator functions
of the subsets of $\pspace\times\values$; and a probability can
assign positive mass to countably many events at most.

Of course this is quite a limitation in the representation. It seems also a limitation that relates to the idea of regularity discussed in Section~\ref{sec:regular}. All this may have us to appreciate that, by taking aside the positive gambles, the weak Archimedean condition~\ref{wA3} is not subject to such a problem: it yields a representation through a set of linear previsions on gambles that can be defined on uncountable spaces. Note, in fact, that in this paper we are not restricting the cardinality of $\pspace$ (we do restrict the cardinality of $\values$ but for other reasons). This seems to be an original feature of our approach that has not been exploited so far.

\section{Extending a preference relation to the worst outcome}\label{sec:z}

We have seen in the previous sections that the presence of the worst outcome in a preference relation is the key to establish the equivalence between preference and desirability. It is interesting to note, moreover, that the worst outcome originates the acceptability of the positive gambles in desirability; and this is what solves the controversy about the regularity assumption described in Section~\ref{sec:regular}. As a consequence, the presence of the worst outcome appears to be somewhat of a fundamental concept for a preference relation. Then, in case a preference relation does not come with a worst act/outcome, it would be important to have tools to extend it to the worst outcome in a minimal kind of way.

So consider a coherent preference relation $\succ$ on $\actsn\times\actsn$. Let us investigate the problem of extending it to $\actsy\times\actsy$ in such a way that $z$ acts as the worst outcome. We start by making precise the notion of extension:

\begin{definition}[{\bf Extension to the worst outcome}]\label{def:ext} Let $\succ$ be a coherent preference relation on $\actsn\times\actsn$. An \emph{extension of $\succ$ to the worst outcome} is a coherent relation, denoted by the same symbol $\succ$, on $\actsy\times\actsy$ such that:
\begin{itemize}
\item[(i)] $p\succ z$ for all $p\in\actsy,p\neq z$;
\item[(ii)] $\pi_1^{-1}(p)\succ\pi_1^{-1}(q)$ whenever $p\succ q$ for some $p,q\in\actsn$.
\end{itemize}
\end{definition}
\noindent Note that the definition is based on two truly minimal
requirements. The first is just the definition of $z$ as the worst
outcome. The second makes the idea of the extension  precise: we use
the one-to-one correspondence between $\actsn$ and
$\pi_1^{-1}(\actsn)$ to require that the original preferences be
preserved by the extension.

The next theorem gives an equivalent characterisation of an extension to the worst outcome; based on this, it shows that there always exists a minimal extension, which is one that is included in any other, and yet that this extension is never weakly Archimedean (except for a trivial case):

\begin{theorem}\label{thm:ext} Let $\succ$ be a coherent relation on $\actsn\times\actsn$ and let $\cone$ be the corresponding cone. Then:
\begin{itemize}
\item[(i)] A coherent preference relation $\succ$ on $\actsy\times\actsy$ extends relation $\succ$ on $\actsn\times\actsn$ to the worst outcome if and only if
\begin{equation}
(\forall p,q\in\actsy)~p\succ q\text{ whenever }p\dom q\text{ or }\pi(p-q)=r-s\text{ for some }r,s\in\actsn,r\succ s.\label{eq:ext}
\end{equation}
\item[(ii)] The minimal extension of $\succ$ to the worst outcome corresponds to the coherent set of desirable gambles given by:
\begin{equation}
\rdesirs\coloneqq\{f\ge\lambda(r-s):r,s\in\actsn,r\succ s,\lambda>0\}\cup\gambles^+.\label{eq:extR}
\end{equation}
\item[(iii)] Let $\lpr$ be the coherent lower prevision induced by $\rdesirs$. For all $f\in\cone$, it holds that $\lpr(f)=0$. As a consequence, the minimal extension is weakly Archimedean if and only if $\cone=\emptyset$.
\end{itemize}
\end{theorem}

Note that Condition~\eqref{eq:ext}, which characterises the extensions in Theorem~\ref{thm:ext}, has a simple interpretation. The part related to objective preferences has been discussed already after Definition~\ref{def:dom}: it means that objective preferences should be indeed preferences irrespective of the specific subject we consider. Concerning the remaining part, observe that $\pi(p-q)\in\cone$ implies that $p(\omega,z)=q(\omega,z)$ for all $\omega\in\pspace$ (note that $p(\omega,z)$ need not be constant on $\omega\in\pspace$). This means that your evaluation of $p$ versus $q$ is not going to be based on their behaviour on $z$, since it is identical in the two cases; and for this reason one can determine the preference relation between $p$ and $q$ by relying only on the original relation (via $\cone$).\footnote{The further reformulation~\eqref{eq:ext2} in~\ref{sec:proofs} introduces additional preferences; those can be read as a simple consequence of the independence axiom~\ref{A2}: they state that if $p\succ q$ then also $\alpha p+(1-\alpha)t\succ\alpha q+(1-\alpha)z$ for all $\alpha\in(0,1]$ and $t\in\actsy\setminus\{z\}$, given that $t\succ z$.}

The second part of Theorem~\ref{thm:ext} gives an important result: that there is always a least-committal extension of a coherent preference relation to the worst outcome. Saying that it is least-committal means that such an extension is not requiring you to accept any preference other that those following from axioms~\ref{A1},~\ref{A2} and the assumption that $\valuesn$ can be bounded by a worst outcome.\footnote{We call it `least-committal' and not, for instance, `without loss of generality', also because the extension to the worst outcome does not seem to be trivial anyway: for, after all, we are bounding the relation in this way, and it is possible that you disagree that your actual preference can be bounded.} This is a positive result in that under minimal conditions it allows us to assume right from the start that we work with preference relations provided with the worst outcome.

On the other hand, in the third part of the theorem we obtain a
result that is not as positive, and in particular that $\rdesirs$ is
\emph{never} a coherent set of strictly desirable gambles (except in
the trivial case that $\cone=\emptyset$). Remember that there is a
one-to-one correspondence between sets of strictly desirable gambles
and coherent lower previsions. In the specific context we are
dealing with, this means that the minimal extension cannot be reproduced by a set of linear previsions. Stated
differently, we obtain that \emph{any} coherent preference relation
(possibly comprising any Archimedean notion we might envisage), when extended in a least-committal way to the worst
outcome, is not going to be weakly Archimedean, as it follows from
Theorem~\ref{prop:rdesirs-archimedean}, and hence not Archimedean either, because of Proposition~\ref{prop:A3'-wA3} (see also Proposition~\ref{pr:equiv-Archs-z} in Section~\ref{sec:arc-rels}).

So there is no minimal extension of relation $\succ$ to the worst outcome that is Archimedean. Still, one might wonder whether there are Archimedean extensions. These would introduce assessments other than those in~\eqref{eq:ext}. But it could also be that these other assessments are necessary for an extension to be Archimedean. And then one could think of defining the notion of minimal extension, this time among the Archimedean ones, and look for such a kind of minimal extension, provided that it exists. The problem is that it does not:
\begin{corollary}\label{cor:Aext} Let $\succ$ be a coherent relation under the assumptions of Theorem~\ref{thm:ext} and $\rdesirs$ be as in~\eqref{eq:extR}. If there is a coherent set of strictly desirable gambles $\rdesirs_1$ that includes $\rdesirs$, then there is another coherent set of strictly desirable gambles $\rdesirs_2$ such that $\rdesirs\subsetneq\rdesirs_2\subsetneq\rdesirs_1$.
\end{corollary}
\noindent We can rephrase the corollary by saying that the notion of minimal extension among the weak Archimedean ones is ill-posed: in fact, the minimal extension should correspond to the intersection of all the Archimedean extensions; but this is just $\rdesirs$, the least-committal extension of Theorem~\ref{thm:ext}.

As a consequence, the choice of an Archimedean extension cannot be made automatically, it is you that should tell us which one to select among the infinitely many possible ones. For the same reason, even if we start from a weak Archimedean preference relation, and drop the worst outcome, then there is no way in general to have an automatic procedure to recover the preference relation from which we started. Stated differently, it appears that if you hold Archimedean preferences, you should first of all have in mind a worst outcome and only then start defining your assessments.

\subsection{Archimedean conditions and their relations}\label{sec:arc-rels}

Let us consider a strengthening of the Archimedean property~\ref{A3}, which will be necessary for the developments in the next section:
\begin{definition}[{\bf Strong Archimedean condition}] Let $\succ$ be a coherent preference relation on $\acts\times\acts$. We say that the relation is \emph{strongly Archimedean} if it satisfies
\begin{enumerate}[label=\upshape sA\arabic*.,ref=\upshape sA\arabic*]
\setcounter{enumi}{2}
\item\label{A3'} $(p,q\in\acts)p \succ q \Rightarrow(\forall r,s \in\acts)(\exists \beta \in (0,1)) \beta p+(1-\beta) r \succ \beta q+ (1-\beta) s$ [strong Archimedeanity].
\end{enumerate}
\end{definition}

That~\ref{A3'} is in fact stronger than the usual Archimedean property~\ref{A3} is established in the following:

\begin{proposition}\label{pr:charac-archimedean} Let $\succ$ be a coherent relation and $\cone$ the corresponding cone. Then:
\begin{enumerate}
\item[(i)] Relation $\succ$ satisfies~\ref{A3} $\Leftrightarrow(\forall f,g\in
\cone)(\exists \beta_1, \beta_2\in (0,1))\beta_1
f-(1-\beta_1)g \in \cone , \beta_2 g-(1-\beta_2)f \in
\cone $.
\item[(ii)] Relation $\succ$ satisfies~\ref{A3'} $\Leftrightarrow(\forall f\in
\cone)(\forall g\in\sdiff)(\exists \beta\in (0,1))\beta
f+(1-\beta)g \in \cone $.
\item[(iii)] Relation $\succ$ satisfies~\ref{A3'} $\Rightarrow$ relation $\succ$
satisfies~\ref{A3}.
\end{enumerate}
\end{proposition}

To see that~\ref{A3} and~\ref{A3'} are not equivalent in
general, consider the following example:

\begin{example}\label{ex:stronger}
Consider a gamble $0\neq f\in\sdiff$, and let us define the
binary relation $\succ$ on ${\mathcal H}\times{\mathcal H}$ by
$p\succ q\Leftrightarrow p-q=\lambda f$ for some $\lambda>0$. This
relation satisfies the following axioms:
\begin{itemize}
 \item[\ref{A1}.] $f\neq 0 \Rightarrow p\nsucc p$ for every
 $p\in{\mathcal H}$. Moreover, if $p\succ q\succ r$, there are
 $\lambda_1,\lambda_2>0$ such that $p-q=\lambda_1 f, q-r=\lambda_2
 f$, and then $p-r=(\lambda_1+\lambda_2) f$, whence $p\succ r$.
 \item[\ref{A2}.] $p\succ q\Rightarrow \exists \lambda>0$ such that
 $p-q=\lambda f$. Given $r\in{\mathcal H}$ and $\beta\in(0,1]$,
 $\beta p+(1-\beta) r\succ \beta q+(1-\beta) r$, since
 $\beta(p-q)=\beta\lambda f$. The converse is analogous.
 \item[\ref{A3}.] We apply Proposition~\ref{pr:charac-archimedean}. Given $f_1,f_2\in\cone $, there are
 $\lambda_1,\lambda_2>0$ such that $f_1=\lambda_1 f, f_2=\lambda_2
 f$. Then there are $\beta_1,\beta_2\in(0,1)$ such that
 $\beta_1\lambda_1-(1-\beta_1)\lambda_2>0$ and
 $\beta_2\lambda_2-(1-\beta_2)\lambda_1>0$, from which we deduce
 that $\beta_1 f_1-(1-\beta_1)f_2$ and $\beta_2 f_2-(1-\beta_2)f_1$
 belong to $\cone $.
\end{itemize}

To see that~\ref{A3'} may not hold, it suffices to consider
$p,q,r,s\in{\mathcal H}$ such that $p\succ q$ and $r\nsucc s$. Then applying Proposition~\ref{pr:charac-archimedean}, there should be $\beta\in(0,1)$ such that $\beta (p-q)+(1-\beta)(r-s)\in\cone$. This implies that there is $\lambda>0$ such that $r-s=\lambda f$, a contradiction. $\lozenge$
\end{example}

Next, assume that relation $\succ$ contains a worst outcome. This makes all the Archimedean conditions we have introduced so far collapse into a single notion (but note that~\ref{wA3} needs to be equipped with an additional assumption of positivity):
\begin{proposition}\label{pr:equiv-Archs-z} Let $\succ$ be a coherent relation with worst outcome $z$, $\cone$ the corresponding cone and $\rdesirs\coloneqq\pi(\cone)$. Let $\lpr$ be the coherent lower prevision induced by $\rdesirs$ through~\eqref{eq:lprR}. Then the following are equivalent:
\begin{itemize}
\item[(i)] Relation $\succ$ satisfies~\ref{A3'}. \item[(ii)] Relation $\succ$ satisfies~\ref{A3}.
\item[(iii)] Relation $\succ$ satisfies~\ref{wA3} and $\lpr(f)>0$ for all $f\in\gambles^+(\pspace\times\values)$.
\end{itemize}
\end{proposition}

\subsection{Existence of Archimedean extensions in the finite case}

Note that at this point we do not know yet whether there are
coherent Archimedean extensions to the worst outcome of a given
coherent preference relation $\succ$. In particular we would like
that there were Archimedean extensions that do not introduce
informative assessments, that is, assessments other than those
necessary for the extension to be weakly Archimedean. We show next
that these extensions may not always exist whenever $\pspace$ is
finite; the case of infinite $\pspace$ is left as an open problem.

Showing that there exists an Archimedean extension with worst
outcome is equivalent to proving that it is possible to include set
$\rdesirs$, given in~\eqref{eq:extR}, in a coherent set of strictly
desirable gambles. We are going to show that, when a coherent
preference relation $\succ$ satisfies the strong Archimedean
condition~\ref{A3'}, and $\pspace$ is indeed finite, then the set
$\rdesirs$ of desirable gambles it induces has a proper superset of
strictly desirable gambles.

\noindent It is interesting to note that, even when we can build an
open superset of $\rdesirs$, it could be that any such superset
introduces new preferences in the original space, as the following
example shows:

\begin{example}
Consider $\pspace\coloneqq\{\omega\},\values\coloneqq\{x_1,x_2,x_3\}$ and let
$\succ$ be the binary relation on ${\mathcal H}\times{\mathcal
H}$ given by
\begin{equation*}
 p \succ q \Leftrightarrow(\exists \lambda >0) p-q=\lambda h,
\end{equation*}
where gamble $h$ is given by $h(\omega,x_1)\coloneqq1, h(\omega,x_2)\coloneqq-1,
h(\omega,x_3)\coloneqq0$. Then $\succ$ is a preference
relation, and the set $\cone$ it induces is given by
\begin{equation*}
 \cone=\{(a,-a,0): a>0\},
\end{equation*}
where in order to simplify the notation we use the vector
$(\alpha_1,\alpha_2,\alpha_3)$ to refer to
$(f(\omega,x_1),f(\omega,x_2),f(\omega,x_3))$. It follows from
Example~\ref{ex:stronger} that this preference relation
satisfies axioms~\ref{A1},~\ref{A2} and~\ref{A3}, but not~\ref{A3'}.

Now, if $\rdesirs_1$ is a set of strictly desirable gambles that
includes the natural extension $\rdesirs$ of $\cone$ and
$\lpr_1$ is the coherent lower prevision it induces, then it must be
$\lpr_1(f)>0$ for every $f\in\cone$. This means that the
credal set $\solp(\lpr_1)$ is a closed subset of
\begin{equation*}
 \solp\coloneqq\{\pr: (\forall f\in\cone)\pr(f)>0\}.
\end{equation*}
Let $\pr$ be a linear prevision on $\gambles(\pspace\times\values)$, and
let $p_1\coloneqq\pr(\{\omega,x_1\}), p_2\coloneqq\pr(\{\omega,x_2\}),p_3\coloneqq\pr(\{\omega,x_3\})$. Then it follows
that $\pr(f)>0$ for all $f\in\cone$ if and only if $p_1>p_2$.

Let us show now that for any closed set
$\solp(\lpr_1)\subseteq\solp$, the set $\{f \in {\mathcal A}:
\lpr_1(f)>0\}$ is a strict superset of $\cone$, from which it
will follow that the set $\rdesirs_1$ of strictly desirable gambles
associated with $\lpr_1$ will induce more preferences than those
encompassed by $\succ$.

To see that this is indeed the case, assume ex-absurdo that $\{f \in
{\mathcal A}: \lpr_1(f)>0\}$ coincides with $\cone$. Then given
$\delta>0$, for the gamble $g\coloneqq(1+\delta,-1,-\delta)$ it must
hold that $\lpr_1(g)\leq 0$. This implies that there is some $\pr\in
\solp(\lpr_1)\subseteq\solp$ such that $0\geq
\pr(g)=p_1-p_2+\delta(p_1-p_3)$. Since we know that $p_1>p_2$, it
must be the case that $p_1< p_3$, whence $\delta\geq
\frac{p_1-p_2}{p_3-p_1}$. Moreover, since we can make the positive
number $\delta$ as close to 0 as we want, this means that for every
natural number $n$ there is some $\pr_n\coloneqq(p_1^n,p_2^n,p_3^n)$
in $\solp(\lpr_1)$ with $p_3^n-p_1^n>0$ such that
\begin{equation}\label{eq:aux-sequence}
 \frac{1}{n}\geq \frac{p_1^n-p_2^n}{p_3^n-p_1^n}.
\end{equation}
But the set $\solp(\lpr_1)$ is compact in the metric space
associated with the Euclidean topology, so it is also sequentially
compact. Thus, the sequence $(\pr_n)_{n\in\nats}$ has a convergent subsequence
to some $\pr'\coloneqq(p_1',p_2',p_3')$. From
Eq.~\eqref{eq:aux-sequence} it follows that $p_1'-p_2'=0$, and this
is a contradiction with $\pr'\in\solp$.

This shows that the set of strictly desirable gambles associated
with any closed subset of $\solp$ will include some gamble in
${\mathcal A}\setminus\cone$ and therefore will not induce
the same preferences. $\lozenge$
\end{example}

One issue in the example above is that the preference relation
$\succ$ satisfies~\ref{A3} but not~\ref{A3'}. Indeed, when
$\succ$ is strongly Archimedean, we can always guarantee the
existence of an open superset of $\rdesirs$, as the following
theorem shows:

\begin{theorem}\label{th:existence-open-superset}
Let $\succ$ be a coherent preference relation on ${\mathcal
H}\times{\mathcal H}$ satisfying~\ref{A3'}. Assume $\pspace$ is
finite, and let $\rdesirs$ be given by~\eqref{eq:extR}. Then
there is a set of strictly desirable gambles $\rdesirs_1$ that
includes $\rdesirs$. 
\end{theorem}

Moreover, under strong Archimedeanity there are cases where we can
find sets of strictly desirable gambles that include $\rdesirs$ and
that introduce no new preferences in the original space:

\begin{example}
Consider $\pspace\coloneqq\{\omega\},\values\coloneqq\{x_1,x_2,x_3\}$ and let
$\succ$ be the binary relation on ${\mathcal H}\times{\mathcal
H}$ given by
\begin{equation*}
 f \succ g \Leftrightarrow f(\omega,x_1)<g(\omega,x_1),
 f(\omega,x_2)<g(\omega,x_2).
\end{equation*}
Then $\succ$ is a coherent preference relation, and the set
$\cone$ it induces is given by
\begin{equation*}
 \cone=\{(a,b,-a-b): \max\{a,b\}<0\}.
\end{equation*}
Now, if $\rdesirs_1$ is a set of strictly desirable gambles that
includes the natural extension $\rdesirs$ of $\cone$ and
$\lpr_1$ is the coherent lower prevision it induces, then it must be
$\lpr_1(f)>0$ for every $f\in\cone$. This means that the
credal set $\solp(\lpr_1)$ is a closed subset of
\begin{equation*}
 \solp\coloneqq\{\pr: (\forall f\in\cone)\pr(f)>0\}.
\end{equation*}
Let $\pr$ be a linear prevision on $\pspace\times\values$, and
let $p_1\coloneqq\pr(\{\omega,x_1\}), p_2\coloneqq\pr(\{\omega,x_2\})$. Then for every
$a,b\in\mathbb{R}$, if we take $f$ on $\pspace\times\values$ given
by $f(\omega,x_1)\coloneqq a, f(\omega,x_2)\coloneqq b, f(\omega,x_3)\coloneqq-a-b$, it holds
that $\pr(f)=a(2p_1+p_2-1)+ b(2p_2+p_1-1)$. From this it
follows that $\pr\in\solp$ if and only if
\begin{equation*}
 \min\{2p_1+p_2-1,2p_2+p_1-1\}<0, \max\{2p_1+p_2-1,2p_2+p_1-1\}\leq
 0.
\end{equation*}

Let us show now that there is a closed convex set
$\solp(\lpr_1)\subseteq\solp$ such that the set $\{f \in {\mathcal
A}: \lpr_1(f)>0\}$ agrees with $\cone$, from which it
will follow that the set $\rdesirs_1$ of strictly desirable gambles
associated with $\lpr_1$ induces the same preferences on the
original space as relation $\succ$.

To see that this is indeed the case, consider
$\apr_1\coloneqq(q_1^1,q_2^1,1-q_1^1-q_2^1),\apr_2\coloneqq(q_1^2,q_2^2,1-q_1^2-q_2^2)$
in $\solp$ such that
\begin{equation*}
 2q_1^1+q_2^1-1<0=2q_2^1+q_1^1-1 \ \text{ and } \
 2q_2^2+q_1^2-1<0=2q_1^2+q_2^2-1.
\end{equation*}
Then given $f\coloneqq(a,b,-a-b)\in{\mathcal A}\setminus\cone$, it
holds that $\max\{a,b\}\geq 0$. Then
\begin{align*}
 &a\geq 0 \Rightarrow \apr_1(f)=a(2q_1^1+q_2^1-1)\leq 0, \\
 &b\geq 0 \Rightarrow \apr_2(f)=b(2q_2^2+q_1^2-1)\leq 0,
\end{align*}
and as a consequence given $\lpr_1\coloneqq\min\{\apr_1,\apr_2\}$,
it holds that $\lpr_1(f)\leq 0$ for every $f\in{\mathcal
A}\setminus\cone$. Thus, $\solp(\lpr_1)$ is a closed convex subset
of $\solp$ satisfying that $\{f\in{\mathcal A}:\lpr_1(f)>0\}=\cone$.
$\lozenge$
\end{example}

It is an open problem at this stage to determine if this is the case
for any coherent preference relation that is strongly
Archimedean. Our conjecture is that this will hold, because of the
following result:

\begin{proposition}\label{pr:str-arch-open}
Assume $\pspace$ is finite, and let $\succ$ be a coherent preference
relation on ${\mathcal H}\times{\mathcal H}$. Then $\succ$
satisfies~\ref{A3'} if and only if $\cone$ is an open subset
of ${\mathcal A}$.
\end{proposition}

\section{Decomposition and completeness of preferences}\label{sec:decompose}
Thanks to the previous results, we have the possibility to represent
coherent preference relations equivalently through coherent sets of
desirable gambles in $\gambles(\pspace\times\valuesn)$. Now we
exploit this equivalence in order to analyse and discuss the most
important special cases that allow us to simplify a
decision-theoretic problem through some type of decomposition of
preferences in a part made of beliefs and another of perceived
values (of consequences).

In fact, without any further assumption, a decision-theoretic
problem, represented by a coherent set $\rdesirs$ in
$\gambles(\pspace\times\valuesn)$, cannot actually be decomposed
while maintaining the same information represented by $\rdesirs$.
This is not to say that we cannot profitably deal with the problem:
we can for instance compute lower and upper (possibly conditional)
previsions for any gamble in $\gambles(\pspace\times\valuesn)$; we
can compute in particular probabilities separately from utilities
and vice versa. But the point is that we cannot represent your preferences by coupling two separate models for belief and
value unless we do some assumption of the type that we illustrate in
the next section.

\subsection{State independence}\label{sec:si}

Let us start by the following definition of irrelevance for coherent sets of gambles.
\begin{definition}[{\bf $\pspace$-$\values$~irrelevant product for gambles}]\label{def:irr-prod}
A coherent set of gambles $\rdesirs$ on $\pspace\times\values$ is
called an \emph{$\pspace$-$\values$~irrelevant product} of its
marginal sets of gambles $\rdesirs_{\pspace},\rdesirs_{\values}$ if
it includes the set
\begin{equation}\label{eq:finite-ext-conditionals}
\rdesirs|\pspace\coloneqq\{h\in\gambles(\pspace\times\values):
(\forall
\omega\in\pspace)h(\omega,\cdot)\in\rdesirs_{\values}\cup\{0\}\}\setminus\{0\}.
\end{equation}
\end{definition}
\noindent These products have been introduced, in another context
and for lower previsions, in~\cite{miranda2015b}. The rationale of the definition is the following. First, the requirement that $h(\omega,\cdot)\in\rdesirs_\values\cup\{0\}$ is there to say that the inferences conditional on $\{\omega\}$ encompassed by $\rdesirs|\pspace$ should yield the marginal set $\rdesirs_\values$. This is just a way to formally state that $\omega$ is irrelevant to $\values$. The same is repeated for every $\omega$, so that $\rdesirs|\pspace$ can be regarded as being born out of aggregating all the irrelevant conditional sets. Given that $\pspace$ is possibly infinite, it should not be surprising that such a definition of $\rdesirs|\pspace$ entails an assumption of conglomerability. We do it because the minimal set $\rdesirs$ we would obtain without this assumption may be too weak to
produce informative preferences (see~\cite[Section~5.1]{miranda2015b} for more details). Finally, that $\rdesirs$ contains $\rdesirs|\pspace$ is imposed to make sure that $\rdesirs$ is a model coherent with the irrelevance of beliefs on values.

In the context of preferences, if $\rdesirs$ satisfies Definition~\ref{def:irr-prod} we say it models
\emph{state independent preferences}. The least-committal among
these models is trivially given by the following:
\begin{proposition}\label{prop:weakest-IP-gmb}
Given two marginal coherent sets of desirable gambles $\rdesirs_{\pspace},\rdesirs_{\values}$, their smallest $\pspace$-$\values$~irrelevant product is given by the marginal extension of $\rdesirs_{\pspace}$ and $\rdesirs|{\{\omega\}}\coloneqq\rdesirs_{\values}$ (for all $\omega\in\pspace$), that is:
\begin{equation}\label{eq:irrel-product2}
\hat{\rdesirs}\coloneqq\{g+h:
g\in\rdesirs_{\pspace}\cup\{0\},
h\in\rdesirs|\pspace\cup\{0\}\}\setminus\{0\}.
\end{equation}
\end{proposition}

Definition~\ref{def:irr-prod} and Eq.~\eqref{eq:irrel-product2}
represent our main proposal to model state-independent preferences,
for two reasons. One the one hand, as discussed above, we think that
it is important that $\rdesirs|\pspace$ is included in the set that
models your preferences. On the other, it follows from our
interpretation of such a set, in Section~\ref{sec:interpretation},
that updating is meaningful only if applied to elements of
$\pspace$, and therefore state independence should be an asymmetric
notion (that is, we should not consider updating on $\values$).

On the other hand, symmetric notions are usually considered in the
literature and hence we introduce some symmetric notions as well so
as to lay the ground for a connection with former proposals (which
we shall deepen in Section~\ref{sec:si-clps} when we consider the
special case of lower previsions):
\begin{definition}[{\bf Independent product for gambles}]
A coherent set of gambles $\rdesirs$ on $\pspace\times\values$ is
called an \emph{independent product} of its marginal sets of gambles
$\rdesirs_{\pspace},\rdesirs_{\values}$ if it includes
$\rdesirs|\pspace$ and $\rdesirs|\values$, where
\begin{equation}\label{eq:RcondV}
\rdesirs|\values\coloneqq\{h\in\gambles(\pspace\times\values):
(\forall
x\in\values)h(\cdot,x)\in\rdesirs_{\pspace}\cup\{0\}\}\setminus\{0\}.
\end{equation}
\end{definition}
\noindent These products were studied for finite spaces and several
variables in \cite{cooman2012}. Our formulation given above is an
extension of the proposal in that paper. Note that the above
conditions hold in particular when $\rdesirs$ includes
$\rdesirs|\pspace$ and
\begin{align}\label{eq:mutual-irrelevance}
 &(\forall \omega\in \pspace)\rdesirs|{\{\omega\}}=\rdesirs_{\values}, \\
 &(\forall x\in \values)\rdesirs|{\{x\}}=\rdesirs_{\pspace};
\end{align}
this is closer to the formulation in \cite{cooman2012}, and it is
not required in general in the definition above.

An independent product models state-independent preferences given
that it includes the weakest $\pspace$-$\values$~irrelevant product.
The smallest, or least-committal, independent product of two sets is
given by the following proposition, whose proof is analogous to that
of Proposition~\ref{prop:weakest-IP-gmb}:
\begin{proposition}[{\bf Independent natural extension}]
Let $\rdesirs_{\pspace},\rdesirs_{\values}$ be two marginal coherent
sets of desirable gambles and let $\rdesirs|\pspace,
\rdesirs|\values$ be given by
Eqs.~\eqref{eq:finite-ext-conditionals}--\eqref{eq:RcondV}. Then the \emph{independent natural extension} of $\rdesirs_{\pspace},\rdesirs_{\values}$ is given by:
\begin{equation*}
 \rdesirs_{\pspace}\otimes\rdesirs_{\values}\coloneqq\{g+h:g\in\rdesirs|\values\cup\{0\},h\in\rdesirs|\pspace\cup\{0\}\}
 \setminus\{0\}.
\end{equation*}
\end{proposition}
\noindent It is the least committal
model, in other words, for which your beliefs on $\pspace$ are independent of
your values of $\values$, and that moreover satisfies conglomerability with respect to $\pspace$.

We need the independent natural extension in particular to be able to define another type of independent product, which is more informative than the former:
\begin{definition}[{\bf Strong product for gambles}] The \emph{strong} (\emph{independent}) \emph{product} of marginal coherent sets
$\rdesirs_{\pspace},\rdesirs_{\values}$ is defined as
\begin{equation*}
 \rdesirs_{\pspace}\boxtimes\rdesirs_{\values}\coloneqq\cap\{\desirs_{\pspace}\otimes\desirs_{\values}:
 \rdesirs_\pspace\subseteq\desirs_{\pspace} \text{ maximal and }
 \rdesirs_{\values}\subseteq\desirs_{\values} \text{ maximal}\}.
\end{equation*}
\end{definition}
\noindent The strong product can be regarded as the extension of a set of stochastically independent models to desirable gambles. We shall see more about the strong product in Section~\ref{sec:lprevs}, when we see the analogous notions in terms of coherent lower previsions.

\subsection{State dependence}\label{sec:sd-gmb}

We say that a set of desirable gambles $\rdesirs$ models
state-dependent preferences when it does not satisfy the conditions
in Definition~\ref{def:irr-prod}, i.e., when it does not include the
set $\rdesirs|\pspace$ derived from its marginal $\rdesirs_\values$
and the assumption of $\pspace$-$\values$ irrelevance. This may
arise for instance when the conditional sets of gambles
$\rdesirs|\{\omega\}, \rdesirs|\{x\}$ that $\rdesirs$ induces do not
satisfy Eq.~\eqref{eq:mutual-irrelevance}.

\subsection{Completeness}

So far we have discussed the general case of coherent sets of gambles $\rdesirs$ without discussing in particular the so-called complete or maximal coherent sets for which it holds that $f\notin\rdesirs\Rightarrow-f\in\rdesirs$ for all $f\in\gambles,f\neq0$. These correspond to precise models for which there is never uncertainty as to whether you accept or reject a gamble, and that are closely related to linear previsions. Given the equivalence between coherent sets of gambles in $\gambles(\pspace\times\valuesn)$ and coherent preferences in $\actsy\times\actsy$, discussing the case of maximal sets is a tantamount to discussing the case of complete preferences.

\begin{definition}[{\bf Completeness of beliefs and values for gambles}]
A coherent set of desirable gambles
$\rdesirs\subseteq\gambles(\pspace\times\valuesn)$ is said \emph{to
represent complete beliefs} if $\rdesirs_{\pspace}$ is maximal. It is said
\emph{to represent complete values} if $\rdesirs_{\values}$ is maximal. Finally, if $\rdesirs$ is maximal, then it is said \emph{to represent complete preferences}.
\end{definition}

The rationale of this definition is straightforward: when
we say, for instance, that $\rdesirs$ represents complete beliefs,
we mean that you are never undecided between two options that are concerned only with $\pspace$. The situation is analogous in the case of complete values. Finally, the case of complete preferences, the definition is just the direct application of the definition of maximality. Not surprisingly, if a coherent set of desirable gambles represents
complete preferences, then it represents both complete beliefs and
complete values. This follows immediately
from~\cite[Proposition~22]{cooman2012}. On the other hand, the
converse does not hold in general: we can construct a non-maximal
coherent set of desirable gambles
$\rdesirs\subseteq\gambles(\pspace\times\values)$ that represents
both complete beliefs and values, and this also in the case of
finite $\pspace$. See~\cite[Appendix~A1]{cooman2012} for an example.

Note that a set of complete preferences need not represent
state-independent preferences: to see this, it suffices to consider
the set
\begin{equation*}
\{g+h: g\in\rdesirs_\pspace\cup\{0\},
h(\omega,\cdot)\in\rdesirs|\{\omega\}\cup\{0\}, \ \forall
\omega\in\pspace\}\setminus\{0\}
\end{equation*}
where $\rdesirs_\pspace$ and $\rdesirs|{\{\omega\}}$ are maximal for
all $\omega\in\pspace$ but $\rdesirs|\{\omega\}\neq\rdesirs_\values$
for some $\omega$: we end up with a coherent set of state-dependent
preferences, and any maximal superset will represent complete
preferences that will be state-dependent (since it will not include
the set $\rdesirs|\pspace$ given by
Eq.~\eqref{eq:finite-ext-conditionals} due to maximality).

\section{In terms of lower previsions}\label{sec:lprevs}

In this section, we shall consider the particular case where your
preferences are modelled by means of a coherent lower prevision
$\lpr$ on $\gambles(\pspace\times\values)$. 

\subsection{State independent preferences}\label{sec:si-clps}
Let us define the products we shall be concerned with.
\begin{definition}[{\bf $\pspace$-$\values$~irrelevant product for lower previsions}]\label{def:irr-prod-clp} Given marginal coherent lower previsions $\lpr_{\pspace},\lpr_{\values}$, let
\begin{equation*}
\lpr(f|\{\omega\})\coloneqq\lpr_\values(f(\omega,\cdot)),
\end{equation*}
for all $f\in\gambles(\pspace\times\values)$, and
\begin{equation*}\label{eq:cong-clp}
\lpr(\cdot|\pspace)\coloneqq\sum_{\omega\in\pspace}I_{\{\omega\}}\lpr(\cdot|\{\omega\}).
\end{equation*}
Then a coherent lower prevision $\lpr$ on
$\gambles(\pspace\times\values)$ is called an
\emph{$\pspace$-$\values$~irrelevant product} of
$\lpr_{\pspace},\lpr_{\values}$ if
\begin{equation*}
\lpr\geq\lpr_\pspace(\lpr(\cdot|\pspace)).
\end{equation*}
In this case we also say that $\lpr$ models state-independent preferences.
\end{definition}
\noindent Here $\lpr(\cdot|\pspace)$ plays the role that was of
$\rdesirs|\pspace$ in desirability; it is a mathematical tool that
conglomerates all the conditional information. The concatenation
$\lpr_\pspace(\lpr(\cdot|\pspace))$ is a marginal extension: in
particular, it is the least-committal coherent
$\pspace$-conglomerable model built out of the given marginals and
the assessment of irrelevance that defines $\lpr(\cdot|\pspace)$
through $\lpr_\values$. Every coherent lower prevision that
dominates $\lpr_\pspace(\lpr(\cdot|\pspace))$ is compatible with the
irrelevance assessment but is also more informative than the
marginal extension. Let us remark it with the following trivial
proposition:
\begin{proposition}\label{sec:ip-lower}
Given two marginal coherent lower previsions $\lpr_{\pspace},\lpr_{\values}$, their smallest $\pspace$-$\values$~irrelevant product is given by the marginal extension
\begin{equation*}
\lpr_\pspace(\lpr(\cdot|\pspace)).
\end{equation*}
\end{proposition}
As in the case of desirability, we can define more precise
irrelevant products. Among them there are the independent
products:\footnote{The following definition is equivalent to the one
established in \cite[Definition~4]{miranda2015b} in the context of
this paper, where one of the referential spaces is finite; we refer
to \cite{miranda2015b} for a more general study of independent
products.}
\begin{definition}[{\bf Independent product for lower previsions}]\label{def:ind-prod-clp} Consider marginal coherent lower previsions $\lpr_{\pspace},\lpr_{\values}$. We say that a coherent lower prevision $\lpr$ on $\gambles(\pspace\times\values)$ is an \emph{independent product} of $\lpr_{\pspace},\lpr_{\values}$ if it is both an $\pspace$-$\values$~irrelevant product and an $\values$-$\pspace$~irrelevant product (this is defined analogously to the previous one).
\end{definition}
The weakest independent product is again called the independent natural extension (see \cite{cooman2012,miranda2015b} for more information). The, more informative, strong (independent) product is simply defined as a lower envelope of stochastic products:
\begin{definition}[{\bf Strong product for lower previsions}] Let $\lpr_{\pspace},\lpr_{\values}$ be two marginal coherent lower previsions. Their \emph{strong} (\emph{independent}) \emph{product} is defined for every $f\in\gambles(\pspace\times\values)$ as
\begin{equation*}
 \lpr_{\pspace}\boxtimes\lpr_{\values}(f)\coloneqq\min\{\pr_{\pspace}(\pr_{\values}(f|\pspace)):
 \pr_{\pspace}\geq \lpr_{\pspace}, \pr_{\values}\geq \lpr_{\values}\},
\end{equation*}
where
\begin{align}
\pr_\values(\cdot|\pspace)&\coloneqq\sum_{\omega\in\pspace}I_{\{\omega\}}\pr(\cdot|\{\omega\}) \text{ and }\label{eq:cong-cond-lp}\\
(\forall
f\in\gambles(\pspace\times\values))\pr(f|\{\omega\})&\coloneqq\pr_\values(f(\omega,\cdot)).\notag
\end{align}
\end{definition}

\begin{remark}[{\bf State independence and dynamic consistency}]\label{rem:si-indep}
It is important at this point to consider the relation of state
independence to dynamic consistency (as introduced in
Remark~\ref{rem:dc1}). In particular, if state-independent
preferences are modelled by the $\pspace$-$\values$~irrelevant
product of the marginals, then the weakest combination of the
marginal $\lpr_\pspace$ and the conditional information
$\lpr(\cdot|\pspace)$ (obtained through irrelevance) is just the
marginal extension $\lpr\coloneqq\lpr_\pspace(\lpr(\cdot|\pspace))$,
as stated by Proposition~\ref{sec:ip-lower}. This means that the
marginal extension $\lpr$ implements a complete separation of utilities from
probabilities:\footnote{Not so the other way around: remember that we are focusing on an asymmetric notion.} not only states of nature are irrelevant to utilities in this
case, as it follows from $\lpr(\cdot|\{\omega\})$ being constant on
$\omega\in\pspace$, but in addition the joint model $\lpr$ of
preferences can be recovered through the marginal and the
conditional models alone: $\lpr=\lpr_\pspace(\lpr(\cdot|\pspace))$,
thanks to dynamic consistency. Taking into account
Eq.~\eqref{eq:dynamic-consistency}, the irrelevant product
corresponds to the lower envelope of the set
\begin{equation}\label{eq:dynamic-consistency-irrelevance}
\solp_1\coloneqq\{\pr_{\pspace}(\pr(\cdot|\pspace)):
 \pr_{\pspace}\in\solp_{\pspace}, (\forall \omega\in\pspace)\pr(\cdot|\{\omega\})
 \in\solp_{\values}\},
\end{equation}
where we are denoting by $\solp_{\values}$ the credal set associated
with $\lpr_{\values}$, which, due to the irrelevance assumption,
agrees with the one determined by $\lpr(\cdot | \{\omega\})$ for every
$\omega\in\pspace$. Note that in the set $\solp_1$ the linear prevision
$\pr(\cdot |\{\omega\})$ may vary with different $\omega$, because
the irrelevance assumption is made on the lower envelopes
$\lpr(\cdot |\{\omega\})$ only.

In case we use an $\pspace$-$\values$~irrelevant product $\lpr$ that is not minimal, what we get is the weaker relation $\lpr\geq\lpr_\pspace(\lpr(\cdot|\pspace))$. This still implements the idea that states be irrelevant to utilities; however, in order to recover $\lpr$ from $\lpr_\pspace$ and $\lpr(\cdot|\pspace)$, we need in addition to describe the specific
way in which $\lpr$ dominates $\lpr_\pspace(\lpr(\cdot|\pspace))$. This information `logically' connects the linear previsions that dominate the marginal and the conditionals, thus preventing these models from being completely separate.

A particularly important dominating model is the strong product, so it is useful to address also this case in some detail. Observe first that if $\lpr$ is a linear prevision with
marginals $\pr_{\pspace},\pr_{\values}$, then it corresponds to an
irrelevant product if and only if it agrees with the marginal extension $\pr_{\pspace}(\pr(\cdot|\pspace))$, where
$\pr(f|\{\omega\})\coloneqq\pr_{\values}(f(\omega,\cdot))$ for every $\omega\in\pspace$.
This is also the outcome of a trivial assumption of dynamic consistency. But it helps to give a clearer view of the meaning
of strong independence: it corresponds to the case where we consider
only linear irrelevant products in set $\solp_1$ above,
so that Eq.~\eqref{eq:dynamic-consistency-irrelevance} becomes
\begin{equation*}
\solp_2\coloneqq\{\pr_{\pspace}(\pr(\cdot|\pspace)):
 \pr_{\pspace}\in\solp_{\pspace}, (\forall \omega, \omega' \in \pspace)\pr(\cdot|\{\omega\})=\pr(\cdot|\{\omega'\})\in\solp_{\values}\}.
\end{equation*}
\noindent Historically, the rationale behind such a choice is that of
sensitivity analysis: under this interpretation, a set such as $\solp_1$ is regarded as a way to represent knowledge about a
`true', but partly unknown, linear model; and once the true model is
supposed to have a certain property, such as that it is an irrelevant product, then all the linear models candidate to be the true
one, which are the elements of $\solp_1$, should satisfy that
property too.

Note moreover that under this interpretation, the strong product does satisfy a minimality property: it is the least-committal
independent product of the marginals for which all the dominating
linear previsions are irrelevant extensions. It is arguable that this is the proper notion of minimality to use under
sensitivity analysis. As a consequence, under this interpretation,
the strong product is a dynamically consistent model and hence it
leads to complete separation too. This would not be the case under the more general behavioural interpretation of coherent lower previsions that we have adopted in this paper.

We end up this remark by pointing out that all these considerations can be made in a very similar way also for sets of desirable gambles; we have opted to describe them in the case of lower previsions only to keep the discussion more accessible. $\lozenge$
\end{remark}

\subsection{State dependent preferences}

Similarly to the case of desirable gambles, when we model your
beliefs by means of coherent lower previsions, we define state
dependence as the lack of state independence: this means that if we
consider a coherent lower prevision $\lpr$ with marginals
$\lpr_{\pspace},\lpr_{\values}$, it is said to model
\emph{state-dependent preferences} when it does not dominate the
concatenation $\lpr_{\pspace}(\lpr(\cdot|\pspace))$, where
$\lpr(\cdot|\pspace)$ is derived from $\lpr_{\values}$ by an
assumption of epistemic irrelevance.

One particular case where we may obtain state dependence is when
$\lpr_{\pspace}(\{\omega\})>0$ for every $\omega$ but the
conditional natural extension $\lnex(\cdot|\pspace)$ of $\lpr$ does
not coincide with the one that $\lpr_{\values}$ induces by means of
epistemic irrelevance. Then we may have that
$\lpr\geq\lpr_{\pspace}(\lnex(\cdot|\pspace))$ but that
$\lpr\ngeq\lpr_{\pspace}(\lpr(\cdot|\pspace))$. We refer to
\cite[Example~8]{miranda2015b} for an example.

\subsection{Completeness}

The definition of completeness for lower previsions is a rephrasing of that for desirable gambles:
\begin{definition}[{\bf Completeness of beliefs and values for lower previsions}]
A coherent lower prevision $\lpr$ on $\gambles(\pspace\times\valuesn)$ is said \emph{to
represent complete beliefs} if its marginal $\lpr_{\pspace}$ is linear. It is said
\emph{to represent complete values} if its marginal $\lpr_{\values}$ is linear. Finally, if $\lpr$ is
linear, then it is said \emph{to represent complete preferences}.
\end{definition}

There are a number of equivalent ways in which we can characterise the linearity of previsions in terms of the corresponding set of desirable gambles:
\begin{proposition}\label{pr:charac-linear}
Let $\lpr$ be a coherent lower prevision on $\gambles$ and $\rdesirs$ its corresponding coherent set of strictly desirable gambles. The
following are equivalent:
\begin{itemize}
 \item[(i)] $\lpr$ is a linear prevision.
 \item[(ii)] If $f\notin\rdesirs$ then $\eps-f\in\rdesirs$ for all $\eps>0$.
 \item[(iii)] $\rdesirs$ is negatively additive, meaning that $f,g\notin\rdesirs\Rightarrow(\forall\eps>0) f+g-\eps\notin\rdesirs$.
 \item[(iv)] For any event $A$, any $\alpha\in(0,1)$ and real numbers $\mu_1>\mu_2$, either $\alpha \mu_1+(1-\alpha)\mu_2 \succ \mu_1 I_A+\mu_2 I_{A^c}$ or
 $\mu_1 I_A+\mu_2 I_{A^c} \succ \alpha' \mu_1+(1-\alpha')\mu_2$ for every $\alpha'<\alpha$, where $\succ$ is the coherent preference relation on gambles associated with $\rdesirs$ by Definition~\ref{def:pref-gmb}.
\end{itemize}
\end{proposition}

The equivalence of (i) and (ii) has been proven
in~\cite[Proposition~6]{miranda2010c}. It clearly shows that the
linearity of previsions, once written in terms of gambles, is
mathematically very similar to the maximality of a set of gambles.
The characterisation of (i) through negative additivity in (iii) is
essentially obtained through a rewriting of (ii); yet, in this form
it shows its resemblance to the notion of negative transitivity that
has been used to deduce the completeness of
values~\cite[Axiom~A.6]{galaabaatar2013}. Finally, the
characterisation of (i) in (iv) is the analog
of~\cite[Axiom~A.7]{galaabaatar2013} used to derive the completeness
of beliefs. So it is interesting to see that two notions that have
been used with separate aims in the literature are actually the same
notion once we represent utilities by coherent lower previsions
besides probabilities. Indeed, we can focus on any of those
formulations to immediately deduce characterisations for all the
cases of completeness in preferences. By choosing negative
additivity, we obtain the following:

\begin{proposition}\label{pr:precise-utilities}
Consider a coherent preference relation $\succ$ on gambles
represented by a coherent lower prevision $\lpr$ on $\gambles(\pspace\times\values)$, whose corresponding coherent set of gambles is denoted by $\rdesirs$. Let $\rdesirs_\pspace,\rdesirs_\values$ denote its marginals. Then: 
\begin{itemize}
 \item[(i)] $\lpr$ represents complete beliefs $\Leftrightarrow$ $\rdesirs_\pspace$ is
 negatively additive.
 \item[(ii)] $\lpr$ represents complete values $\Leftrightarrow$ $\rdesirs_\values$ is
 negatively additive.
\item[(iii)] $\lpr$ represents complete preferences $\Leftrightarrow$ $\rdesirs$ is negatively additive.
\end{itemize}
\end{proposition}

\section{State independence means strong product in former proposals}\label{sec:si=sp}

In the previous sections we have discussed the questions of state
independence to some length and proposed ways to address it through
notions of irrelevance or independence for sets of probabilities.
Now it is time to detail how our proposal relates to previous ones
that have been made under incomplete preferences as well. We focus
in particular on Galaabaatar and Karni's work
\cite{galaabaatar2013}, which provides quite a general treatment and
is easier to put in correspondence with ours. They obtain state
independence in particular by means of \cite[Axioms A.4 and
A.5]{galaabaatar2013}. The first of these two axioms is also called
the \emph{dominance axiom} or the \emph{sure thing principle}. On
the other hand, \cite[Axiom A.5]{galaabaatar2013} corresponds to
Eq.~\eqref{eq:A5} in the following:
\begin{theorem}\label{pr:domin-by-str-prod}
Let $\lpr$ be a coherent lower prevision on
$\gambles(\pspace\times\values)$ and let
$\lpr_{\pspace},\lpr_{\values}$ denote the marginals of $\lpr$. For any $\pr_{\values}\geq\lpr_{\values}$, we define
$\pr_{\values}(\cdot|\pspace)$ by means of
Eq.~\eqref{eq:cong-cond-lp}. Then
\begin{equation}\label{eq:A5}
 \lpr\leq \lpr(\pr_{\values}(\cdot|\pspace)) \text{ for every } \pr_{\values}\geq \lpr_{\values}
\end{equation}
if and only if $\lpr \leq \lpr_{\pspace}\boxtimes \lpr_{\values}$.
\end{theorem}
\noindent Thus, it turns out that \cite[Axiom A.5]{galaabaatar2013}
is just imposing that preferences are represented by a coherent
lower prevision $\lpr$ that is at most as precise as the strong
product of the marginal models for probabilities and utilities. In
other words, the greatest coherent lower prevision $\lpr$ with
marginals $\lpr_{\pspace},\lpr_{\values}$ satisfying~\eqref{eq:A5}
is given by their strong product.

Next we give an alternative characterisation of the strong product
as the least informative model satisfying (our reformulation of)
\cite[Axiom A.4]{galaabaatar2013}. For every $\omega\in\pspace$ and
$f\in\gambles(\pspace\times\values)$, let us define the
$\values$-measurable gamble
\begin{eqnarray*}\label{eq:constant-f}
 f^{\omega}:\pspace\times\values &\rightarrow &\reals \\
  (\omega',x) &\mapsto & f(\omega,x).\notag
\end{eqnarray*}
The reformulation of \cite[Axiom A.4]{galaabaatar2013} in the
language of coherent lower previsions is the following:
\begin{equation}\label{eq:A4}
(\forall g,f\in\gambles(\pspace\times\values))((\forall \omega\in\pspace) \lpr(g-f^{\omega})\geq 0 \Rightarrow\lpr(g-f)\geq 0).
\end{equation}
Let us show first of all that this condition can be used to
characterise independent products in the linear case:

\begin{lemma}\label{le:linear-env-A5}
A linear prevision $\pr$ satisfies~\eqref{eq:A4} if and only if it
is an independent product of its marginals. Moreover,
Eq.~\eqref{eq:A4} is preserved by taking lower envelopes.
\end{lemma}

In order to extend Lemma~\ref{le:linear-env-A5} to the imprecise case, we prove
that (i) the result holds when both $\pspace,\values$ are finite,
and that (ii) the case of $\pspace$ infinite can be approximated as a
limit of finite sets. The next proposition addresses the first case.

\begin{proposition}\label{pr:ind-selection}
Let $\pspace,\values$ be finite spaces, and let $\lpr$ be a coherent
lower prevision on $\gambles(\pspace\times\values)$. Then $\lpr$
satisfies~\eqref{eq:A4} if and only if it is a lower envelope of
linear previsions satisfying
\begin{equation}\label{eq:factorising-linear}
(\forall \omega\in\pspace)(\forall x\in\values) \pr(\{(\omega,x)\})=\pr(\{\omega\})\pr(\{x\}).
\end{equation}
\end{proposition}
\noindent It is an open problem at this stage whether a result akin to
Proposition~\ref{pr:ind-selection} holds when one of the spaces is
infinite.

The second case to be dealt with is done  through the next:
\begin{theorem}\label{thm:dom-sp}
Consider a coherent lower prevision $\lpr$ on
$\gambles(\pspace\times\values)$ with marginals
$\lpr_{\pspace},\lpr_{\values}$. If $\lpr$ satisfies
Eq.~\eqref{eq:A4}, then it dominates the strong product
$\lpr_{\pspace}\boxtimes\lpr_{\values}$.
\end{theorem}

We see then that the strong product is the smallest (i.e., most
conservative) coherent lower prevision with given marginals that
satisfies~\eqref{eq:A4}. We have thus two characterisations of the
strong product. As we shall see, if we put them together we can
express the strong product as the \emph{only} coherent lower
prevision that satisfies~\eqref{eq:A5} and~\eqref{eq:A4}
simultaneously.

Before establishing it, we are going to see
that each of the conditions can be used to characterise stochastic
independence in the case of complete preferences. Remember that a coherent lower prevision $\lpr$ on
$\pspace\times\values$ models complete preferences when it is a linear prevision, i.e., when $\lpr(f)=-\lpr(-f)$ for every $f\in\gambles(\pspace\times\values)$. The marginals $\lpr_{\pspace},\lpr_{\values}$ of a linear prevision $\pr$ on $\gambles(\pspace\times\values)$ are also linear previsions
on $\gambles(\pspace),\gambles(\values)$. Conversely, if you have complete preferences and they are compatible with some state independent model, then there is only one
possible state independent model: the concatenation
$\pr_{\pspace}(\pr_{\values}(\cdot|\pspace))$. This coincides with
the only independent product of these marginals, as we see in the
following result, whose proof is omitted since it is a direct consequence of Theorem~\ref{pr:domin-by-str-prod} and Lemma~\ref{le:linear-env-A5}. Note that it holds also for the case where $\pspace$ is infinite.

\begin{corollary}\label{cor:linear}
Let $\pr$ be a linear prevision on $\gambles(\pspace\times\values)$.
The following are equivalent:
\begin{enumerate}
\item[(i)] $\pr$ is the product of its
marginals.
\item[(ii)] $\pr$ satisfies~\eqref{eq:A5}.
\item[(iii)] $\pr$ satisfies~\eqref{eq:A4}.
\end{enumerate}
\end{corollary}

Finally, our main result can be summarised by the following theorem,
which follows immediately from Theorem~\ref{pr:domin-by-str-prod}
and Lemma~\ref{le:linear-env-A5} for the direct implication, and
from Theorems~\ref{pr:domin-by-str-prod} and~\ref{thm:dom-sp} for
the converse one.
\begin{theorem}\label{theo:strong-main}
Let $\lpr$ be a coherent lower prevision on
$\gambles(\pspace\times\values)$. Then $\lpr$ is the strong product
of its marginals if and only if it satisfies~\eqref{eq:A5}
and~\eqref{eq:A4}.
\end{theorem}

Let us briefly comment on this result. What we have obtained is that the notion of state independence is implicitly implemented in \cite{galaabaatar2013} by the strong product. We find this interesting for several reasons:
\begin{itemize}
\item It provides a clear bridge to a well-known notion of independence used in imprecise probability.
\item It also shows that the two conditions (\cite[Axioms A.4 and A.5]{galaabaatar2013}) used to formalise state independence in \cite{galaabaatar2013}, can be employed to that end also when $\pspace$ is infinite.
\item It implies that there are weaker and potentially better notions than the strong product that can be employed to model state independence, such as the weakest independent products of Sections~\ref{sec:si} and~\ref{sec:si-clps}.
\item Turning the viewpoint around, the correspondence with the strong product can be regarded as a way to provide a behavioural interpretation of the strong product (so far, the strong product appears to have been justified only through sensitivity analysis), in particular by regarding it as the weakest model that satisfies the sure thing principle, that is, Eq.~\eqref{eq:A4}.
\end{itemize}

\section{The full Archimedean case}\label{sec:archimedes}

We know from Section~\ref{sec:equiv} that the case of Archimedean
preferences is equivalent to that of a coherent set of strictly
desirable gambles
$\rdesirs\subseteq\gambles(\pspace\times\valuesn)$. Taking into
account the correspondence between strict desirability and lower previsions, we deduce that Archimedean preferences can be equivalently represented by means of
a set of linear previsions.

Despite the usefulness of this outcome, there are important
limitations it is subject to. Too see
this, assume that there is $\omega\in\pspace$ such that
$\lpr(\{\omega\})=0$ (it will often be the case even that
$\upr(\{\omega\})=0$ if $\pspace$ is infinite). Then the conditional
natural extension of $\lpr$ given $\omega$, given by
Eq.~\eqref{eq:cond-natural-extension}, will be vacuous. In other words,
it turns out that your assumption to get to know $\omega$ renders your utilities vacuous. The problem here is that $\lpr$ is too much of an uninformative model to be able to represent
non-vacuous conditionals in the presence of conditioning events with zero probability.

And yet, it is entirely possible that your (unconditional) preferences are represented by a coherent lower prevision $\lpr$ on
$\gambles(\pspace\times\valuesn)$, and at the same time that
your utility model conditional on a state of nature $\omega$ is
given by a conditional lower prevision $\lpr(\cdot|\{\omega\})$ that is
different from the conditional natural extension of $\lpr$ given
$\omega$ (see~\cite[Appendix~F4]{walley1991}). This is to say that despite a pair such as $\lpr,\lpr(\cdot|\pspace)$ is only made up of probabilities and utilities, there is no single coherent lower prevision that is
equivalent to it; that is, there is no set of strictly desirable gambles on $\gambles(\pspace\times\valuesn)$ that is
equivalent to it, and hence there is no Archimedean relation either.
What we argue, in other words, is that there are useful preferences
relations, which are expressed only via (a collection of) sets of linear previsions,
that are not characterised by Axiom~\ref{wA3} and therefore that the
Archimedean axiom is too restrictive.

If we focus the attention on the case that $\pspace$ is finite, we immediately have a characterisation of the preference relations that can be assessed by relying only on sets of linear previsions:
\begin{definition}[{\bf Fully Archimedean preferences---finite $\pspace$}] Given a coherent preference relation $\succ$ on $\actsy\times\actsy$, with $\pspace$ finite, we say that \emph{relation $\succ$ is fully Archimedean} if the corresponding coherent set of desirable gambles $\rdesirs\subseteq\gambles(\pspace\times\valuesn)$ satisfies the following condition:
\begin{equation}
f\in\rdesirs\Rightarrow(\exists\eps>0)I_{S(f)}(f-\eps)\in\rdesirs,\label{eq:cA}
\end{equation}
where $S(f)\coloneqq\{(\omega,x)\in\pspace\times\valuesn:f(\omega,x)\neq0\}$ is the \emph{support} of $f$.
\end{definition}
In fact, it has been shown in~\cite[Theorem~15]{miranda2010c} that a
coherent set of desirable gambles $\rdesirs$ satisfies~\eqref{eq:cA}
if and only if it is equivalent to a collection of separately
coherent conditional lower previsions. The remaining coherent sets
of desirable gambles, those that do not satisfy~\eqref{eq:cA}, are
in a sense the `purest' non-Archimedean sets, since there is no way
to represent them equivalently by collections of closed convex sets
of linear previsions. Example~\ref{ex:non-fullA} describes an
instance of this situation.

Note that Eq.~\eqref{eq:cA} can be interpreted quite simply: what is required there is a property analogous to strict desirability but extended to the conditional case, for all events that are supports of gambles in $\rdesirs$. In other words, Eq.~\eqref{eq:cA} is still a continuity property as the Archimedean one, but it is extended so as to be satisfied by all the relevant conditional inferences, not only the unconditional ones, as in the case of~\ref{wA3}.

With regard to the case that $\pspace$ is infinite, it is possible
to generalise~\eqref{eq:cA} to such a situation:

\begin{definition}[{\bf Full strict desirability}]\label{def:cond-str-des}
Let $\rdesirs$ be a coherent set of gambles. We say that it is \emph{fully strictly desirable} if it is the natural
extension of a set of gambles $\rdesirs'$ satisfying
Eq.~\eqref{eq:cA}.
\end{definition}

\noindent The term above means that $\rdesirs$ keeps the same information as a
set of conditional lower previsions. This will be a consequence of
the following result:

\begin{proposition}\label{prop:fsd-char}
Consider a coherent set of desirable gambles $\rdesirs$ on $\gambles(\pspace\times\values)$. Let
\begin{align}\label{eq:cond-gambles}
&\rdesirs_1|B:=\{f\in \rdesirs: f=Bf,\inf_{S(f)} f>0\}\cup\{0\},\notag\\
&\rdesirs_2|B:=\{f \in \rdesirs: f=Bf,(\exists \epsilon>0)
B(f-\epsilon)\in\rdesirs\},\notag\\ &\widetilde\rdesirs|B:=\{f+g:
f\in\rdesirs_1|B, g\in\rdesirs_2|B\},
\end{align}
for every set
$B\subseteq\pspace\times\values$, and
\begin{equation*}
\widetilde\rdesirs:=\cup_{B\subseteq\pspace\times\values}\widetilde\rdesirs|B.
\end{equation*}
Then $\rdesirs$ is the natural extension of $\widetilde\rdesirs$ if
and only if $\rdesirs$ is fully strictly desirable.
\end{proposition}

As a consequence, we have the following:

\begin{theorem}\label{pr:equal-expressivity}
Let $(\lpr(\cdot|B_i))_{i\in I}$ be a family of coherent conditional lower
previsions satisfying
\begin{equation}\label{eq:williams-apl}
(\forall J\subseteq I, |J|<+\infty)(\forall j\in
J)(\forall f_j\in\gambles(\pspace\times\values))\sup_{\omega\in \cup_{j\in J} B_j} \left[\sum_{j\in J}B_j(f_j-\lpr(f_j|B_j))\right]
(\omega) \geq 0,
\end{equation}
and let $\rdesirs$ be the natural extension of $\cup_{i\in
I}\widetilde\rdesirs|B_i$, where
\begin{align*}
 &\rdesirs_1|B_i\coloneqq\{B_if: \inf_{S(f)}
 f>0\}\cup\{0\}\\
 &\rdesirs_2|B_i\coloneqq\{B_i(f-\lpr_i(f|B_i)+\eps):
 f\in\gambles,\eps>0\}\\
 &\widetilde\rdesirs|B_i\coloneqq\{f_1+f_2: f_1\in\rdesirs_1|B_i,
 f_2\in\rdesirs_2|B_i\}.
\end{align*}
Then $\rdesirs$ is a fully strictly desirable set of gambles.
\end{theorem}
\noindent In other words, if a family of  coherent conditional lower
previsions satisfies the regularity condition~\eqref{eq:williams-apl}, then there is always a fully strictly desirable coherent set that is equally expressive to them, in the sense that each inference that we can do from the family of conditionals can be done from the set, and vice versa. As a consequence,
Theorem~\ref{pr:equal-expressivity} generalises
\cite[Theorem~15]{miranda2010c}.

Let us show next that for finite possibility spaces, Definition~\ref{def:cond-str-des} can be simplified:

\begin{proposition}\label{pr:CA-finite}
If $\pspace$ is finite, then the natural extension of a set of
gambles satisfying Eq.~\eqref{eq:cA} also satisfies~\eqref{eq:cA}.
As a consequence, a coherent set $\rdesirs$ is fully strictly desirable if and only if it
satisfies~\eqref{eq:cA}.
\end{proposition}

Let us show that this result does not hold in general on infinite
spaces:

\begin{example}
Let our possibility space be the set of natural numbers,
$B_n\coloneqq\{2n-1,2n\}$ for $n\geq 1$, and $\pr(\cdot | B_n)$ the
uniform probability distribution. Consider the set of gambles
$\rdesirs':=\cup_{n\in\nats}\rdesirs|B_n$, where $\rdesirs|B_n$ is defined by
Eq.~\eqref{eq:cond-from-lpr}. It is easy to prove that $\rdesirs'$
satisfies condition~\eqref{eq:cA}. Its natural extension is given by
\begin{equation*}
\rdesirs:=\{f\neq 0: (\exists J \subseteq \mathbb{N} \text{
finite})((\forall j\in J) (\lpr(f|B_j)>0)\text{ and }((\forall
n\notin J)(B_n f\geq 0)))\}.
\end{equation*}

To see that $\rdesirs$ does not satisfy Eq.~\eqref{eq:cA}, consider
the gamble $f\in\rdesirs$ given by $f(1)\coloneqq-1,f(2)\coloneqq2,
f(n)\coloneqq\frac{1}{n}$ for every $n\geq 3$. Then there is no $\epsilon>0$
such that $f-\epsilon$ belongs to $\rdesirs$, because for every
$\epsilon>0$ there is some natural number $n_\epsilon$ such that
$(f-\epsilon)(m)<0\ \forall m\geq n_\epsilon$.  $\lozenge$
\end{example}

To summarise the aim of this section, what we claim is that it could be worth considering full strict desirability as a replacement of the traditional Archimedean condition (\ref{A3} or \ref{wA3}), given that it fully characterises all the problems that can be expressed only through probabilities and utilities.

\section{Related work}\label{sec:relatedWs}

The link between desirability and preference has been surfacing in
the literature in a number of cases, but has apparently gone
unnoticed. That it has surfaced is not surprising, because the
theoretical study of preferences based on the mixture-independence
axiom results in, and is worked out using, cones; cones are also the
fundamental tool in desirability. That it has not been remarked and exploited explicitly, as
we do in this paper, is. Perhaps the most
evident case where the two theories have nearly touched each other
is in Galaabaatar and Karni's work \cite{galaabaatar2013}. In the next sections, we discuss this and two other main approaches in the literature that have dealt with the axiomatisation of incomplete preferences, and compare our approach with them.

\subsection{The work of Galaabaatar and Karni}\label{sec:galaabaatar}

One of the most influential works for this paper has been the one
carried out by Galaabaatar and Karni (GK) in \cite{galaabaatar2013}.

\subsubsection{In short}

They
provide an axiomatisation of incomplete preferences for finite spaces
of possibilities and outcomes. They consider a strict preference ordering on acts and use axioms that are essentially ours as well, with the difference that they impose Archimedeanity all the way through, so as to obtain representation results always based on probabilities and utilities. They consider the special cases of partial and complete separation of probabilities and utilities and also provide an axiomatisation for two particular cases of interest: those of precise utilities with imprecise probabilities (that they call \emph{Knightian uncertainty}) and of precise probabilities with imprecise utilities.

Remarkably, for their results they exploit a cone that is very similar to our set of desirable gambles $\rdesirs$  obtained out of $\cone$ by dropping an outcome $z\in\valuesy$ from consideration. The difference is that they do not assume that there is the worst outcome (so $z$ denotes just an element of $\valuesy$ in this section) but rather they assume that there is a worst act (besides a best act); and probably for this reason they just remove an arbitrary element of $\valuesy$ to obtain the derived cone. This seems to be the key passage that prevents them from identifying the worst outcome with the set of positive gambles and hence to obtain an actual set of desirable gambles. Yet, the similarity of their approach to ours has allowed us to take advantage of some of their results, in particular in some key passages of Theorem~\ref{prop:rdesirs-archimedean} and Proposition~\ref{pr:ind-selection}.

The main differences with our work is that we have taken full advantage of desirability to work out the extension to infinite (not
necessarily countable) $\pspace$ and to model non-Archimedean problems. Working directly with desirability has also allowed us to give a more primitive axiomatisation and to formulate some results in a way that is arguably more intuitive; for instance, in our approach what they call state independence is equivalently reformulated explicitly as the factorisation of joint models into marginal probability and utility models.

\subsubsection{A deeper view}
With respect to the basic axioms, what we have called a coherent
preference relation corresponds to \cite[Axioms~A.1 and~A.3]{galaabaatar2013}, and the Archimedean condition exactly corresponds to the one they
consider in their paper (that is, \cite[Axioms~A.2]{galaabaatar2013}).

The first result they provide in \cite[Lemma~1]{galaabaatar2013} is
that if a preference relation is bounded, coherent and Archimedean,
then it can be represented by means of a family $\Phi$ of
real-valued functions on $\pspace\times\valuesy$, so that for all
$p,q\in\actsy$:
\begin{equation*}
 p\succ q \iff (\forall \phi \in \Phi)\sum_{(\omega,x)\in\pspace\times\valuesy} p(\omega,x) \phi(\omega,x) > \sum_{(\omega,x)\in\pspace\times\valuesy} q(\omega,x) \phi(\omega,x).
\end{equation*}
This has been somewhat improved in our
Theorem~\ref{prop:rdesirs-archimedean}: we show that a weaker
version of the Archimedean condition suffices to establish the
correspondence with a coherent set of strictly desirable gambles
$\rdesirs\subseteq\gambles(\pspace\times\valuesn)$. This set is in turn equivalent
to a family $\solp$ of finitely additive probabilities on
$\pspace\times\valuesn$, and as a consequence we obtain that for all $p,q\in\actsy$, denoting $f\coloneqq\pi(p),g\coloneqq\pi(q)$:
\begin{equation*}
 p\succ q, p\ndom q \iff (\forall \pr \in \solp)\pr(f)> \pr(g),f-g \in \rdesirs\setminus\gambles^+\iff f\succ g,f\not\gneq g.
\end{equation*}
Note on the other hand that $p\dom q$ implies that
$f-g\in\gambles^+$ and that $f\succ g$. However, it may be that in
that case the lower prevision associated with $\rdesirs$ satisfies
$\lpr(f-g)=0$, meaning that $\pr(f)=\pr(g)$ for some
$\pr\in\solp$. This is an example of non-Archimedeanity that cannot be represented using closed convex sets of probabilities alone (i.e., lower previsions) and which shows once again the additional expressiveness of desirable gambles; this was mentioned already in Section~\ref{sec:regular}.

Another important difference between our work and that of GK is that
they assume that a preference relation is bounded by a worst
\emph{and} a best act. We have shown in Section~\ref{sec:z} that the
assumption of a worst act is not without consequences: although if
we start with a coherent preference relation in a space without a
worst act we can always extend it to a bigger space with a worst
outcome, the minimal extension will \emph{never}\footnote{Except in
the trivial case where we start with a vacuous preference relation;
see Theorem~\ref{thm:ext}.} satisfy the Archimedean
property~\ref{A3}, irrespectively of the notion of Archimedeanity
one starts from. Moreover, in order for an Archimedean extension to
exist, under the same conditions of GK, we must assume that the
original relation satisfies a stronger version of~\ref{A3}, as
Theorem~\ref{th:existence-open-superset} shows; in addition we have
shown that there is not such a thing like the minimal (weak)
Archimedean extension in such a situation. Note also that our
axiomatisation, perhaps surprisingly, only needs the existence of
the worst act. Also, the related notion of objective preference in
Definition~\ref{def:dom} allows us to represent what seems to be the
appropriate definition of a vacuous preference relation (see
Proposition~\ref{prop:vacuous-Arch} and the discussion that follows
it).

In \cite[Theorem.~1]{galaabaatar2013}, they show that a coherent and
Archimedean preference relation that is moreover bounded and
satisfies their so-called \emph{dominance} axiom is in
correspondence with a family $\Phi$ of probability-utility models.
They represent these as $\{(\Pi^U, U): U \in {\mathcal U}\}$, where
$\Pi^U$ denotes the family of probability measures on $\pspace$ that
is combined with a particular utility function $U$ on $\values$.
Note that in a completely similar manner we could represent $\Phi$
as $\{(P, {\mathcal U}^{P}): P\in {\mathcal P}\}$, where ${\mathcal
U}^P$ denotes the set of utility functions that is combined with a
particular probability measure on $\pspace$. This second representation is closer to the one we have provided in this paper in the case of state independence: it stresses the fact that our (possibly imprecise) utility function may rather depend on the belief model we have on the states of nature.

What the dominance axiom mentioned above requires is that, for every
pair of horse lotteries $p,q$,
\begin{equation*}
(\forall \omega\in\pspace) q\succ p^{\omega} \Rightarrow q\succ p,
\end{equation*}
where $p^\omega$ is the von Neumann-Morgenstern horse lottery defined by $p^\omega(\omega',\cdot)\coloneqq p(\omega,\cdot)$ for all $\omega'\in\pspace$. This axiom is analogous to our condition~\eqref{eq:A4}, which we have used to establish a correspondence with lower envelopes of linear products in Proposition~\ref{pr:ind-selection}.

On the other hand, in \cite[Theorem~2]{galaabaatar2013} they provide
conditions for the complete separation of probabilities and
utilities, so that $\Pi^U$ does not depend on $U$ in the above
representation. This combination of each probability measure on the
set that models your beliefs on $\pspace$ with every utility
function on the set that models your values of $\values$ is
implicitly made by the strong product in GK. We make this explicit
by Theorem~\ref{pr:domin-by-str-prod}, which provides a sufficient
condition for the dominance by the strong product, by means of a
property analogous to \cite[Axiom~A.5]{galaabaatar2013}; by putting
it together with condition~\eqref{eq:A4}, we can eventually obtain a
characterisation of the strong product along the lines of GK. Note,
in passing, that this implies that GK's representation of
state-independent preferences is dynamically consistent provided
that we opt for a sensitivity-analysis interpretation of sets of
probabilities and utilities (see Remark~\ref{rem:si-indep}).

Yet, the behavioural interpretation can often be more natural in
practice. In that case, we can use weaker and more realistic ways to
obtain state-independent models that are dynamically consistent. For
example, we can model state independence similarly to GK but using
independent products, such as the independent natural extension, as
discussed in Section~\ref{sec:lprevs}. Or we can drop the assumption
that irrelevance should be symmetric between beliefs and values and
move on to the weaker, and arguably more reasonable, marginal
extension models of Propositions~\ref{prop:weakest-IP-gmb}
and~\ref{sec:ip-lower}.

Finally, GK also discuss the cases of precise
probabilities or utilities. Their work has actually been connected to other characterisations of linear previsions in Proposition~\ref{pr:charac-linear}.

First, in \cite[Theorem~3]{galaabaatar2013} they characterise
Knightian uncertainty, which corresponds to the case of imprecise
probabilities and precise utilities. They show that for this it is
necessary to add an axiom of \emph{negative transitivity} on
constant acts on $\pspace$. We have established an analogous result
in our context in Proposition~\ref{pr:precise-utilities}(ii) using
$\values$-measurable gambles and the negative additivity of a
coherent set of desirable gambles. Note that in our result we
are allowing for an infinite $\pspace$ and we are not requiring the existence of a best act. 

Concerning precise probabilities and imprecise utilities, the
characterisation in \cite[Theorem~4]{galaabaatar2013} requires their Axiom~A.7
that is analogous to condition~(iii) of Proposition~\ref{pr:precise-utilities}, and again our result is established also for an infinite $\pspace$. 

\subsection{The work by Nau}

Another important work in the axiomatisation of the expected utility
model with imprecise probabilities and utilities was carried out by
Nau in \cite{nau2006}. Like GK, he considers the case of finite
spaces of possibilities and outcomes; unlike them, he makes an
axiomatisation of weak preferences, that is, he uses a weak relation
$\succeq$, instead of the strict preferences (relation $\succ$) we
have considered in this paper. He discusses this point in some
detail in \cite[Section~4]{nau2006}.

From the point of view of this paper, the distinction between weak and strict preferences is not too important: the reason is that there is an alternative formulation of desirability that assumes the zero gamble to be desirable, and hence implicitly defines a weak preference order (see \cite[Appendix~F]{walley1991}). In such a formulation, the convexity conditions~\ref{D3},~\ref{D4} are the same as before, so that we still deal with convex cones; the remaining two conditions \ref{D1},~\ref{D2} are instead replaced by the following two:
\begin{enumerate}[label=\upshape D\arabic*$'$.,ref=\upshape D\arabic*$'$]
\item\label{D1'} $f\geq0\Rightarrow f\in\rdesirs$;
\item\label{D2'} $f\lneq0\Rightarrow f\notin\rdesirs$.
\end{enumerate}
The first makes clear that the zero gamble is desirable now, besides the positive ones. The second explicitly states that the negative gambles are not desirable; this was a consequence of~\ref{D1},~\ref{D2} and~\ref{D4} in the original formulation.

Now, consider a set of gambles $\rdesirs$ that satisfies~\ref{D1'},~\ref{D2'} and~\ref{D3},~\ref{D4}. If we  define preference as in Definition~\ref{def:pref-gmb}, by saying that $f$ is preferred to $g$ whenever $f-g\in\rdesirs$, then what we get is a weak preference order, given that reflexivity is allowed by $f-f=0\in\rdesirs$.

In other words, the original formulation lets us model strict preference but not indifference; the new one models weak, but not strict, preference.\footnote{In order to represent both indifference and strict preference, one needs a richer language than desirability, such as the one proposed by Quaeghebeur et al. in \cite{quaeghebeur2015}.} Other than this, the two formulations of desirability are essentially equivalent; this means that we could rewrite this paper with minor changes by using the alternative formulation together with weak preference. In particular, one of the minor changes would be related to Archimedeanity. Recall that in this paper Archimedeanity corresponds to sets of strictly desirable gambles, which are open convex cones, in the sense that they exclude the border (except for the part of the cone made up by the positive gambles). In the new formulation, Archimedeanity would correspond to closed convex cones, that is, cones that include the border.

Going back to Nau's work, he imposes the existence of a best and
a worst outcome (similarly to GK, with the difference that they impose best and worst acts); we refer to
the previous section for a discussion of this point. In addition, he
also requires that relation $\succeq$ satisfies a continuity
axiom. This axiom appears in fact to be the  analog of focusing on sets of desirable gambles that include the border, that is, on the Archimedean condition for desirability in case of weak preference. He proves in \cite[Theorem~2]{nau2006} that under these
conditions, and together with the analogue of
axioms~\ref{A1},~\ref{A2} in case of weak preference, relation
$\succeq$ can be represented by means of a family of state-dependent
expected utility functions. The continuity axiom enables him in
addition to obtain in some cases a representation in terms of a
basis.

This result is similar to our basic representation theorem in terms
of a coherent set of desirable gambles, or its associated coherent
lower prevision $\lpr$: in the finite case any linear prevision
$\pr\in\solp(\lpr)$ will be equal to $\pr(\pr(\cdot|\pspace))$, by
total expectation (i.e., marginal extension), so that we can regard
it as a state-dependent representation.

With respect to state independence, he studies under which
conditions relation $\succeq$ is characterised by means of a family of state-independent models, which in the language of this paper correspond to factorising linear previsions. He first obtains a
representation of this type for constant von Neumann-Morgenstern lotteries only
\cite[Theorem~3]{nau2006}, and then another representation for
arbitrary horse lotteries in \cite[Theorem~4]{nau2006}, by means of
the following axiom:
\begin{align*}
 & \text{If } p\succ q,\ A \succeq \alpha b+ (1-\alpha) z,\text{ and } \beta b + (1-\beta) z \succeq B, \text{ then:} \\
 &\left[\alpha' A p+ (1-\alpha') p' \succeq \alpha' A q+ (1-\alpha') q' \Rightarrow
 \beta' B p+ (1-\beta') p' \succeq \beta' B q+ (1-\beta') q'\right],
\end{align*}
where $b$ denotes the best act, $A,B$ are subsets of $\pspace$,
$\alpha>0$, $f',g'$ are constant acts and $\alpha',\beta'$ are
related by the formula
\begin{equation*}
 \alpha'=1 \Rightarrow \beta'=1 \text{ and } \alpha'< 1 \Rightarrow \frac{\beta'}{1-\beta'}\leq \frac{\alpha \alpha'}{(1-\alpha)(1-\alpha')}.
\end{equation*}
He calls this axiom \emph{strong state-independence}. In our
language, the representation he obtains means that the corresponding
lower prevision is a lower envelope of factorising linear previsions. It has been
characterised in Proposition~\ref{pr:ind-selection} by means of the
somewhat simpler axiom~\eqref{eq:A4}.

\subsection{The work by Seidenfeld, Schervish and Kadane}

Seidenfeld, Schervish and Kadane (SSK) are the authors of one of the
first axiomatisations of incomplete preferences based on a
multiprior expected multiutility representation
\cite{seidenfeld1995}.

\subsubsection{Overview}
They study the problem for strict preference orders and impose on them Axioms~\ref{A1} and~\ref{A2}, such as GK and we do. However, their work presents some significant departure from ours and the others we have been discussing, mainly for the following reasons:
\begin{itemize}
\item They allow for a countable space of rewards $\values$ and for simple lotteries defined as $\sigma$-additive probabilities on such a countable support. (The space of possibilities is assumed to be infinite but only finite partitions of $\pspace$ are eventually used, so that we can actually regard $\pspace$ as finite in their approach.)

\item They allow for sets of expected utilities that are not necessarily closed; therefore their representation is richer than others thanks to the additional expressiveness due to the accurate treatment of the border of such sets. These sets are convex, on the other hand.
\end{itemize}

\noindent These two features make SSK's treatment substantially more complex than those of Nau or GK.

\subsubsection{At a deeper level}

For a start, they define a special type of Archimedean axiom based on a topological structure on preferences:
\begin{enumerate}[label=\upshape tA\arabic*.,ref=\upshape tA\arabic*]
\setcounter{enumi}{2}
\item\label{TA3} $(\forall n\in\nats)(p_n\succ q_n)((p_n)_{n\in\nats} \rightarrow p)((q_n)_{n\in\nats} \rightarrow q) \Rightarrow(\forall r\in\acts)(r \succ p \Rightarrow r
\succ q) (q \succ r \Rightarrow p \succ
r)$,
\end{enumerate}
where the convergence of the horse lotteries is understood as point-wise. Their motivation to do so is given through the following:
\begin{example}[Example~2.1 in \cite{seidenfeld1995}]\label{ex:SSK} Consider a trivial space of possibilities $\pspace$ with only one element and a set of three rewards: $\values\coloneqq\{z,a,b\}$. Here $z$ is the worst outcome and $b$ the best one. Acts can then be represented as simple von Neumann-Morgenstern lotteries, which depend on outcomes alone. Preferences are determined by the utility functions $\{u_\alpha:\values\rightarrow[0,1]|u_\alpha(z)\coloneqq0,u_\alpha(b)\coloneqq1,u_\alpha(a)\coloneqq\alpha,\alpha\in(0,1)\}$ by means of:
\begin{equation*}
(\forall \alpha\in(0,1)) \sum_{x\in\values}p(x)u_\alpha(x)>\sum_{x\in\values}q(x)u_\alpha(x)\Rightarrow p\succ q,
\end{equation*}
for all $p,q\in\acts$. SSK show that relation $\succ$ satisfies Axioms~\ref{A1} and~\ref{A2} but it fails the Archimedean axiom~\ref{A3}. $\lozenge$
\end{example}
\noindent The point here is that the common Archimedean condition~\ref{A3} is
not suited to deal with sets of expected utilities that are informative on the border, as in the example. And since SSK want to have a representation in terms of expected utilities, then they drop~\ref{A3} as inadequate to capture the situation and move on to~\ref{TA3}.

However, there is a another possible way out: it consists in dropping the idea that we should have an expected utility representation in the case of the example. In fact, since relation $\succ$ of the example satisfies~\ref{A1} and~\ref{A2}, we can use Proposition~\ref{prop:rdesirs} to equivalently represent it by a coherent set of desirable gambles $\rdesirs$, which is in particular given by
\begin{equation*}
\rdesirs\coloneqq\{f\in\gambles(\{a,b\}):(\forall\alpha\in(0,1))\alpha f(a)+f(b)>0\}.
\end{equation*}
It is not difficult to check that $\rdesirs$ is not a strictly desirable set; this is a direct consequence of relation $\succ$ failing~\ref{A3}, and more precisely~\ref{A3p}, which is equivalent to~\ref{wA3} in the considered problem (see Proposition~\ref{prop:A3'-wA3}). This means that in our formalisation there is not the possibility to have a representation of the problem by a closed and convex set of linear previsions; yet, we can deal with it in an unproblematic way via desirability.

We can look at this question also in more general terms as follows.
In the Conclusions of their paper, SSK mention that the border of
the set of expected utilities is not something that should be
sidestepped, as by taking it into proper account, one enhances the
expressiveness of the models used. We are in fact very close in
spirit, as well as very sympathetic, towards the underlying idea;
only, we are particularly concerned with the border of the sets of
desirable gambles rather than that of the sets of expected
utilities. One reason for this standpoint is that the related theory
looks simpler in our view; we can for instance avoid introducing the
transfinite induction used by SSK. Moreover, we can formally
identify our axiomatisation based on desirability with a logic, as
it follows from the discussion in Section~\ref{sec:foundations}.
That this is possible for SSK's theory is unclear to us. Also, we
need not be concerned with so-called null events, which are events
related to zero probabilities. They are usually neglected in
traditional theories and at the same time they complicate the formal
development, whereas we can take them into account while treating
them uniformly with others, in an implicit way. More generally
speaking, our standpoint is that desirability greatly empowers a
theory of incomplete preferences to gain expressiveness while
keeping the overall approach and the mathematics relatively simple
and straightforward.

This is not to make any definite claim about the relative expressiveness of a desirability-based axiomatisation and SSK's theory, because this would need a separate investigation. We believe, for instance, that the information on the border of the sets of expected utilities, which SSK can capture, can as well be captured by using desirable gambles. But turning this into a formal claim needs some careful work. On the other hand, SSK can deal with countable spaces of outcomes. This seems to be somewhat limited when they extend a partial preference relation $\succ$ to a total one by transfinite induction, as they can get a weak preference relation $\succeq$ only. Related to this, they eventually obtain a notion of \emph{almost} state independence through a limit argument, as actual state independence seems prevented from being obtained in their setting. On our side, we conjecture that it is possible to extend the work in this paper without major problems to a countable space of outcomes while preserving the main structure and results of the paper, but, also in this case, this has to be verified. Finally, the use of $\sigma$-additive probabilities would make it difficult to extend SSK's approach fully to an uncountable set $\pspace$, which is something we instead can do quite simply by desirability. At the same time, it is clear that SSK's idea of taking advantage of the border of the set of expected utilities has been a bright idea as it considerably expands the power of a theory of preferences compared to more traditional approaches.

\subsubsection{Extension to the worst act}

Another question of particular relevance for this paper is that SSK do not assume that best and worst act exist: they give a theorem (\cite[Theorem~2]{seidenfeld1995}) to prove that any relation can be extended to best and worst acts, while preserving in particular their Archimedean condition~\ref{TA3}. This somewhat clashes with our results in Section~\ref{sec:z}, so it is useful to take some time to deepen this question. We focus on the worst act, which has been the subject of our discussion in Section~\ref{sec:z}.

A minor consideration is that there seems to be some trouble with SSK's Theorem~2. In particular the theorem appears to fail in case the original relation $\succ$ is empty. In this case all acts are incomparable to each other and, unfortunately, SSK's definition of indifference (Definition~8) coincides with incomparability for an empty relation. This eventually leads SSK's Definition~16 to yield the non-rational extension in which every act is preferred to every other. Yet, this problem should not be too difficult to fix; recall that Theorem~\ref{thm:ext} (point~3) shows that the empty relation is a well-solved case under~\ref{wA3}. On the other hand, we cannot use this result directly, given that it exploits an Archimedean condition different from~\ref{TA3}. It is useful to contrast the two possible solutions with the help of a simple example.

\begin{example}
Consider a problem with a finite possibility space $\pspace$ and a
trivial set $\values$ made by only one element (this means that we
can neglect $\values$ and work only with $\pspace$). Assume that we
start with an empty relation $\succ$. Then Theorem~\ref{thm:ext}
shows that the minimal Archimedean extension under~\ref{wA3}
corresponds to the vacuous set of desirable gambles
$\rdesirs=\gambles^+$. In turn, this is in a one-to-one relation,
through~\eqref{eq:lprR}, with the vacuous credal set: that is, the
set of all the mass functions on $\pspace$. This is the traditional
model for the complete lack of information about $\pspace$. The
vacuous credal set is not a possible solution for SSK's extension of
the empty relation, though. The problem is originated by zero
probabilities: assume that a mass functions assigns zero probability
to some $\omega\in\pspace$; then gamble $f$, equal to $1$ in
$\omega$ and $0$ everywhere else, has lower prevision equal to zero.
Since $f$ can be written as $f=\pi(p-z)$ for some $p\in\actsy$, then
we should deduce that there is an act that is not preferred to $z$.
We can correct for this problem by considering the open set of all
the mass functions that assign positive probability to each element
of $\pspace$, which we argue should be SSK's solution for the
present case. Notice however that, strictly---or
logically---speaking, this cannot be truly regarded as a vacuous
credal set given that some mass functions are missing from it.
$\lozenge$
\end{example}

Most importantly, we can try to map the Archimedean condition~\ref{TA3} into one for sets of desirable gambles. In fact, SSK's relation $\succ$ satisfies axioms~\ref{A1},~\ref{A2}, whence (provided $\values$ is finite) it is in a one-to-one correspondence with a
coherent set $\rdesirs$ of desirable gambles, as it follows from Theorem~\ref{thm:equiv}. Moreover, from the point-wise convergences $(p_n)_{n\in\nats} \rightarrow p$ and $(q_n)_{n\in\nats} \rightarrow q$, in~\ref{TA3}, we deduce that the sequence $(p_n-q_n)_{n\in\nats}$ converges point-wise to $p-q$. Then, if $p_n \succ q_n$ for every
$n\in\nats$, it follows that $\pi(p-q)$ is in the closure of a sequence of
gambles in $\rdesirs$, or, in other words, that $\pi(p-q)$ is a so-called
\emph{almost-desirable gamble} \cite[Section~3.7]{walley1991}.\footnote{A subtlety is that almost-desirable gambles are \emph{uniform} limits of gambles in $\rdesirs$, while~\ref{TA3} is based on point-wise convergence. This difference is immaterial in the present context as the two convergences coincide in a countable space $\pspace\times\values$ for gambles obtained through differences of acts.} Conversely, it is well known that any almost-desirable
gamble is the limit of a sequence of desirable gambles, and therefore the almost-desirable gambles can be identified with the topological closure of set $\rdesirs$ (see~\cite[Proposition~4]{miranda2010c}), which we denote by $\cluc{\rdesirs}$. Overall, this means
that Axiom~\ref{TA3}
can be reformulated as follows in the context of desirability:
\begin{enumerate}[label=\upshape tA\arabic*$'$.,ref=\upshape tA\arabic*$'$]
\setcounter{enumi}{2}
\item\label{TA3-equiv} $ f \in \cluc{\rdesirs}, g \in \rdesirs \Rightarrow f+g \in \rdesirs
$.
\end{enumerate}
In the language of \cite{cooman2010}, condition~\ref{TA3-equiv} requires that any almost-desirable gamble should be \emph{weakly desirable}. We have
shown in \cite[Example~11]{miranda2010c} that this is not always
the case, which is another way of showing that axiom~\ref{TA3} is
not trivial.

Interestingly, we can use the above formulation to
easily show that there are cases where there is indeed a minimal
extension of a relation satisfying
axioms~\ref{A1},~\ref{A2},~\ref{TA3} to include a worst outcome:
\begin{example} Let $\lpr$ be a coherent lower prevision on
$\gambles(\pspace\times\valuesn)$ such that $\lpr(\{(\omega,x)\})>0$
for every $(\omega,x)\in\pspace\times\values$, and consider the
relation $\succ'$ on $\actsn\times\actsn$ given by $$r\succ' s \iff
\lpr(r-s)>0.$$ This relation satisfies axioms~\ref{A1},~\ref{A2}
and~\ref{TA3}.  Let relation $\succ$ on $\actsy\times\actsy$, given
by Eq.~\eqref{eq:ext2} in the Appendix, be the minimal extension of
relation $\succ'$ to the worst outcome that satisfies
axioms~\ref{A1},~\ref{A2}.

Then it holds that $p\succ q$ implies $\lpr(\pi(p-q))>0$. To prove this, note that if $p\dom q$, then
$\pi(p-q)\in\gambles^+$, and then given
$(\omega,x)\in\pspace\times\values$ such that
$\pi(p-q)(\omega,x)>0$, it holds that
\begin{equation*}
 \lpr(\pi(p-q))\geq\lpr(I_{\{(\omega,x)\}} \pi(p-q)(\omega,x))>0.
\end{equation*}
On the other hand, if $\pi(p-q)=r-s$ for some $r,s\in\actsn,r\succ'
s$, it follows that $\lpr(\pi(p-q))=\lpr(r-s)>0$.

From this it follows that if we consider a gamble $f$ which is the
limit of a sequence of gambles of the type $p_n-q_n$, $n\in\nats$, with $p_n\succ
q_n$, it holds that $\lpr(\pi(f))\geq 0$. Using this, together with
the super-additivity of $\lpr$, we deduce that~\ref{TA3-equiv} holds and as a consequence relation $\succ$ satisfies~\ref{TA3}. $\lozenge$
\end{example}

One may wonder at this point what is the underlying reason that makes the minimal extension in the previous example Archimedean according to~\ref{TA3} and non-Archimedean in the sense of~\ref{wA3} (recall that the latter means that there is no open coherent set of desirable gambles that can represent it, or, equivalently, that there is no closed set of linear previsions that can). The reason appears to be just a question of border: in fact, the example shows that it is possible to characterise the extended relation by an open set of linear previsions, all of which assign positive prevision to each gamble $\pi(p-q)$ such that $p\succ q$. On the other hand, the example shows that for a limit gamble like $f$ it only holds that $\lpr(\pi(f))\geq 0$. This is easy to see in a specific case, if we take the limit gamble $0=\lim_{n\rightarrow\infty}\pi((p-z)/n)$, whose lower prevision is clearly zero whereas each gamble inside the limit has positive lower prevision (provided that $p\neq z$). This means that relation $\succ$ cannot be represented by a closed set of linear previsions. SSK bypass this problem by working with sets of linear previsions that need not be closed; we deal with the same problem by using coherent sets of desirable gambles that need not be strictly desirable, that is, open.

\section{Conclusions}
Traditional decision-theoretic axiomatisations, \`a la Anscombe-Aumann, typically turn a problem of incomplete preferences into an equivalent convex cone $\cone$ made of scaled differences of horse lotteries (on the possibility space $\pspace$ and set of outcomes $\values$). This cone is used as a tool to derive a representation of preferences in terms of a set of expected utilities through the Archimedean axiom.

In this paper we have argued that convex cones should be fully given the primary role in the axiomatisation of incomplete preferences. We have shown as the objects making up cone $\cone$ can, with a simple transformation, be directly regarded as gambles (i.e., bounded random variables $f:\pspace\times\values\rightarrow\reals$) a subject finds acceptable, or desires. This formally turns cone $\cone$ into an equivalent convex cone $\rdesirs$ of desirable gambles, as defined in the literature of imprecise probability. 

The implications on the decision-theoretic side appear far-reaching:
we have discussed how a decision-theoretic problem can be formulated
exclusively in terms of desirable gambles, preference being a
derived notion; and that we can deal with non-Archimedean problems
by directly working on cone $\rdesirs$ using the operations it comes
equipped with. Using these operations, we have given new, simple
formulations of state independence and completeness directly at the
level of the cone, which then hold also for non-Archimedean
problems. In the Archimedean case, these formulations are naturally
specialised to the set of expected utilities induced by the cone. In
this case, we have shown in particular that the traditional notion
of state independence is eventually based on the idea of stochastic
independence; and that we can actually take advantage of this
insight to propose much a weaker notion of state independence
exploiting the literature of imprecise probability. We have also
taken a close look at the Archimedean axiom itself using
desirability: we have shown that it can profitably be weakened,
which has allowed us to work with uncountable possibility spaces; on
the other hand, we have given a new condition, full Archimedeanity,
that is needed to truly capture all the problems that can be
represented through sets of expected utilities. In doing so, we have
uncovered the importance of the worst act in a preference relation,
in order to make the bridge to desirability; and shown that the best
act is actually unnecessary for the theoretical development. We have
also shown that it is always possible to extend a preference
relation in a minimal way to one with a worst act, but not in an
Archimedean way.

We can read the results also in the other way around, so as to find
out the implications of this work for imprecise probability. We
mention just two of them. Most importantly, the extension of
desirability we have proposed here to vector-valued gambles (that is,
gambles $f:\pspace\times\values\rightarrow\reals$ rather than
$f:\pspace\rightarrow\reals$) can be regarded as the foundation of
the theory of desirability generalised to deal with imprecise and
non-linear utility, which is done here for the first time. Another
new outcome is the characterisation of the strong product as the
weakest model that satisfies the sure thing principle. The strong product
defines a notion of independence obtained through a set of
stochastically independent models. It is a well-known and much used
notion that is however widely regarded, and to some extent
criticised, as lacking a behavioural interpretation (in the sense
that it can only be given a sensitivity analysis interpretation).
Our characterisation shows that this is actually not the case as it
gives the strong product, for the first time, a clear behavioural
interpretation.

There are a few main opportunities for research that the present work opens up. One is determined by the current limitation of the set of outcomes $\values$ to be finite. Our viewpoint is that the extension to a countable set $\values$ should be relatively easy to attain by defining a horse lottery $p$ so as that $p(\omega,\cdot)$ is a countably additive probability for each $\omega\in\pspace$; the main reasoning line in the paper could probably be preserved by this choice, but it might be the case that some of the proofs have to be also quite substantially reworked. The passage to an uncountable set $\values$ seems to involve a different degree of complexity. The main reason is that we cannot represent the simple lottery $p(\omega,\cdot)$ through a countably additive probability when $\values$ is uncountable; and, on the other hand, using a finitely additive probability seems to be impractical. Therefore to properly deal with this case, one seems to have to develop ideas that may be also fundamentally different from those in the present paper.

Another challenge is related to the question of conglomerability. In this paper we have tried to minimise the involvement of conglomerability in the development, as it is already quite a controversial topic per se. At the moment, conglomerabiliy is used only in our definitions of state independence based on irrelevant products---removing conglomerability in this case could give rise to definitions of state independence with questionable properties. This nevertheless, the issue of conglomerability relates to many parts of the paper and there should be, in the future, work dedicated to clarify its role in this context. For instance, in the Archimedean case, our formulation leads to modelling beliefs by a set of finitely additive probabilities on $\pspace$. In case conglomerability is considered as an essential property of a probabilistic model, then a conglomerability axiom should be added to the desirability formulation of decision-theoretic problems, so as to make sure that only sets of probabilities with the property of being conglomerable are obtained.

It could be also useful to take advantage of the definition of natural extension for a set of gambles so as to give a weaker axiomatisation of incomplete preferences. The idea is that a subject would be allowed to express desirability statements only on a subset of all the possible gambles (e.g., only the gambles corresponding to von Neumann-Morgenstern lotteries); requiring that these statements avoid partial loss would be a rationality condition necessary and sufficient to extend them into a coherent set of gambles through natural extension, that is, to a coherent preference relation. In this way, we would obtain a more general and flexible assessment procedure through the use of a weaker rationality condition than coherence.

Finally, it would be useful to evaluate the expressive power of the theory based on cones that we propose here relative to the theory by Seidenfeld et al. \cite{seidenfeld1995}. The interest lies in particular in the different ways the two theories cope with non-Archimedean problems: through the border of the set of expected utilities in the case of Seidenfeld et al., and through the border of the cone of gambles in ours.

\section*{Acknowledgements} We acknowledge the financial support by
project TIN2014-59543-P.

\appendix

\section{Proofs}\label{sec:proofs}

In this appendix we collect the proofs of all the results in the paper as well as a few results that are more technical and are needed for some proofs.

\begin{proof}[{\bf Proof of Lemma~\ref{lem:w-degen}}]
Suppose ex-absurdo that there is $\bar\omega\in\pspace$ such that for no $x\in\values$, $w(\bar\omega,x)=1$. Then there must be $x_1,x_2\in\values$ such that $w(\bar\omega,x_1),w(\omega,x_2)\in(0,1)$. Let $\eps\coloneqq\frac{1}{2}\min\{w(\bar\omega,x_1),w(\bar\omega,x_2),1-w(\bar\omega,x_1),1-w(\bar\omega,x_2)\}$. Define acts $p,q$ so that they coincide with $w$ for all $\omega\in\pspace, \omega\neq\bar\omega$, while in $\bar\omega$ they are defined by
\begin{align*}
 p(\bar\omega,x)\coloneqq\begin{cases}
  w(\bar\omega,x)+\eps \quad &\text{ if } x=x_1, \\
  w(\bar\omega,x)-\eps \quad &\text{ if } x=x_2, \\
  w(\bar\omega,x) \quad &\text{ otherwise,}
  \end{cases}\qquad
  q(\bar\omega,x)\coloneqq\begin{cases}
  w(\bar\omega,x)-\eps \quad &\text{ if } x=x_1, \\
  w(\bar\omega,x)+\eps \quad &\text{ if } x=x_2, \\
  w(\bar\omega,x) \quad &\text{ otherwise.}
  \end{cases}
\end{align*}
Since by assumption it must hold that $p\succ w$ and $q\succ w$, we
deduce, thanks also to the convexity of cone $\cone$, that
$0=(p-w)+(q-w)\in\cone$. This is a contradiction with the
irreflexivity of the preference relation, through
Proposition~\ref{prop:charC}.
\end{proof}

\begin{proof}[{\bf Proof of Proposition~\ref{prop:rdesirs}}] The proof of the second point is trivial, so we concentrate on showing that $\rdesirs$ satisfies~\ref{D1}--\ref{D4}:
\begin{itemize}

\item[\ref{D1}.]
Consider $f\in\gambles^+(\pspace\times\valuesn)$ and let $\lambda\coloneqq\sup_{\omega\in\pspace}\sum_{x\in\valuesn}f(\omega,x)$. Note that $\lambda>0$ otherwise $f$ would not be positive and moreover that $\lambda<+\infty$ because $f$ is bounded. Then let $p\in\actsy$ be defined by
\begin{equation*}
 p(\omega,x)\coloneqq\begin{cases}
  \frac{f(\omega,x)}{\lambda} \quad &\text{ if } x\neq z, \\
  1-\sum_{x\neq z}\frac{f(\omega,x)}{\lambda} \quad &\text{ otherwise}.
  \end{cases}
\end{equation*}
It follows that $\pi(\lambda(p-z))=f$ and since $p\neq z$ because
$f$ is positive, then $p\succ z$ and hence $f\in\rdesirs$.

\item[\ref{D2}.] $\cone$ does not contain the origin by Proposition~\ref{prop:C-convex-cone} and hence  $0\notin\rdesirs$.

\item[\ref{D3}.] This follows trivially from the convexity of $\cone$, again by Proposition~\ref{prop:C-convex-cone}.

\item[\ref{D4}.] This follows trivially from the convexity of $\cone$ taking into consideration that $\pi$ is a linear
functional. $\qedhere$
\end{itemize}
\end{proof}

\begin{proof}[{\bf Proof of Proposition~\ref{prop:domsucc}}]
$p\dom q\Rightarrow \pi(p-q)\in\gambles^+(\pspace\times\valuesn)$.
By Proposition~\ref{prop:rdesirs}, there are $p',q'\in\actsy$ and
$\lambda>0$ such that $p'\succ q'$ and $\lambda\pi(p'-q')=\pi(p-q)$.
Applying $\pi_2^{-1}$ to both sides of the equation, we get that
\begin{equation*}
 \lambda (p'-q')=p-q,
\end{equation*}
whence $\frac{\lambda}{\lambda+1}(p'-q')=\frac{1}{\lambda+1}(p-q)$.
Applying now Proposition~\ref{prop:nau}, we conclude that $p\succ
q$.
\end{proof}

\begin{proof}[{\bf Proof of Proposition~\ref{prop:vacuous-Arch}}]
We begin with the direct implication. Assume first of all that
$|\valuesy|\geq3$. Let $u$ be the \emph{uniform} act defined by
$u(\omega,x)\coloneqq1/|\valuesy|$ for all $\omega\in\pspace$. Then
fix any $x'\neq z$ and $\eps>0$ small enough so as to define act $p$
given, for all $\omega\in\pspace$, by
\begin{equation*}
p(\omega,x)\coloneqq
\begin{cases}
u(\omega,x)+\eps\quad&\text{ if }x=x',\\
u(\omega,x)-\eps\quad&\text{ if }x=z,\\
u(\omega,x)\quad&\text{ otherwise.}
\end{cases}
\end{equation*}
It follows that $p\succ u$ since $p\dom u$. Given that also $u\succ
z$, we can apply the Archimedean axiom~\ref{A3} to obtain that there
is a $\beta\in(0,1)$ such that:
\begin{equation*}
\beta p + (1-\beta)z \succ u.
\end{equation*}
But $\beta p + (1-\beta)z \ndom u$: it is enough to notice that for
any $\omega\in\pspace$ and any $x\neq x',z$, it holds that $[\beta p
+ (1-\beta)z - u](\omega,x)=(\beta-1)/|\valuesy|<0$.

Next, if $|\valuesy|=2$ and $|\pspace|\geq 2$, then we can take
$\omega_1\in\pspace$ and define horse lotteries $p,q$ by
$p(\omega_1,x)=1, q(\omega_1,x)=0, p(\omega',x)=q(\omega'x)=\frac{1}{2}$ for
any $\omega'\neq\omega_1$. Then $p\dom q \dom z$, whence $p\succ q
\succ z$. However, for any $\beta\in (0,1)$ it holds that
\begin{equation*}
(\forall
 \omega'\neq\omega_1) \beta p(\omega',x)+ (1-\beta) z(\omega',x) < q(\omega',x),
\end{equation*}
whence $\beta p +(1-\beta) z \ndom q$ and as a consequence $\succ$
is not Archimedean.

Conversely, given $\pspace\coloneqq\{\omega\},\valuesy\coloneqq\{x,z\}$, then it holds
that $p\dom q \iff p(\omega,x)>q(\omega,x)$, and then if $p\succ q
\succ r$ there are $\alpha,\beta\in (0,1)$ such that $\alpha
p(\omega,x)+ (1-\alpha) r(\omega,x) > q(\omega,x) > \beta
p(\omega,x)+(1-\beta) q(\omega,x)$, meaning that $\alpha
p+(1-\alpha) r \succ q \succ \beta p+(1-\beta) r$, and as a
consequence $\succ$ is Archimedean.
\end{proof}

\begin{proof}[{\bf Proof of Proposition~\ref{prop:A3'-wA3}}]
\begin{enumerate}
\item[(i)] Given $p\succ q\succ r,q\ndom r$, we have by~\ref{A2} that $\alpha q + (1-\alpha)p\succ \alpha r + (1-\alpha)p\succ\alpha r +(1-\alpha)q$ for any $\alpha\in(0,1)$. We apply~\ref{A3p} to the latter chain of preferences; note that this requires that $\alpha q + (1-\alpha)p\ndom \alpha r + (1-\alpha)p$, which is true since $q\ndom r$. As a consequence, there is $\gamma\in(0,1)$ such that $\gamma(\alpha q + (1-\alpha)p)+(1-\gamma)(\alpha r +(1-\alpha)q)\succ\alpha r + (1-\alpha)p$, or, equivalently, that $\alpha\gamma(q-r)-(1-\alpha)(1-\gamma)(p-q)\in\cone$, where $\cone$ is the cone associated with $\succ$. Normalizing we find $\beta\in (0,1)$ such that $\beta q +(1-\beta)q=q\succ\beta p +(1-\beta)r$.

\item[(ii)] ($\Rightarrow$) Consider $p,q,r\in\actsy$ such that $p\succ q\succ r,p\ndom q$. By~\ref{wA3} it holds that $\alpha p +(1-\alpha)z\succ q$ for some $\alpha\in(0,1)$. And since by Proposition~\ref{prop:nau} we have that $\alpha p + (1-\alpha)r\succ \alpha p +(1-\alpha)z$, we deduce that $\alpha p + (1-\alpha)r\succ q$.

($\Leftarrow$) This implication trivially follows by a direct
application of~\ref{A3p} to $p,q\in\actsy$ such that $p\succ
q,p\ndom q$, given that $q\succ z$. $\qedhere$
\end{enumerate}
\end{proof}

Next we give a lemma that considers a special translation of a coherent set of desirable gambles. Its proof, needed for Theorem~\ref{prop:rdesirs-archimedean}, is a reformulation of ideas in \cite[Lemma~1]{galaabaatar2013}, on top of which we establish the connection between~\ref{wA3} and~\ref{D0}.

\begin{lemma}\label{lem:trans}
Let $\rdesirs$ be the coherent set of desirable gambles on $\pspace\times\valuesn$ arising from a coherent preference relation $\succ$ on $\actsy\times\actsy$ through~\eqref{eq:RfromC}. Denote by $u$ the uniform act in $\actsy$ defined by $u(\omega,x)\coloneqq1/|\valuesy|$ for all $\omega\in\pspace,x\in\valuesy$. Consider the following translated sets:
\begin{eqnarray*}
\tau(\rdesirs)&\coloneqq&\{g\in\gambles(\pspace\times\valuesn):g=\pi(u)+f, f\in\rdesirs\},\\
\tau(\gambles^+(\pspace\times\valuesn)
)&\coloneqq&\{g\in\gambles(\pspace\times\valuesn):g\gneq\pi(u)\},
\end{eqnarray*}
and the following version of~\ref{D0} adapted to $\tau(\rdesirs)$:
\begin{enumerate}[label=\upshape D0$^\prime$.,ref=\upshape D0$^\prime$]
\item\label{D0'} $g\in \tau(\rdesirs)\setminus \tau(\gambles^+(\pspace\times\valuesn))\Rightarrow(\exists\delta>0)g-\delta\in \tau(\rdesirs)$.
\end{enumerate}
Then it follows that:
\begin{itemize}
\item[(i)] \ref{D0} and~\ref{D0'} are equivalent conditions.
\item[(ii)] $\tau(\rdesirs)=\{\pi(u)+\lambda\pi(p-u):\lambda>0,p\succ u\}$.
\end{itemize}

\end{lemma}

\begin{proof}
Note first that if $g=\pi(u)+f$, $g\in \tau(\rdesirs)$, then $f\in\rdesirs$: in fact if $g\in \tau(\rdesirs)$, then $g=\pi(u)+h$ for some $h\in\rdesirs$, whence $f=h$.

\begin{itemize}

\item[(i)] Let us show that~\ref{D0} and~\ref{D0'} are equivalent conditions.

Consider $g\in \tau(\rdesirs)\setminus
\tau(\gambles^+(\pspace\times\valuesn))$; then $g=\pi(u)+f$,
$f\in\rdesirs\setminus\gambles^+(\pspace\times\valuesn)$.
Applying~\ref{D0}, there is $\delta>0$ such that
$f-\delta\in\rdesirs$, whence $\pi(u)+f-\delta=g-\delta\in
\tau(\rdesirs)$.

Conversely, if $f\in\rdesirs\setminus\gambles^+(\pspace\times\valuesn)$, then $\pi(u)+f\in \tau(\rdesirs)\setminus \tau(\gambles^+(\pspace\times\valuesn))$; applying~\ref{D0'} we get that there is $\delta>0$ such that $\pi(u)+f-\delta \in \tau(\rdesirs)$, whence $f-\delta\in\rdesirs$.

\item[(ii)] 
It is trivial that $\{\pi(u)+\lambda\pi(p-u):\lambda>0,p\succ
u\}\subseteq \tau(\rdesirs)$, so we concentrate on the converse
inclusion.

Consider $\pi(u)+f$. Then there are $\lambda>0$ and $p_1\succ p_2$ such that $f=\lambda\pi(p_1-p_2)$. Remember that $|\valuesy|\geq2$ by assumption (see Remark~\ref{rem:assumptionz}); whence, if we take $\mu\in(0,\frac{1}{\lambda|\valuesy|})$, we obtain that $[u+\mu\lambda(p_1-p_2)](\omega,x)\in[0,1]$ for all $(\omega,x)\in\pspace\times\valuesy$. Since it holds is addition that $\sum_{x\in\valuesy}[p_1-p_2](\omega,x)=0$ for all $\omega\in\pspace$, we have also that $\sum_{x\in\valuesy}[u+\mu\lambda(p_1-p_2)](\omega,x)=\sum_{x\in\valuesy}u(\omega,x)=1$ for all $\omega\in\pspace$. We deduce that $p\coloneqq u+\mu\lambda(p_1-p_2)\in\actsy$.

Given that $p-u=\mu\lambda(p_1-p_2)$ and that $p_1\succ p_2$, we
obtain from Proposition~\ref{prop:nau} that $p\succ u$. And since
$\pi(u)+f=\pi(u)+\frac{1}{\mu}\pi(p-u)$, we deduce that
$\pi(u)+f\in\{\pi(u)+\lambda\pi(p-u):\lambda>0,p\succ u\}$.
$\qedhere$
\end{itemize}
\end{proof}

\begin{proof}[{\bf Proof of Theorem~\ref{prop:rdesirs-archimedean}}]
Coherence follows from Proposition~\ref{prop:rdesirs}. In order to
prove that $\rdesirs$ is a set of strictly desirable gambles, it is
enough to show that~\ref{D0'} holds, thanks to
Lemma~\ref{lem:trans}.

Consider $g\in \tau(\rdesirs)\setminus
\tau(\gambles^+(\pspace\times\valuesn))$. Then, according to Lemma~\ref{lem:trans}, $g=\pi(u)+f$, with $f\coloneqq\lambda\pi(p-u)\in\rdesirs$, $\lambda>0$, and for some $p\in\actsy$ such that $p\succ u, p\ndom u$.

Since $u\succ z$, it follows from~\ref{wA3} that there is some
$\beta\in(0,1)$ such that $\beta p+(1-\beta)z\succ u$, whence
$$\pi(\beta (p-u)-(1-\beta)(u-z))=\frac{\beta f}{\lambda}-(1-\beta)\pi(u)\in\rdesirs.$$ Since $\rdesirs$ is a cone, this
means that $f-\frac{\lambda(1-\beta)}{\beta} \pi(u) \in\rdesirs$ for
some $\beta\in (0,1)$, or, in other words, that $f-\delta\in\rdesirs$, with $\delta\coloneqq\frac{\lambda(1-\beta)}{\beta|\valuesy|}>0$. We deduce that $g-\delta=\pi(u)+f-\delta\in \tau(\rdesirs)$, whence~\ref{D0'} holds and as a consequence $\rdesirs$ is a coherent
set of strictly desirable gambles.
\end{proof}

\begin{proof}[{\bf Proof of Theorem~\ref{prop:rdesirsInv}}]
\begin{enumerate}

\item[(i)] The direct implication is trivial. For the converse, we proceed with a constructive proof.

The case $f=0$ is trivial so we assume that $f\neq0$. We use
$\lambda\coloneqq\sup_{\omega\in\pspace}\sum_{x\in\valuesy}
|f(\omega,x)|$ to normalise $f$ through $f\coloneqq f/\lambda$, so
that each of the previous sums is not greater than one. At this
point for each $\omega\in\pspace$, we need to represent
$f(\omega,\cdot)$ as the difference of two mass functions on
$\values$, say $m_\omega^+ - m_\omega^-$. We proceed as follows:
\begin{eqnarray*}
&&m_\omega^+(x)\coloneqq
\begin{cases}
0&\quad\text{ if }x\neq z\text{ and }f(\omega,x)\leq0,\\
f(\omega,x)&\quad\text{ if }x\neq z\text{ and }f(\omega,x)>0,\\
1-\sum_{x\in\valuesn}m_\omega^+(x)&\quad\text{ if }x=z.
\end{cases}\\
&&m_\omega^-(x)\coloneqq
\begin{cases}
0&\quad\text{ if }x\neq z\text{ and }f(\omega,x)\geq0,\\
-f(\omega,x)&\quad\text{ if }x\neq z\text{ and }f(\omega,x)<0,\\
1-\sum_{x\in\valuesn}m_\omega^-(x)&\quad\text{ if }x=z.
\end{cases}\\
\end{eqnarray*}
It is trivial that the difference $m_\omega^+ - m_\omega^-$
reproduces $f(\omega,x)$ for all $x\neq z$. In the case $x=z$, by
assumption $f(\omega,z)=-\sum_{x\neq z} f(\omega,x)$, and since
\begin{align*}
-\sum_{x\neq z} f(\omega,x) &= \sum_{x\neq z:f(\omega,x)<0}
-f(\omega,x) - \sum_{x\neq z:f(\omega,x)>0} f(\omega,x) \\ &=
(1-\sum_{x\neq z:f(\omega,x)>0} f(\omega,x)) - (1-\sum_{x\neq
z:f(\omega,x)<0} -f(\omega,x)),
\end{align*}
we have also that $f(\omega,z)=m_\omega^+(z)-m_\omega^-(z)$. Since
this holds for all $\omega\in\pspace$, $f$ can be represented as the
difference of two horse lotteries.

Finally, it is enough to multiply $f$ by $\lambda$ to go back to the
original, unnormalised, gamble, which is then an element of
$\sdiff(\actsy)$.

\item[(ii)] To any $\lambda(p-q)\in\sdiff(\actsy)$, we associate the gamble $f\coloneqq\lambda\pi(p-q)$. Let us show that such a map is injective.

Assume that there are $\lambda_1(p_1-q_1)$ and $\lambda_2(p_2-q_2)$
that are both mapped to the same gamble $f$. This makes it trivial
that
$\lambda_1(p_1-q_1)(\omega,x)=f(\omega,x)=\lambda_2(p_2-q_2)(\omega,x)$
for all $\omega\in\pspace$ and $x\neq z$. In the case of $x=z$, it
is enough to write
\begin{equation*}
\lambda_1(p_1-q_1)(\omega,z)=-\lambda_1\sum_{x\in\valuesn}
(p_1-q_1)(\omega,x)=-\lambda_2\sum_{x\in\valuesn}
(p_2-q_2)(\omega,x)=\lambda_2(p_2-q_2)(\omega,z)
\end{equation*}
for every $\omega\in\pspace$. As a consequence,
$\lambda_1(p_1-q_1)=\lambda_2(p_2-q_2)$.

Concerning surjectivity, if we take a gamble
$f\in\gambles(\pspace\times\valuesn)$, we can always extend it to a
gamble $g\in\gambles(\pspace\times\valuesy)$ via
\begin{equation*}
g(x)\coloneqq
\begin{cases}
f(x)\quad&\text{ if }x\neq z,\\
-\sum_{x\in\valuesn}f(x)\quad&\text{ if }x=z.
\end{cases}
\end{equation*}
Applying the first statement, we deduce that there are acts $p,q$
and $\lambda>0$ such that $g=\lambda(p-q)$, whence
$f=\lambda\pi(p-q)$.

\item[(iii)] Let $\cone\coloneqq\pi_2^{-1}(\rdesirs)$. By the second
statement,
$\emptyset\neq\cone\subseteq\sdiff(\actsy)\setminus\{0\}$. To show
that $\cone$ is convex, consider $p_1,p_2,q_1,q_2\in\actsy$ and
$\lambda_1,\lambda_2>0$ such that
$h_1\coloneqq\lambda_1(p_1-q_1),h_2\coloneqq\lambda_2(p_2-q_2)\in\cone$.
By construction, $\gamma_1\pi(h_1)+\gamma_2\pi(h_2)\in\rdesirs$ for
all $\gamma_1,\gamma_2>0$, whence
$\pi_2^{-1}(\gamma_1\pi(h_1)+\gamma_2\pi(h_2))=\gamma_1 h_1+\gamma_2
h_2\in\cone$.

Using Proposition~\ref{prop:one2one}, we obtain that $\rdesirs$ induces the following coherent preference relation:
\begin{equation*}
p\succ q\iff\lambda(p-q)\in\cone\iff\lambda\pi(p-q)\in\rdesirs.
\end{equation*}
To show that $p\succ z$ for all $p\in\actsy,p\neq z$, it is then
enough to consider that $\pi(p-z)$ is a positive gamble.

We are left to show that if $\rdesirs$ is strictly desirable, then
$\succ$ is weakly Archimedean. Let us consider $p\succ q,p\ndom q$.
By~\ref{D0}, there is $\delta>0$ such that
$\pi(p-q)-\delta\in\rdesirs$. Choose $\beta\in(0,1)$ so that
$\pi((1-\beta)p)\leq\delta$. Then $\pi(\beta p -
q)\geq\pi(p-q)-\delta$, whence $\pi(\beta p-q)\in\rdesirs$, which
implies that $\beta p + (1-\beta)z\succ q$. $\qedhere$
\end{enumerate}
\end{proof}

\begin{proof}[{\bf Proof of Theorem~\ref{thm:equiv}}]
As usual, given a coherent preference relation $\succ$, we map it
into the set of gambles defined by $\rdesirs\coloneqq\pi(\cone)$,
where $\cone\coloneqq\{\lambda(p-q):p\succ q,\lambda>0\}$. Set
$\rdesirs$ is coherent as a consequence of
Proposition~\ref{prop:rdesirs}. That the map is injective is trivial
considered that $\pi_2$ (the restriction of $\pi$ to scaled differences of
acts) is invertible. That it is surjective follows from
Theorem~\ref{prop:rdesirsInv}. The equivalence of the weak
Archimedean condition and strict desirability follows from
Theorems~\ref{prop:rdesirs-archimedean}
and~\ref{prop:rdesirsInv}. Finally, the relation and the set of
strictly desirable gambles give rise to the same representation because $\cone$ and $\rdesirs$ induce the same set of linear previsions.
\end{proof}

\begin{proof}[{\bf Proof of Theorem~\ref{thm:ext}}]
\begin{itemize}
\item[(i)] Let us consider the direct implication.
If $p\dom q$, then $p\succ q$ holds as a consequence of
Proposition~\ref{prop:domsucc} and the fact that $z$ is the worst
outcome for $\succ$. On the other hand, if $r\succ s$, then using
Definition~\ref{def:ext} we obtain that
$\pi_1^{-1}(r)\succ\pi_1^{-1}(s)$, which is equivalent to
$\pi_2^{-1}(r-s)\in\cone$, where $\cone$ is the cone induced by
$\succ$. As a consequence, $p-q=\pi_2^{-1}(r-s)\in\cone$, whence
$p\succ q$.

For the converse implication, it is enough to note that $p\dom z$
for all $p\neq z$, so that $p\succ z$; and that, on the other hand,
$r\succ s$ and $\pi(\pi_{1}^{-1}(r)-\pi_{1}^{-1}(s))=r-s$ imply that $\pi_{1}^{-1}(r)\succ\pi_{1}^{-1}(s)$.

\item[(ii)] Remember that any extension has a corresponding convex cone $\cone$
(which is not empty because of the worst outcome). We know by the
first point in the theorem that $\pi_2^{-1}(r-s)\in\cone$ if $r\succ
s$ and moreover that $t-z\in\cone\cup\{0\}$ for any $t\in\actsy$.
Given that $\cone$ is convex, then also
$\lambda_1\pi_2^{-1}(r-s)+\lambda_2(t-z)\in\cone$ for all
$\lambda_1,\lambda_2>0$. In other words, if there are $p,q\in\actsy$
and $\lambda_3>0$ such that
$\lambda_1\pi_2^{-1}(r-s)+\lambda_2(t-z)=\lambda_3(p-q)$, then
$p\succ q$. Equivalently, $p\succ q$ if $\pi(p-q)\geq\lambda(r-s)$,
$r,s\in\actsn,r\succ s,\lambda>0$, where the inequality is obtained
by letting $\lambda\coloneqq\lambda_1/\lambda_3$ and observing that
$\pi(\lambda_2(t-z))=\lambda_2\pi(t)$ spans the entire set
$\{f\geq0\}$ when $t\in\actsy$. This allows us to
reformulate~\eqref{eq:ext} in an equivalent way as follows:
\begin{equation}
(\forall p,q\in\actsy)~p\succ q\text{ whenever }p\dom q\text{ or }\pi(p-q)\geq\lambda(r-s)\text{ for some }r,s\in\actsn,r\succ s, \lambda>0.\label{eq:ext2}
\end{equation}

At this point we consider $\cone'\coloneqq\{\lambda(r-s):r,s\in\actsn,r\succ
s, \lambda>0\}$. Assume for the moment that $\cone'$ is not empty.
Then $\cone'$ is a convex cone that does not contain the origin. If
we regard $\cone'$ as a set of gambles in
$\gambles(\pspace\times\valuesn)$, then $\cone'$ avoids partial loss
because it does not contain zero while being convex. On the other
hand, it is not coherent because it does not contain $\gambles^+$.
We can therefore take its natural extension (see Eq.~\eqref{eq:posi-nex}) so as to obtain the smallest, or
least-committal, coherent set of gambles that contains $\cone'$:
\begin{equation*}
\rdesirs\coloneqq\{f\geq\lambda(r-s):r,s\in\actsn,r\succ s, \lambda>0)\}\cup\gambles^+.
\end{equation*}
Observe that if $\cone'=\emptyset$, then its natural extension is
$\rdesirs=\gambles^+$, which is the vacuous set of gambles;
therefore $\rdesirs$ represents the correct extension also in this
special case.

Given that $\rdesirs$ is a coherent set of gambles on $\pspace\times\valuesn$, we can use Theorem~\ref{prop:rdesirsInv} to create a new coherent preference relation $\succ$ on $\actsy\times\actsy$, with corresponding cone $\cone\coloneqq\pi^{-1}(\rdesirs)$, for which $z$ is the worst outcome. Recalling that $p\succ q\iff\lambda\pi(p-q)\in\rdesirs$ for some $\lambda>0$, it is then trivial to prove that the new relation $\succ$ contains all the preferences expressed by~\eqref{eq:ext2} and only them. As a consequence it is the minimal extension of $\succ$ on $\actsn\times\actsn$ to the worst outcome, because any extension has to express at least those preferences.

\item[(iii)] Assume first of all that $\cone\neq\emptyset$, and let us
prove that $\lpr(f)=0$ for every $f\in\cone$. This will imply that
$\rdesirs$ is not a coherent set of strictly desirable gambles,
because $\cone\subseteq\rdesirs\setminus\gambles^+$.

Since $\cone\subseteq\rdesirs$, it follows from Eq.~\eqref{eq:lprR}
that $\lpr(f)\geq 0$ for every $f\in\cone$. On the other hand, if
$\lpr(f)>0$ for some $f\in\cone$, there should be some $\epsilon>0$
such that $f-\epsilon\in\rdesirs$; but then we would have that for
any $\omega\in\values$, $\sum_{x\in\values}
(f-\epsilon)(\omega,x)=-\epsilon |\values|<0$, while by construction
any gamble $g\in\rdesirs$ satisfies $\sum_{x\in\values}
g(\omega,x)\geq 0$. Thus, $\lpr(f)=0$ and as a consequence,
$\rdesirs$ is not a set of strictly desirable gambles.

The converse implication, where $\cone=\emptyset$, follows from the
fact that in this case $\rdesirs=\gambles^+$, which is a (trivial)
set of strictly desirable gambles. $\qedhere$
\end{itemize}
\end{proof}

\begin{proof}[{\bf Proof of Corollary~\ref{cor:Aext}}]
Denote by $\lpr_1$ the coherent lower prevision associated with
$\rdesirs_1$ and by $\lpr$ the lower prevision induced by
$\rdesirs$. Consider $f\in\cone$. We know by point~(iii) in
Theorem~\ref{thm:ext} that $\lpr(f)=0$. Considered in addition that
$\rdesirs_1$ includes $\rdesirs$, it is strictly desirable and that
$f\notin\gambles^+$, we obtain that $\lpr_1(f)>0$. Written in other
words:
\begin{equation*}
(\forall f \in\cone) \lpr(f)=0<\lpr_1(f).
\end{equation*}

\noindent It follows that $\solp(\lpr_1) \subsetneq \solp(\lpr)$ and
that any linear prevision $\pr_1\geq \lpr_1$ satisfies $\pr_1(f)>0$
for every $f\in\cone$. Fix now $f\in\cone$, consider
$\pr_1\in\solp(\lpr_1)$ and $\pr\in\solp(\lpr)$ such that
$\pr_1(f)=\lpr_1(f)$ and $\pr(f)=\lpr(f)$. Take $\pr_2:=\frac{1}{2}
\pr_1+\frac{1}{2}\pr \in \solp(\lpr)$, and let $\solp_2$ be the
convex hull of $\solp(\lpr_1)\cup \{\pr_2\}$. Its lower envelope
$\lpr_2$ is a lower prevision that lies between $\lpr$ and $\lpr_1$
and satisfies in particular $\lpr_2(f)=\lpr_1(f)/2$. As a
consequence, its associated set of strictly desirable gambles
$\rdesirs_2$ is a proper subset of $\rdesirs_1$ and moreover it
properly includes $\rdesirs$ (because $\lpr_2(g)>0$ for all
$g\in\cone$ by construction while $\lpr(g)=0$).
\end{proof}

\begin{proof}[{\bf Proof of Proposition~\ref{pr:charac-archimedean}}]
\begin{enumerate}
\item[(i)] \begin{itemize}
\item[($\Rightarrow$)] Consider $f,g \in \cone $. Then we can
express them as $f=\lambda_1(p_1-q_1), g=\lambda_2 (p_2-q_2)$, for
$\lambda_1,\lambda_2>0, p_1\succ q_1, p_2\succ q_2$. Assume for the
time being that $\lambda_1=\lambda_2=1$. By~\ref{A2},
\begin{equation*}
 \alpha p_1+ (1-\alpha) p_2 \succ \alpha q_1 +(1-\alpha) p_2 \succ
 \alpha q_1 +(1-\alpha) q_2
\end{equation*}
for any $\alpha \in (0,1)$. Applying now~\ref{A3}, we deduce the
existence of $\beta_1, \beta_2 \in (0,1)$ such that
\begin{align*}
 &\beta_1(\alpha p_1+ (1-\alpha) p_2)+ (1-\beta_1) (\alpha q_1 +(1-\alpha)
 q_2) \succ \alpha q_1 +(1-\alpha) p_2 \text{ and }\\
 & \alpha q_1 +(1-\alpha) p_2 \succ \beta_2(\alpha p_1+ (1-\alpha) p_2)+ (1-\beta_2) (\alpha q_1 +(1-\alpha)
 q_2),
\end{align*}
from which we deduce that
\begin{equation*}
 \alpha \beta_1 (p_1-q_1)- (1-\alpha)(1-\beta_1)(p_2-q_2) \in
 \cone  \text{ and }
 (1-\alpha) (1-\beta_2) (p_2-q_2)- \alpha \beta_2(p_1-q_1) \in
 \cone ;
\end{equation*}
normalizing we find $\gamma_1, \gamma_2 \in (0,1)$ such that
\begin{equation*}
 \gamma_1 f-(1-\gamma_1) g \in \cone  \text{ and } \gamma_2 g-(1-\gamma_2) f \in {\mathcal
 C}.
\end{equation*}
Next, if either $\lambda_1$ or $\lambda_2 \neq 1$, we deduce from
the reasoning above that there is some $\beta\in (0,1)$ such that
\begin{equation*}
 \beta \frac{f}{\lambda_1}-(1-\beta) \frac{g}{\lambda_2} \in
 \cone ,
\end{equation*}
whence, given $k\coloneqq\beta+(1-\beta) \frac{\lambda_1}{\lambda_2}$, it
holds that
\begin{equation*}
 \beta\frac{f}{\lambda_1}\frac{\lambda_1}{k}- (1-\beta)\frac{g}{\lambda_2}\frac{\lambda_1}{k} \in
 \cone ,
\end{equation*}
and $\beta_1\coloneqq\frac{\beta}{k}\in (0,1)$. Analogously, we can find
$\beta_2 \in (0,1)$ such that $\beta_2 g-(1-\beta_2) f\in (0,1)$.

\item[($\Leftarrow$)] Given $p\succ q \succ r$, it follows from
Eq.~\eqref{eq:diffs-hl-prefs} that $f\coloneqq p-q, g\coloneqq
q-r\in \cone $. Applying~\ref{A3}, there are
$\beta_1,\beta_2 \in (0,1)$ such that $\beta_1 f-(1-\beta_1)g \in
\cone , \beta_2 g-(1-\beta_2)f \in \cone $. But taking
into account that
\begin{equation*}
 \beta_1 f-(1-\beta_1)g=\beta_1 p+ (1-\beta_1) r -q\text{ and }
 \beta_2 g-(1-\beta_2)f=q-\beta_2 r- (1-\beta_2) p,
\end{equation*}
we deduce from Proposition~\ref{prop:charC} that $\beta_1 p+
(1-\beta_1) r\succ q \succ \beta_2 r+ (1-\beta_2) p$.
\end{itemize}

\item[(ii)] \begin{itemize}
\item[$(\Rightarrow)$] Consider $f\in \cone , g\in\sdiff$. Then there are $\lambda_1,\lambda_2 >0$, $p\succ q$ and
$r,s \in {\mathcal H}$ such that $f=\lambda_1 (p-q),
g=\lambda_2 (r-s)$. Assume first of all that
$\lambda_1=\lambda_2=1$. Then there is some $\beta\in (0,1)$ such
that $\beta p+(1-\beta) r \succ \beta q + (1-\beta) s$, whence
$\beta f+ (1-\beta) g \in \cone $.

Now, if either $\lambda_1,\lambda_2 \neq 1$, it follows from the
above reasoning that there is some $\beta\in (0,1)$ such that $\beta
\frac{f}{\lambda_1}+(1-\beta) \frac{g}{\lambda_2}\in \cone $,
whence $\beta f+ \frac{(1-\beta) \lambda_1}{\lambda_2} g \in
\cone $, and, by considering the normalising constant $k\coloneqq\beta+ \frac{(1-\beta)
\lambda_1}{\lambda_2}$, it follows that, given
$\beta'\coloneqq \frac{\beta}{k}\in (0,1)$, $\beta' f+(1-\beta')
g\in\cone $.

\item[$(\Leftarrow)$] Consider $p\succ q$ and take $r,s\in{\mathcal
H}$. Then $f\coloneqq p-q\in\cone , g\coloneqq r-s\in\sdiff$, whence
there is some $\beta\in (0,1)$ such that $\beta f+(1-\beta) g=\beta
p+(1-\beta) r- \beta q-(1-\beta) s \in \cone $. Applying
Proposition~\ref{prop:charC}, we deduce that $\beta p+(1-\beta) r
\succ \beta q+ (1-\beta) s$.
\end{itemize}

\item[(iii)] Consider $f,g\in \cone $. Since $-g,-f\in\sdiff$,
we deduce from the second statement that there are $\beta_1,\beta_2
\in (0,1)$ such that $\beta_1 f- (1-\beta_1) g \in \cone $ and
$\beta_2 g- (1-\beta_2) f \in \cone $. Applying the first
statement, we deduce that $\succ$ satisfies~\ref{A3}. This completes
the proof. $\qedhere$
\end{enumerate}
\end{proof}

\begin{proof}[{\bf Proof of Proposition~\ref{pr:equiv-Archs-z}}]
\begin{enumerate}
{\setlength\itemindent{15pt}\item[(i $\Rightarrow$ ii)] This case
has been proved in Proposition~\ref{pr:charac-archimedean}.}
\end{enumerate}

\begin{enumerate}
\item[(i $\Leftarrow$ ii)] Consider $p,q,r,s\in\actsy$ such that
$p\succ q$. Observe first of all that the following is true for some
$\alpha\in(0,1)$:
\begin{equation}
\alpha p + (1-\alpha)z\succ\alpha q + (1-\alpha)s.\label{eq:withz}
\end{equation}
This is trivial if $s=z$ and otherwise it follows from
Proposition~\ref{pr:charac-archimedean} through
$\alpha(p-q)-(1-\alpha)(s-z)\in\cone$. It shows in particular that
the case $r=z$ holds. Therefore assume that $r\neq z$.
Apply~\ref{A2} to obtain that $\alpha p + (1-\alpha)r\succ \alpha
p+(1-\alpha)z$. Combining this with~\eqref{eq:withz}, we conclude
that $\alpha p + (1-\alpha)r\succ\alpha q + (1-\alpha)s$.

\item[(ii $\Rightarrow$ iii)] By using the equivalence between~\ref{wA3} and~\ref{A3p} (see Proposition~\ref{prop:A3'-wA3}), it is trivial that~\ref{A3} implies~\ref{wA3}. Thus it is enough to have that $\lpr(f)>0$ for all $f\in\gambles^+(\pspace\times\values)$; this is equivalent to showing that for every
$(\omega,x)\in\pspace\times\values$ there is some $\delta>0$ such
that $I_{\{(\omega,x)\}}-\delta\in\rdesirs$.

Let $u$ be the uniform act on $\pspace\times\valuesy$, and let $p$
be the horse lottery given by $p\coloneqq u+\frac{1}{|\valuesy|}
(I_{\{(\omega,x)\}}-I_{\{(\omega,z)\}})$. Then it follows that $p\succ u
\succ z$. Applying~\ref{A3}, we deduce that there is some $\alpha\in
(0,1)$ such that $\alpha p +(1-\alpha) z \succ u$, whence
\begin{equation*}
 \pi(\alpha p +(1-\alpha) z - u)=(2\alpha-1) \frac{1}{|\valuesy|}
I_{\{(\omega,x)\}}- (1-\alpha) \frac{1}{|\valuesy|} I_{\{(\omega,x)\}^c} \in
\rdesirs.
\end{equation*}
From this it follows that it must be $2\alpha-1>0$, and applying the
non-negative homogeneity of $\rdesirs$ we deduce that there is some
$\delta>0$ such that $I_{\{(\omega,x)\}}-\delta I_{\{(\omega,x)\}^c}
\in \rdesirs$. Consider now a natural number $N>\delta$. Then
\begin{equation*}
 I_{\{(\omega,x)\}}-\frac{\delta}{N+1}=\left(\frac{1}{N+1} I_{\{(\omega,x)\}}-\frac{\delta}{N+1}
 I_{\{(\omega,x)\}^c}\right)+\left(\frac{N}{N+1}-\frac{\delta}{N+1}\right)I_{\{(\omega,x)\}}\in\rdesirs.
\end{equation*}

\item[(ii $\Leftarrow$ iii)] Thanks to Proposition~\ref{prop:A3'-wA3}, we know that~\ref{wA3} implies~\ref{A3} as long as the latter is applied to non-trivial preferences. Let us show that the positivity condition implies the case with objective preferences involved. In particular, let us consider $p\dom q\succ r$ (the remaining cases are analogous). By the super-additivity and positive homogeneity of coherent lower previsions, and for any $\alpha\in(0,1)$, we have that $\lpr(\alpha\pi(p-q) + (1-\alpha)\pi(r-q)) \geq \alpha \lpr(\pi(p-q)) + (1-\alpha) \lpr(\pi(r-q))$; and this is positive for small enough $1-\alpha$, considered that $\lpr(\pi(p-q))>0$ by assumption. This implies that $\alpha\pi(p-q) + (1-\alpha)\pi(r-q)\in\rdesirs$, or, in other words, that $\alpha p+ (1-\alpha)r\succ\alpha q+ (1-\alpha)q=q$.
\end{enumerate}
\end{proof}

\begin{lemma}\label{le:charac-superset}
Assume $\pspace$ is finite, and let $\succ$ be a coherent
relation on ${\mathcal H}\times{\mathcal H}$. Let $\rdesirs$, from~\eqref{eq:extR}, be
its associated coherent set of desirable gambles, and $\lpr$ its
corresponding coherent lower prevision.
\begin{enumerate}
 \item[(i)] $\rdesirs$ has an open superset $\iff(\forall
 f\in\cone) \upr(f)>0$.
 \item[(ii)] If $\succ$ satisfies~\ref{A3'}, then: $(\exists f\in\cone)\upr(f)=0\Rightarrow
(\forall g\in{\mathcal A}) \lpr(g)=\upr(g)=0$.
\end{enumerate}
\end{lemma}

\begin{proof}
\begin{enumerate}
 \item[(i)] We begin with a direct implication. Let $\rdesirs'$ be an
 open superset of $\rdesirs$, and let $\lpr'$ be its associated
 coherent lower prevision. The inclusion
 $\rdesirs\subseteq\rdesirs'$ implies that
 $\lpr(f)\leq\lpr'(f)\leq\upr'(f)\leq\upr(f)$ for any gamble $f$.
 Moreover, for any $f\in\rdesirs\setminus\gambles^+ \subseteq
 \rdesirs'\setminus\gambles^+$, it holds that $\lpr'(f)>0$ because
 $\rdesirs'$ is a coherent set of strictly desirable gambles, and as
 a consequence also $\upr(f)>0$ for every $f\in\cone$.

 Conversely, taking into account that linear previsions on
 $\gambles(\pspace\times\values)$ are characterized by their restrictions to
 events, the credal set $\solp(\lpr)$ can be regarded as a closed and convex subset of
 the Euclidean space.

 Take $\pr$ in the relative interior of $\solp(\lpr)$.
 Then there is some $\epsilon>0$ such that $B(\pr;\epsilon)\cap
 \aff(\solp(\lpr))\subseteq\solp(\lpr)$, where $\aff(\solp(\lpr))$
 denotes the affine expansion of $\solp(\lpr)$. We are going to
 prove that $\pr(f)>0$ for any $f\in\cone$. Fix $f\in\cone$. Using point~(iii) in Theorem~\ref{thm:ext},
 $\lpr(f)=0<\upr(f)$. Assume ex-absurdo that $\pr(f)=0$.

 Then for any other $\pr_1\neq\pr$ in $B(\pr;\epsilon)$ we can
 consider $\pr_2\coloneqq2\pr-\pr_1\in \aff(\solp(\lpr))\cap
 B(\pr;\epsilon)$: to prove that indeed
 $\pr_2\in B(\pr;\epsilon)$, note that
 $|\pr_2(\{\omega,x\})-\pr(\{\omega,x\})|=|\pr(\{\omega,x\})-\pr_1(\{\omega,x\})|\leq
 \epsilon$. Now, $\pr(f)=\lpr(f)=0$
 implies that $\pr_1(f)=\pr_2(f)=0$, and with this reasoning we
 would conclude that $\pr'(f)=0$ for every $\pr'\in B(\pr;\epsilon)$.

 However, since $f\neq 0$ because $\succ$ is irreflexive, there must
 be some $\omega\in\pspace$ such that $f(\omega,\cdot)$ is not constant on $0$, and then we deduce that there must be
 some $x_i\neq x_j\in\values$ such that
 \begin{equation}\label{eq:ab-dicotomy}
  f(\omega,x_i)\neq f(\omega,x_j) \text{ and } \max\{\pr(\{(\omega,x_i)\}),\pr(\{(\omega,x_j)\})\}>0;
 \end{equation}
 to prove this, note that there must be some $x^*\in\values$ such that $\pr(\{(\omega,x^*)\})>0$, and then if Eq.~\eqref{eq:ab-dicotomy} did not hold we would
 obtain that $f(\omega,x^*)=f(\omega,x')$ for every $x'\neq x^*$; since $\sum_{x\in\values} f(\omega,x)=0$, this would mean that $f(\omega,\cdot)$ is constant on $0$, a contradiction.

 Consider then $x_i,x_j$ in the conditions of Eq.~\eqref{eq:ab-dicotomy}, and assume without loss of generality that $\pr(\{(\omega,x_j)\})>0$. Then
 for $\epsilon'$ small enough it holds that $\pr'\coloneqq\pr+
 \epsilon'
 (I_{\{(\omega,x_i)\}}-I_{\{(\omega,x_j)\}})$ belongs to $B(\pr;\epsilon)$
 and $\pr'(f)\neq\pr(f)$, a contradiction.

 Thus, we conclude that there is some linear prevision
 $\pr\in\solp(\lpr)$ satisfying $\pr(f)>0$ for every $f\in\cone$. It follows that its associated set of strictly desirable gambles
 includes $\rdesirs$.

 \item[(ii)] Assume $\lpr(f)=\upr(f)=0$ for some $f\in\cone$, and
 take $g\neq f$ in ${\mathcal A}$. Since $\succ$ satisfies~\ref{A3'}, we
 deduce from Proposition~\ref{pr:charac-archimedean} that there is
 some $\beta\in(0,1)$ such that $\beta f+ (1-\beta)g \in \cone$, whence $0\leq \lpr(\beta f- (1-\beta) g)\leq \upr(\beta f)+
 \lpr((1-\beta) g)=0-(1-\beta) \lpr(g)\leq 0$, whence also
 $\lpr(g)=0$. Applying the same reasoning to $-g$ we conclude that
 $\lpr$ is constant on $0$ in ${\mathcal A}$. $\qedhere$
\end{enumerate}
\end{proof}

\begin{proof}[{\bf Proof of Theorem~\ref{th:existence-open-superset}}]
Let $\lpr$ be the coherent lower prevision induced by $\rdesirs$,
and assume ex-absurdo that $\rdesirs$ has no open superset. From
Lemma~\ref{le:charac-superset}, it follows that $\lpr(f)=\upr(f)=0$
for every $f\in{\mathcal A}$. Let us show that this last condition cannot
happen.

For every $f\in{\mathcal A}$, $\lpr(f)=0$ implies that for every
$\epsilon>0$ there is some $g_\epsilon\in\cone$ such that
$f+\epsilon\geq g_{\epsilon}$, or, equivalently, that $g_\eps-f\leq\eps$. If there were a pair $(\omega,x)$ such that $g_\eps(\omega,x)-f(\omega,x)<-\eps|\pspace\times\values|$, then it would be impossible to satisfy the constraint that $\sum_{(\omega,x)\in\pspace\times\values}
f(\omega,x)=\sum_{(\omega,x)\in\pspace\times\values} g_{\epsilon}(\omega,x)=0$, taking into account that $g_\eps-f\leq\eps$. Therefore, we obtain that $\|f-g_{\epsilon}\|\leq \epsilon |\pspace\times\values|$, where $||\cdot||$ denotes the supremum norm. As a consequence, ${\mathcal A}=\overline\cone$, where the closure is taken in the topology of uniform convergence.

Let now $\pi'$ denote the projection operator from
$\pspace\times\values$ to $\pspace\times\values'$, where $\values'$
results from dropping the last component in $\values$. It follows
that $\pi'({\mathcal A})=\mathbb{R}^{|\pspace\times\values'|}$. In
addition, since $\cone$ is convex we deduce that $\pi(\cone)$ is
also convex, and also
$\overline{(\pi(\cone))}=\pi({\overline\cone})=\mathbb{R}^{|\pspace\times\values'|}$.
But the only convex set that is dense in an Euclidean space is the
Euclidean space itself \cite[Corollary~1.2.13]{cheridito2013}, whence
$\pi(\cone)=\pi({\mathcal A})$, and therefore $\cone={\mathcal A}$.
But this would mean that for every $f\in{\mathcal A}$ both
$f,-f\in\cone$, which contradicts the coherence of the set
$\rdesirs$.

Thus, it must be $\upr(f)>0$ for every $f\in\cone$ and as a
consequence $\rdesirs$ has an open superset.
\end{proof}

\begin{lemma}\label{le:finite-generating}
Assume that $\pspace$ is finite, and let $k\coloneqq|\values|-1$.
Then the set $\sdiff$ has a finite generating family
$\{g_1,\dots,g_{\ell}\}$ such that, if $\|f\|\leq \epsilon$, then we
can write $f=\sum_{i=1}^{\ell} \lambda_i g_i$ and $\lambda_i\in
[0,k\eps]$ for $i=1,\dots,\ell$.
\end{lemma}

\begin{proof}[{\bf Proof of Lemma~\ref{le:finite-generating}}]
Denote the elements of $\values$ by $x_1,\dots,x_{k+1}$ and consider
the family $\{I_{\{(\omega,x_i)\}}-I_{\{(\omega,x_j)\}}: \omega\in\pspace,
x_i\neq x_j \in \values\}$, and let $f\in\sdiff$. Let us show
that we can express $f$ as a linear combination, with non-negative
coefficients, of the functions in the generating family. To this
end, we can work independently on each $\omega$.

Fix then $\omega\in\pspace$. In order to simplify the notation, in
the following we rename $f_{\omega}\coloneqq f(\omega,\cdot)$ and
drop any other reference to $\omega$. Moreover, we assume without
loss of generality that all the non-negative values of $f_{\omega}$
occur prior to the negative ones; that is, that $f_{\omega}(x_i)\geq
0$ for $i=1,\dots,m$ and $f_{\omega}(x_i)<0$ for $i>m$. Let us show
then that
\begin{equation}
f_{\omega}=\sum_{i=1}^k\left[(I_{\{x_i\}}-I_{\{x_{i+1}\}})\sum_{j=1}^i
f_{\omega}(x_i)\right]\label{eq:f-om}.
\end{equation}

Observe first that the coefficients $\lambda_i\coloneqq\sum_{j=1}^i f_{\omega}(x_i)$ are indeed non-negative, thanks to the assumption that the non-negative values occur prior to the remaining ones, and because $f\in\sdiff$. It trivially holds also that $\lambda_i\leq k\eps$.

Now, to show that Eq.~\eqref{eq:f-om} holds, it is convenient to rewrite it as follows:
\begin{equation}
f_{\omega}=\lambda_1I_{\{x_1\}}+\left[\sum_{i=2}^k
(\lambda_i-\lambda_{i-1})I_{\{x_i\}}\right]-\lambda_kI_{\{x_{k+1}\}}.\label{eq:f-om2}
\end{equation}
\noindent Then it is trivial that~\eqref{eq:f-om2} holds in the case
of $x_1$. In case $i\in\{2,\dots,k\}$, it is enough to observe that
$\lambda_i-\lambda_{i-1}=\sum_{j=1}^i
f_{\omega}(x_i)-\sum_{j=1}^{i-1} f_{\omega}(x_i)=f_{\omega}(x_i)$.
As a consequence, the case of $x_{k+1}$ becomes trivial as well,
given that it must be that $\sum_{i=1}^{k+1}f_{\omega}(x_i)=0$.
\end{proof}

\begin{lemma}\label{lem:open-superset}
Let $\succ$ be a coherent preference relation on ${\mathcal
H}\times{\mathcal H}$ satisfying~\ref{A3'} and with corresponding
cone $\cone$. Assume $\pspace$ is finite, and let $\rdesirs$ be given by Eq.~\eqref{eq:extR}. Then:

\begin{itemize}
\item[(i)] For any gamble
$f\in\cone $, there exists some $\epsilon_f>0$ such that the open
ball $B(f;\epsilon_f)$ satisfies
\begin{equation*}
B(f;\epsilon_f)\cap \sdiff\subseteq\cone .
\end{equation*}
\item[(ii)] $\posi(\widetilde{\rdesirs})\cap {\mathcal
A}=\posi(\widetilde{\rdesirs}\cap \sdiff)$, where
$\widetilde{\rdesirs}:=\cup_{f\in\cone } B(f;\epsilon_f)$.
\end{itemize}
\end{lemma}

\begin{proof}[{\bf Proof of Lemma~\ref{lem:open-superset}}]
\begin{itemize}
\item[(i)] Let $\lpr$ be the coherent lower prevision induced by $\rdesirs$. Let $f\coloneqq p-q$, and $\{g_1,\dots,g_\ell\}$ be the finite generating family of ${\mathcal
A}$, existing by Lemma~\ref{le:finite-generating}. By
Proposition~\ref{pr:charac-archimedean}, for any $i=1,\dots,\ell$,
there is some $\beta_i\in(0,1)$ such that $\beta_i f+(1-\beta_i) g_i
\in\cone $. Moreover, since $f\in\cone $ it follows that for any
$\beta\geq\beta_i$ the gamble $\beta f+(1-\beta) g_i$ also belongs
to $\cone $. As a consequence, given
$\beta^*\coloneqq\max_{i=1,\dots,\ell} \beta_i\in (0,1)$, it holds
that $\beta^* f+(1-\beta^*) g_i \in \cone $ for all
$i=1,\dots,\ell$. Now, since $\cone $ is a convex cone, given
$\alpha_i\geq 0$ such that $\sum_{i=1}^{\ell} \alpha_i>0$,
it holds that
\begin{equation}\label{eq:mixture-included}
 \sum_{i=1}^{\ell} \alpha_i (\beta^* f+(1-\beta^*)
 g_i)=\beta^*(\sum_{i=1}^{\ell} \alpha_i) f+ (1-\beta^*) \sum_{i=1}^{\ell}
 \alpha_i g_i \in \cone.
\end{equation}

Consider $\epsilon<\frac{1-\beta^*}{\beta^* k \ell}$. By
Lemma~\ref{le:finite-generating}, for every $g\in\sdiff$ such
that $\|g-f\|\leq\epsilon$, there are $\lambda_i\in[0,\epsilon k]$
such that $g-f=\sum_{i=1}^{\ell} \lambda_i g_i$, whence
\begin{equation*}
g=f+\sum_{i=1}^{\ell} \lambda_i g_i=f+\frac{1-\beta^*}{\beta^*}
\sum_{i=1}^{\ell} \left(\frac{\lambda_i \beta^*}{1-\beta^*}\right)
g_i;
\end{equation*}
recall that $\|\cdot\|$ refers to the supremum norm. Applying
Eq.~\eqref{eq:mixture-included} with
$\alpha_i\coloneqq\frac{\lambda_i\beta^*}{1-\beta^*}$ and
$\delta\coloneqq\sum_{i=1}^\ell \alpha_i>0$, we deduce that
\begin{equation*}
 \beta^* \delta f + (1-\beta^*) \sum_{i=1}^{\ell} \frac{\lambda_i
\beta^*}{1-\beta^*} g_i= \beta^* \delta f + \sum_{i=1}^{\ell}
\lambda_i \beta^* g_i \in \cone ,
\end{equation*}
and since this set is a cone,
\begin{equation*}
\delta f + \sum_{i=1}^{\ell} \lambda_i g_i \in \cone .
\end{equation*}
Observe that by construction, $0<\delta=\sum_{i=1}^{\ell} \frac{\lambda_i
\beta^*}{1-\beta^*} \leq \frac{\ell k \epsilon \beta^*}{1-\beta^*}<1$,
whence also $(1-\delta) f\in \cone $; and since this set is convex, we deduce that
\begin{equation*}
f + \sum_{i=1}^{\ell} \lambda_i g_i=g \in \cone .
\end{equation*}
We conclude that there is some $\epsilon>0$ such that $g\in\cone$ for every $g$ satisfying $\|g-f\|\leq \epsilon$.

\item[(ii)] We immediately have that
$\posi(\widetilde{\rdesirs}\cap \sdiff)\subseteq
\posi(\widetilde{\rdesirs})\cap\posi({\mathcal
A})=\posi(\widetilde{\rdesirs})\cap \sdiff$, so we only need to
establish the converse inclusion.

Let us begin by showing that $\widetilde{\rdesirs}$ is a
cone. Consider a gamble $g\in\widetilde{\rdesirs}$ and $\lambda>0$.
Then there is some $f\in\cone$ such that $g\in B(f;\eps_f)$, whence
$\lambda g \in B(\lambda f;\lambda\eps_f)$. Since on the other hand
\begin{equation*}
B(\lambda f; \lambda\eps_f)\cap {\mathcal A}=\lambda B(f;\eps_f)\cap
{\mathcal A}=\lambda (B(f;\eps_f)\cap {\mathcal
A})\subseteq\lambda\cone=\cone,
\end{equation*}
we deduce that $\eps_{\lambda f}\geq \lambda\eps_f$, and as a
consequence $\lambda g \in\widetilde{\rdesirs}$.

Consider a gamble $f\in\posi(\widetilde{\rdesirs})\cap \sdiff$. Then
there are
$f_1,\dots,f_n\in\widetilde{\rdesirs},\lambda_1,\dots,\lambda_n>0$
such that $f=\sum_{i=1}^{n} \lambda_i f_i$. Note that since
$\widetilde\rdesirs$ is a cone, we can assume without loss of
generality that $\lambda_i=1$ for every $i=1,\dots,n$. For every
$i=1,\dots,n$, there is some $g_i\in\cone $ and some $h_i\in
B(0;\epsilon_{g_i})$ such that $f_i=g_i+h_i$.

Let us prove that there are $h_i''\in B(0;\epsilon_{g_i})\cap
\sdiff$ such that $f=\sum_{i=1}^{n}(g_i+h_i'')$. Since $g_i+h_i''$
will belong to $\widetilde{\rdesirs}\cap{\mathcal A}$, we will
conclude that $f\in\posi(\widetilde{\rdesirs}\cap {\mathcal A})$,
and the proof will be complete.

Note that we can work independently for each $\omega\in\pspace$; in
addition, our construction of the gambles $h_i''$ shall not depend
on the gambles $g_i$ above. Fix then $\omega\in\pspace$, and let
$\values=\{x_1,\dots,x_k\}$.

Consider the table
\begin{center}
\begin{tabular}{cccc}
  $h_1(\omega,x_1)$ & $h_2(\omega,x_1)$ & $\dots$ & $h_n(\omega,x_1)$ \\
  $h_1(\omega,x_2)$ & $h_2(\omega,x_2)$ & $\dots$ & $h_n(\omega,x_2)$ \\
  $\vdots$ & $\vdots$ & $\dots$ & $\vdots$ \\
  $h_1(\omega,x_k)$ & $h_2(\omega,x_k)$ & $\dots$ & $h_n(\omega,x_k)$ \\
\end{tabular}
\end{center}
By construction, the i-th column is bounded in
$(-\epsilon_{g_i},\epsilon_{g_i})$, and the sum of all the elements
of the table is equal to $0$, because $f\in{\mathcal A}$. Note,
though, that the sum in each row is not necessarily equal to $0$.

Our goal then is to find a table
\begin{center}
\begin{tabular}{cccc}
  $h_1''(\omega,x_1)$ & $h_2''(\omega,x_1)$ & $\dots$ & $h_n''(\omega,x_1)$ \\
  $h_1''(\omega,x_2)$ & $h_2''(\omega,x_2)$ & $\dots$ & $h_n''(\omega,x_2)$ \\
  $\vdots$ & $\vdots$ & $\dots$ & $\vdots$ \\
  $h_1''(\omega,x_k)$ & $h_2''(\omega,x_k)$ & $\dots$ & $h_n''(\omega,x_k)$ \\
\end{tabular}
\end{center}
such that:
\begin{itemize}
 \item[$\circ$] $(\forall
 i=1,\dots,n)\sum_{j=1}^{k} h''_i(\omega,x_j)=0$.
 \item[$\circ$] $(\forall i=1,\dots,n)(\forall
 j=1,\dots,k)h''_i(\omega,x_j)\in(-\epsilon_{g_i},\epsilon_{g_i})$.
 \item[$\circ$] $(\forall j=1,\dots,k)\sum_{i=1}^{n} h_i''(\omega,x_j)=\sum_{i=1}^{n} h_i(\omega,x_j)$.
\end{itemize}
The first of these requirements is done in order to guarantee that
$h_i''$ belongs to ${\mathcal A}$; the second, to prove that $h_i''$
belongs to $B(0;\eps_{g_i})$; and the third, to guarantee that
$f=\sum_{i=1}^{n} (g_i+h_i'')$. Then by repeating the process for
every $\omega$ we end up with the gambles $h_1'',\dots,h_n''$
required above.

Consider $j\in\{1,\dots,n-1\}$, and let $f'_j$ be a gamble on
$\values$ whose supremum norm satisfies
\begin{equation*}
 \|f'_j\|<\sum_{i=j}^{n} \epsilon_{g_j}.
\end{equation*}
We shall show that there exists a gamble $h'_j$ on $\values$
satisfying the following five conditions:
\begin{align}
 &\|f'_j-h_j'\|< \sum_{i=j+1}^{n} \epsilon_{g_i}.\label{eq:aux-cond1}\\
 &\|h_j'\|< \epsilon_{g_j}.\label{eq:aux-cond2}\\
 &\sum_{x\in\values} h_j'(x)=0.\label{eq:aux-cond3}
\end{align}

Once we have established that it is possible to find the gamble
$h'_j$, we proceed as follows:
\begin{itemize}
 \item[$\circ$] We consider the gamble $f'_1$ on $\values$ given by
 $f'_1(x)\coloneqq\sum_{i=1}^{n} h_i(\omega,x)$. It satisfies
 $\|f'_1\|<\sum_{i=1}^{n} \eps_{g_i}$, so we find $h'_1$ satisfying
 Eqs.~\eqref{eq:aux-cond1}--\eqref{eq:aux-cond3} for $j=1$.
 \item[$\circ$] The gamble $f'_2\coloneqq f'_1-h'_1$ satisfies
 $\|f'_2\|<\sum_{i=2}^{n} \eps_{g_i}$ because of Eq.~\eqref{eq:aux-cond1}, so we find $h'_2$ satisfying
 Eqs.~\eqref{eq:aux-cond1}--\eqref{eq:aux-cond3} for $j=2$.
 \item[$\circ$] We repeat the procedure for $f'_j\coloneqq
 f'_{j-1}-h'_{j-1}$ for $j=3,\dots,n-1$.
 \item[$\circ$] Finally, given $f'_n\coloneqq
 f'_{n-1}-h'_{n-1}=f'_1-h'_1-\dots-h'_{n-1}$, we take $h'_n\coloneqq
 f'_{n}$. By~\eqref{eq:aux-cond1}, $\|h'_n\|<\eps_{g_n}$ and
 $f'_1(x)=\sum_{i=1}^{n} h'_i(x)$. Moreover, $$\sum_{x\in\values} h'_n(x)=\sum_{x\in\values}
f'_1(x)-\sum_{i=1}^{n-1}\sum_{x\in\values} h'_i(x)=0-\sum_{i=1}^{n-1} 0=0.$$
\end{itemize}

At that stage it suffices to make $h_i''(\omega,x_j)=h_i'(x_j)$ for
every $j=1,\dots,k$ to obtain the gambles $h_1'',\dots,h''_n$ we
need.

So take $j\in\{1,\dots,n-1\}$, a gamble $f'_j$ with
$\|f'_j\|<\sum_{i=j}^{n} \eps_{g_i}$ and let us establish the
existence of a gamble $h'_j$ satisfying
conditions~\eqref{eq:aux-cond1}--\eqref{eq:aux-cond3} above.
Consider the partition of $\values$ given by:
\begin{align*}
A_0&:=\{x: f'_j(x)=0\}\\
A_1&:=\{x: f'_j(x)>0, f'_j(x)\geq \epsilon_{g_j}\}\\
A_2&:=\{x: f'_j(x)>0, f'_j(x)<\epsilon_{g_j}\}\\
A_3&:=\{x: f'_j(x)<0, f'_j(x)\leq-\epsilon_{g_j}\}\\
A_4&:=\{x: f'_j(x)<0, f'_j(x)> -\epsilon_{g_j}\}.
\end{align*}

We make $h_j'(x)\coloneqq f'_j(x)=0$ for every $x\in A_0$. Then a gamble
$h_j'$ satisfies
Eqs.~\eqref{eq:aux-cond1},~\eqref{eq:aux-cond2} when it satisfies
\begin{align}
&(\forall x \in A_1)\ \ h_j'(x) \in \left(\max\{f'_j(x)-\sum_{i=j+1}^{n}
\epsilon_{g_i},0\},\epsilon_{g_j}\right),\label{eq:constraints-h1-A1}\\
&(\forall x \in A_2)\ \ h_j'(x) \in \left(\max\{f'_j(x)-\sum_{i=j+1}^{n}
\epsilon_{g_i},0\},f'_j(x)\right],\label{eq:constraints-h1-A2}\\
&(\forall x \in A_3)\ \ h_j'(x) \in
\left(-\epsilon_{g_j},\max\{f'_j(x)+\sum_{i=j+1}^{n}
\epsilon_{g_i},0\}\right),\label{eq:constraints-h1-A3}\\
&(\forall x \in A_4)\ \ h_j'(x) \in
\left[f'_j(x),\max\{f'_j(x)+\sum_{i=j+1}^{n}
\epsilon_{g_i},0\}\right).\label{eq:constraints-h1-A4}
\end{align}

Let us detail this for the case of $x\in A_1$; the proof for the
other cases is similar. First of all, since by construction
$h'_j(x)<\epsilon_{g_j}<f'_j(x)$ and also $0<h'_j(x)$, we see that
Eq.~\eqref{eq:aux-cond2} is satisfied. With respect to~\eqref{eq:aux-cond1}, note that
$|f'_j(x)-h_j'(x)|=f'_j(x)-h_j'(x) < \sum_{i=j+1}^{n} \eps_{g_i}
\Leftrightarrow h_j'(x)> f'_j(x)-\sum_{i=j+1}^{n} \eps_{g_i}$.

Now, it follows from interval arithmetic that $\sum_{x\in A_1\cup
A_2} h_j'(x)$ can take any value in
\begin{equation*}
\left(\sum_{x\in A_1\cup A_2} \max\{f'_j(x)-\sum_{i=j+1}^{n}
\epsilon_{g_i},0\}, \sum_{x\in A_1} \eps_{g_j}+ \sum_{x\in A_2}
f'_j(x) \right)\eqqcolon(C,D),
\end{equation*}
and similarly $\sum_{x\in A_3\cup A_4} h_j'(x)$ can take any value
in
\begin{equation*}
\left(-\sum_{x\in A_3} \eps_{g_j}+\sum_{x\in A_4} f'_j(x),
\sum_{x\in A_3\cup A_4} \max\{f'_j(x)+\sum_{i=j+1}^{n}
\epsilon_{g_i},0\} \right)\eqqcolon(E,F),
\end{equation*}
and therefore that $\sum_{x\in \values} h'_j(x)$ can take any value
in the open interval $(C+E,D+F)$. Thus, in order to prove that it is
possible to satisfy the constraints in
Eqs.~\eqref{eq:constraints-h1-A1}--\eqref{eq:constraints-h1-A4} and
also $\sum_{x\in\values} h_j'(x)=0$ (i.e.,
Eq.~\eqref{eq:aux-cond3}), we are going to establish that
$C+E<0<D+F$.

Let us establish that $C+E<0$, or, in other words, that $C<-E$; the
proof of $D+F>0$ is symmetrical.

Let $k:=\sum_{x\in A_1\cup A_2} f'_j(x)>0$. Then, there exists some
natural number $m$ such that $k\in [m\sum_{i=j}^{n}
\epsilon_{g_i},(m+1)\sum_{i=j}^{n} \epsilon_{g_i})$ (because
the ratio $\frac{k}{\sum_{i=j}^{n} \epsilon_{g_i}}$ is positive and
so it must belong to an interval $[m,m+1)$ for some natural number
$m$). Define $k'\coloneqq k-m\sum_{i=j}^{n} \epsilon_{g_i} \in
[0,\sum_{i=j}^{n} \epsilon_{g_i})$. Let us establish that $C$ is
strictly bounded above by
\begin{equation}\label{eq:upper-bound-C}
 m \epsilon_{g_j}+ \max\left\{k'-\sum_{i=j+1}^{n} \epsilon_{g_i},
 0\right\}.
\end{equation}

Note that for every $x\in A_1\cup A_2$, it holds that
$\max\{f'_j(x)-\sum_{i=j+1}^{n} \eps_{g_i},0\}>0$ only when
$f'_j(x)\geq\sum_{i=j+1}^{n}\eps_{g_i}$, so if we define
$B\coloneqq\{x: f'_j(x)\geq\sum_{i=j+1}^{n} \eps_{g_i}\}$ it holds
that
\begin{equation}\label{eq:approx-C}
 C=\begin{cases}
  \sum_{x\in B} (f'_j(x)-\sum_{i=j+1}^{n} \eps_{g_i}) &\text{ if } B\neq
  \emptyset,\\
  0 &\text{ otherwise}.
  \end{cases}
\end{equation}
This means that in order to maximize the value of $C$ we would like
that $\sum_{x\in B} f'_j(x)$ is as close to $k$ as possible. On the
other hand, for a given value of $\sum_{x\in B} f'(x)$, the outer sum in the
right-hand side in~\eqref{eq:approx-C} maximises when $B$ is as
small as possible, which in turn implies that $f'(x)$ is as large as
possible for the elements of $B$. Putting these considerations
together, we conclude that we maximize $C$ for a given value of $k$
when $B$ has $m+1$ elements: in $m$ of them we have
$f'_j(x)=\sum_{i=j}^{n} \eps_{g_i}$ (and so $h'_j=\eps_{g_j})$,
while in the other one $f'_j(x)=k-m\sum_{i=j}^{n} \eps_{g_i}=k'$
(and so $h'_j(x)=\max\{k'-\sum_{i=j+1}^{n} \eps_{g_i},0\}$).

Thus, $C$ is strictly bounded by
Eq.~\eqref{eq:upper-bound-C}; the strict inequality holds because by
assumption $|f'_j(x)|<\sum_{i=j}^{n} \epsilon_{g_i}$ for every $x$,
so we can only get arbitrarily close to the equalities mentioned in
the previous phrase.

On the other hand, since $\sum_{x\in\values} f'_j(x)=0$, it holds
that $k=-\sum_{x\in A_3 \cup A_4} f'_j(x)$. Then $-E=\sum_{x\in A_3}
\eps_{g_j}-\sum_{x\in A_4} f'_j(x)$. For a lower bound on $-E$ we
want this sum to be as far from $k$ as possible; this means that we
want $A_3$ to be as large as possible, and that for any $x\in A_3$
we want to maximise $-f'(x)-\eps_{g_j}$. We do this when
$-f'(x)=\sum_{i=j}^{n} \eps_{g_i}$ as often as possible (again $m$
times, taking into account the assumption of $k\in [m\sum_{i=j}^{n} \epsilon_{g_i},(m+1)\sum_{i=j}^{n}
\epsilon_{g_i})$); while for the remaining $k'=k-m\sum_{i=j}^{n}
\eps_{g_i}$ we must allocate at least $\min\{k',\eps_{g_j}\}$ (we
allocate $\eps_{g_j}$ if we can still be in $A_3$ and $k'$ when we
must be in $A_4$). These considerations imply that $-E$ is bounded
below by
\begin{equation}\label{eq:lower-bound-E}
 m \epsilon_{g_j}+ \min\{k',\epsilon_{g_j}\},
\end{equation}
with the same considerations about the strict inequality as above.

The result then follows by noting that the value
in~\eqref{eq:upper-bound-C} is smaller than the value
in~\eqref{eq:lower-bound-E}. Indeed, there are two possible cases:
\begin{itemize}
 \item[$\bullet$] If $k' \geq \epsilon_{g_j}$, then
 \begin{equation*}
  \min\{k',\epsilon_{g_j}\}=\epsilon_{g_j} \geq \max\left\{k'-\sum_{i=j+1}^{n} \epsilon_{g_i},
  0\right\}, \text{ because } k'\leq \sum_{i=j}^{n} \eps_{g_i};
 \end{equation*}
 \item[$\bullet$] if $k'< \epsilon_{g_j}$, then we obtain
 \begin{equation*}
  \max\left\{k'-\sum_{i=j+1}^{n} \epsilon_{g_i}, 0\right\} \leq
  k'=\min\{k',\epsilon_{g_j}\}.
 \end{equation*}
\end{itemize}

We see then that $C+E<0$. In a similar way we can see that $D+F>0$.
As a consequence, if we consider $h'_j(x)$ in the intervals
determined by
Eqs.~\eqref{eq:constraints-h1-A1}--\eqref{eq:constraints-h1-A4}, the
sum $\sum_{x\in \values}h'_j(x)$ can take any value in the interval
$(C+E,D+F)$, and in particular the value $0$. Thus, there is a
gamble $h'_j$ satisfying
conditions~\eqref{eq:aux-cond1}--\eqref{eq:aux-cond3}. If we now
apply this result in the manner described above, we end up with
gambles $h'_1,\dots,h'_n$ such that $f'_1(x)=\sum_{i=1}^{n}
h'_i(x)$, $\|h'_i\|< \epsilon_{g_i}$ and $\sum_{x\in\values}
h'_i(x)=0$ for all $i=1,\dots,n$. Finally, repeating the process for
every $\omega$ we obtain the gambles $h''_1,\dots,h''_n$ required.
\end{itemize}
\end{proof}

\begin{proof}[\bf Proof of Proposition~\ref{pr:str-arch-open}]
We begin with the direct implication. From
Lemma~\ref{lem:open-superset}, if $\succ$ satisfies~\ref{A3'} then for
every $f\in\cone$ there is some $\epsilon_f>0$ such that
$B(f;\epsilon_f)\cap{\mathcal A}\subseteq\cone$. As a
consequence, $\cone$ is an open subset of ${\mathcal A}$.

Conversely, let $\cone$ be an open convex cone in ${\mathcal A}$,
and let $\succ$ be the coherent preference relation given by $p\succ
q\Leftrightarrow p-q\in\cone$. We show that $\succ$
satisfies~\ref{A3'} by means of the characterization in
Proposition~\ref{pr:charac-archimedean}. Consider $f\in\cone,
g\in{\mathcal A}$. By Lemma~\ref{lem:open-superset}, there is some
$\epsilon_f>0$ such that $B(f;\epsilon_f)\cap{\mathcal
A}\subseteq\cone$. In particular, given $\beta\in(0,1)$ such
that $\|(1-\beta)(f-g)\|<\epsilon_f$, it holds that
\begin{equation*}
 \|f-(\beta f+ (1-\beta) g)\|=\|(1-\beta)(f-g)\|<\epsilon_f
 \Rightarrow \beta f+ (1-\beta) g \in B(f;\epsilon_f)\cap {\mathcal
 A}.
\end{equation*}
As a consequence, $\beta f+(1-\beta)g\in\cone$ and applying
Proposition~\ref{pr:charac-archimedean} we conclude that $\succ$ is
strongly Archimedean.
\end{proof}

\begin{proof}[{\bf Proof of Proposition~\ref{prop:weakest-IP-gmb}}]
$\hat\rdesirs$ is the conglomerable natural extension of
$\rdesirs_\values,\rdesirs|\{\omega\}$ (for all $\omega\in\pspace$),
as it follows from~\cite[Proposition~29]{miranda2012}. Therefore it
is a coherent set and it trivially includes $\rdesirs|\pspace$ by
construction. We are left to show that $\hat\rdesirs$ induces both
$\rdesirs_{\pspace}$ and $\rdesirs_{\values}$. We begin by proving
that the $\pspace$-marginal of $\hat{\rdesirs}$ is
$\rdesirs_{\pspace}$. Consider an $\pspace$-measurable gamble
$f\in\hat{\rdesirs}$. Then there are $g\in\rdesirs_{\pspace}$ and
$h\in\rdesirs|\pspace$ such that $f\geq g+h$. For any
$\omega\in\pspace$ such that $h(\omega,\cdot)\neq 0$, it holds that
\begin{equation*}
0 \leq \max_{x\in\values} h(\omega,x) \leq \max_{x\in\values}
f(\omega,x)-g(\omega,x)=f(\omega)-g(\omega),
\end{equation*}
where in last equality we are using that both $f,g$ are
$\pspace$-measurable. On the other hand, for any $\omega\in\pspace$
such that $h(\omega,\cdot)=0$ it holds that $f(\omega)\geq
g(\omega)$. Thus, $f\geq g$, and since $\rdesirs_{\pspace}$ is a
coherent set of gambles we conclude that also
$f\in\rdesirs_{\pspace}$. The converse inclusion is trivial.

Fix next $\omega\in\pspace$. If
$fI_{\{\omega\}}\in\hat{\rdesirs}$, then there are
$g\in\rdesirs_{\pspace}$ and $h\in\rdesirs|\pspace$ such that
$fI_{\{\omega\}}\geq g+h$. For any $\omega'\neq \omega$, it holds
that
\begin{equation*}
(\forall x\in\values) 0\geq g(\omega')+h(\omega',x),
\end{equation*}
whence it must be $g(\omega')\leq \min_{x\in\values}
-h(\omega',x)\leq 0$. Thus, if we consider instead $g'\coloneqq
gI_{\{\omega\}} \geq g \in\rdesirs_{\pspace}$ and $h'\coloneqq
hI_{\{\omega\}}\in\rdesirs|\pspace$, we also have that
$fI_{\{\omega\}}\geq g'+h'$. Taking into account that
$g'\in\rdesirs_{\pspace}$, it should be $g'(\omega)\geq 0$, whence
$fI_{\{\omega\}}\geq hI_{\{\omega\}}$ and thus
$f\in\rdesirs_{\values}$. Again, the converse inclusion is trivial.
\end{proof}

\begin{proof}[{\bf Proof of Proposition~\ref{pr:charac-linear}}]
\begin{itemize}[leftmargin=*,labelindent=12.5mm]
 \item[$(i)\Leftrightarrow (ii)$] This has been proven in~\cite[Proposition~6]{miranda2010c}.

\item[$(i)\Rightarrow(iii)$] Consider gambles $f,g\notin\rdesirs$. Using~(ii), we have that $\eps/2-f,\eps/2-g\in\rdesirs$ for all $\eps>0$; whence $\eps-(f+g)\in\rdesirs$ for all $\eps>0$. By~\ref{D2} and~\ref{D4}, it follows that $f+g-\eps\notin\rdesirs$ for all $\eps>0$.

\item[$(iii)\Rightarrow(i)$] Using~(ii), assume ex-absurdo that there is a gamble $h$ and $\eps>0$ such that $h,\eps-h\notin\rdesirs$. Using (iii), we obtain that $h+\eps-h+\delta=\eps+\delta\notin\rdesirs$ for all $\delta>0$, which contradicts~\ref{D1}.

 \item[$(i)\Rightarrow(iv)$] Note that
 \begin{equation*}
  \alpha \mu_1+(1-\alpha)\mu_2 \succ \mu_1 I_A+\mu_2 I_{A^c} \Leftrightarrow (1-\alpha) (\mu_2-\mu_1) I_A+ \alpha(\mu_1-\mu_2) I_{A^c} \in \rdesirs,
 \end{equation*}
 and similarly
 \begin{equation*}
  \mu_1 I_A+\mu_2 I_{A^c} \succ \alpha' \mu_1+(1-\alpha')\mu_2 \Leftrightarrow (1-\alpha') (\mu_1-\mu_2) I_A+ \alpha'(\mu_2-\mu_1) I_{A^c} \in \rdesirs.
 \end{equation*}
 Now, if $\pr$ is a linear prevision,
 \begin{equation*}
  \pr((1-\alpha) (\mu_2-\mu_1) I_A+ \alpha(\mu_1-\mu_2) I_{A^c})=(\mu_2-\mu_1)(\pr(A)-\alpha).
 \end{equation*}
 Hence, there are two possibilities: either $\pr(A)<\alpha$, and then $\pr((1-\alpha) (\mu_2-\mu_1) I_A+ \alpha(\mu_1-\mu_2) I_{A^c})>0$ and as a consequence $(1-\alpha) (\mu_2-\mu_1) I_A+ \alpha(\mu_1-\mu_2) I_{A^c}\in\rdesirs$; or
 $\pr(A)\geq \alpha$, whence for every $\alpha'<\alpha$ it holds that
 \begin{equation*}
  \pr((1-\alpha') (\mu_1-\mu_2) I_A+ \alpha'(\mu_2-\mu_1) I_{A^c})=(\mu_1-\mu_2)(\pr(A)-\alpha')>0,
 \end{equation*}
 and as a consequence $(1-\alpha') (\mu_1-\mu_2) I_A+ \alpha'(\mu_2-\mu_1) I_{A^c} \in \rdesirs$.

\item[$(iv)\Rightarrow(i)$] If $\lpr$ is not a linear prevision, then its restriction to events cannot be additive, so there is some event $A$ such that $\lpr(A)<\upr(A)$.
 Then if we take $\mu_1=1, \mu_2=0$ and $\alpha\in(\lpr(A),\upr(A))$,
 \begin{equation*}
  \alpha \mu_1+(1-\alpha)\mu_2-(\mu_1 I_A+\mu_2 I_{A^c})=\alpha-I_A \notin \rdesirs \text{ because } \upr(A)>\alpha,
 \end{equation*}
 and similarly given $\alpha'\in(\lpr(A),\alpha)$,
 \begin{equation*}
  \mu_1 I_A+\mu_2 I_{A^c}-(\alpha' \mu_1+(1-\alpha') \mu_2)=I_A-\alpha' \notin \rdesirs \text{ because } \lpr(A)<\alpha'.
 \end{equation*}
This is a contradiction. $\qedhere$
\end{itemize}
\end{proof}

\begin{proof}[{\bf Proof of Theorem~\ref{pr:domin-by-str-prod}}]
Let us begin by establishing the direct implication. Assume
ex-absurdo that $\lpr$ is not dominated by the strong product
$\lpr_{\pspace}\boxtimes \lpr_{\values}$. Then there are some
$\pr_{\pspace} \in \solp(\lpr_{\pspace})$ and $\pr_{\values} \in
\solp(\lpr_{\values})$ such that $\pr_{\pspace}\times \pr_\values
\notin \solp(\lpr)$. Thus, there is some gamble $f$ in
$\gambles(\pspace\times\values)$ such that $\lpr(f)>
(\pr_{\pspace}\times \pr_\values)(f)$. On the other hand,
\begin{equation*}
 (\pr_{\pspace}\times\pr_{\values})(f)=\pr_{\pspace}(\pr_{\values}(f|\pspace))
 \geq \lpr_{\pspace}(\pr_{\values}(f|\pspace)),
\end{equation*}
and this is a contradiction with Eq.~\eqref{eq:A5}.

Conversely, let us prove first of all that the strong product
$\lpr_{\pspace}\boxtimes\lpr_{\values}$ satisfies Eq.~\eqref{eq:A5}.
Consider $\pr_{\values}\geq\lpr_{\values}$ and a gamble
$f\in\gambles(\pspace\times\values)$. Given the gamble
$\pr_{\values}(f|\pspace)$ on $\pspace$, there is some
$\pr_{\pspace}\geq\lpr_{\pspace}$ such that
\begin{equation*}
 \pr_{\pspace}(\pr_{\values}(f|\pspace))=\lpr_{\pspace}(\pr_{\values}(f|\pspace)),
\end{equation*}
whence $\lpr_{\pspace}\boxtimes\lpr_{\values}(f) \leq
\lpr_{\pspace}(\pr_{\values}(f|\pspace))$, and as a
consequence~\eqref{eq:A5} holds. Now, if
$\lpr\leq\lpr_{\pspace}\boxtimes\lpr_{\values}$, given
$\pr_{\values}\geq\lpr_{\values}$ and a gamble
$f\in\gambles(\pspace\times\values)$,
$\lpr(f)\leq\lpr_{\pspace}\boxtimes\lpr_{\values}(f)
\leq\lpr_{\pspace}(\pr_{\values}(f|\pspace))$. Hence, $\lpr$ also
satisfies~\eqref{eq:A5}.
\end{proof}

Our next result is used in the proof of
Lemma~\ref{le:linear-env-A5}. It extends
\cite[Proposition~25]{cooman2011a}.

\begin{lemma}\label{pr:unique-ind-product}
Let $\lpr$ be a coherent lower prevision on
$\gambles(\pspace\times\values)$ with marginals
$\lpr_{\pspace},\lpr_{\values}$. If
$\lpr_{\values}\eqqcolon\pr_{\values}$ is a linear prevision, then
the only independent product of $\lpr_{\pspace},\pr_{\values}$ is
the strong product $\lpr_{\pspace}\boxtimes\pr_{\values}$, which
coincides moreover with
$\lpr_{\pspace}(\pr_{\values}(\cdot|\pspace))$.
\end{lemma}
\begin{proof}
It has been established in \cite[Proposition~7]{miranda2015b} that
when $\lpr_{\values}$ is a linear prevision, then $\lpr$ is an
independent product if and only if it dominates
$\lpr_{\pspace}(\pr_{\values}(\cdot | \pspace))$. Let us show that
we must actually have the equality.

Given a gamble $f$, it holds that
\begin{align*}
 &\lpr(f-\pr_\values(f|\pspace))\geq
 \lpr_{\pspace}(\pr_{\values}(f-\pr_\values(f|\pspace))|\pspace))=0 \text{ and } \\ &\lpr(-f-\pr_\values(-f|\pspace))\geq
 \lpr_{\pspace}(\pr_{\values}(-f-\pr_\values(-f|\pspace))|\pspace))=0,
\end{align*}
whence
\begin{equation*}
0=\lpr(0)=\lpr((f-\pr_\values(f|\pspace))+(-f-\pr_\values(-f|\pspace)))\geq\lpr(f-\pr_\values(f|\pspace))+\lpr(-f-\pr_\values(-f|\pspace))\geq
0,
\end{equation*}
and as a consequence $\lpr(f-\pr_\values(f|\pspace))=0$ for every $f$.
From this it follows that
$\upr(f-\pr_\values(f|\pspace))=-\lpr(-(f-\pr_\values(f|\pspace)))=-\lpr(-f-\pr_\values(-f|\pspace))=0$
for every $f\in\gambles(\pspace\times\values)$. Now, it is a
consequence of \cite[Theorem~2.6.1(e)]{walley1991} that
\begin{equation*}
 \lpr(f)=\lpr((f-\pr_\values(f|\pspace))+\pr_{\values}(f|\pspace))\leq
 \upr(f-\pr_\values(f|\pspace))+\lpr(\pr_{\values}(f|\pspace))=\lpr_{\pspace}(\pr_{\values}(f|\pspace)),
\end{equation*}
and since we have the double inequality we obtain the equality. We
deduce in particular that the strong product, which is an
independent product, must also coincide with
$\lpr_{\pspace}(\pr_{\values}(\cdot|\pspace))$.
\end{proof}

\begin{proof}[{\bf Proof of Lemma~\ref{le:linear-env-A5}}]
Let us first show that Eq.~\eqref{eq:A4} is equivalent to the following:
\begin{equation}
(\forall g,f\in\gambles)\lpr(g-f)\geq\inf_{\omega\in\pspace}\lpr(g-f^\omega).\label{eq:A4.2}
\end{equation}
To show the direct implication, let $g'\coloneqq g-\inf_{\omega\in\pspace}\lpr(g-f^\omega)$. Considered that $\inf_{\omega\in\pspace}\lpr(g'-f^\omega)=0$, we obtain by Eq.~\eqref{eq:A4} that $\lpr(g'-f)\geq0$, whence $\lpr(g-f)\geq\inf_{\omega\in\pspace}\lpr(g-f^\omega)$. For the converse implication, it is enough to consider that $\lpr(g-f)\geq\inf_{\omega\in\pspace}\lpr(g-f^\omega)\geq0$.

Now, in the particular case that $\lpr$ is a linear prevision $\pr$
(with marginals $\pr_\pspace,\pr_\values$), then Eq.~\eqref{eq:A4.2}
is in turn equivalent to $\pr(-f)\geq\inf_{\omega\in\pspace}
\pr(-f^{\omega})=\inf_{\omega\in\pspace} \pr(-f|\omega)$, where
$\pr(\cdot|\pspace)$ is derived from $\pr_{\values}$ by
Eq.~\eqref{eq:cong-cond-lp}. This is also equivalent to
\begin{equation*}
\pr(f)\geq\inf_{\omega\in\pspace} \pr_{\values}(f|\omega)
\end{equation*}
because the former inequality has to hold for all gambles $f$.

If $\pr$ is an independent product of its marginals, then it is in
particular coherent with $\pr_{\values}(\cdot|\pspace)$, and by
\cite[Theorem~6.5.7]{walley1991}, this is equivalent to
$\pr(f)\geq\inf_{\omega\in\pspace}\pr_{\values}(f|\omega)$.
Conversely, if it is coherent with $\pr_{\values}(\cdot|\pspace)$,
then by \cite[Theorem~6.7.3]{walley1991} it must be
$\pr=\pr_{\pspace}(\pr_{\values}(\cdot|\pspace))$, and then, by
Lemma~\ref{pr:unique-ind-product}, $\pr$ is an independent product.

Let us show now that~\eqref{eq:A4} is preserved by taking lower
envelopes. Let $\solp$ be a family of linear previsions satisfying
Eq.~\eqref{eq:A4}, and take $\lpr\coloneqq\inf\solp$. Consider gambles
$f,g\in\gambles(\pspace\times\values)$. Then for any $\eps>0$ there
is some $\pr\in\solp$ such that
\begin{equation*}
 \lpr(g-f-\inf_{\omega\in\pspace} \lpr(g-f^{\omega}))=\pr(g-f-\inf_{\omega\in\pspace} \lpr(g-f^{\omega}))-\eps \geq
 \pr(g-f-\inf_{\omega\in\pspace} \pr(g-f^{\omega}))-\eps\geq -\eps,
\end{equation*}
taking into account that $\pr\geq\lpr$ and that it
satisfies~\eqref{eq:A4}. Since this holds for any $\eps>0$, we
deduce that $\lpr(g-f-\inf_{\omega\in\pspace}
\lpr(g-f^{\omega}))\geq 0$ and as a consequence $\lpr$ also
satisfies Eq.~\eqref{eq:A4}.
\end{proof}

The next proof is a reformulation of the one given by Galaabaatar and Karni in \cite[Theorem~1]{galaabaatar2013}.

\begin{proof}[{\bf Proof of Proposition~\ref{pr:ind-selection}}]
The direct implication follows from Lemma~\ref{le:linear-env-A5}.
Let us prove the converse.

Let $\exp(\solp(\lpr))$ denote the set of exposed
points\footnote{From \cite[Section~18]{rockafellar1970}, an
\emph{exposed} point of a convex set $\cone$ is a point through which
there is a supporting hyperplane which contains no other point of
$\cone$.} of $\solp(\lpr)$. Since $\pspace,\values$ are finite, it
follows from Straszewicz's theorem \cite[Theorem~18.6]{rockafellar1970}
that $\solp(\lpr)$ is the closed convex hull of $\exp(\solp(\lpr))$,
and as a consequence $\lpr$ is the lower envelope of the set of
exposed points. Let us show that each of these linear previsions satisfies Eq.~\eqref{eq:factorising-linear}.

In order to prove this, we are going to show first that the following
condition is necessary for the factorisation in Eq.~\eqref{eq:factorising-linear}:
\begin{equation*}
(\forall x_1,x_2 \in\values)(\exists k_{x_1,x_2}\in\reals)(\forall \omega\in\pspace) P(\{(\omega,x_1)\})=k_{x_1,x_2} P(\{(\omega,x_2)\}).
\end{equation*}

To prove that this is the case, note that under the above conditions
we get, for any $\omega\in\pspace,x\in\values$, that
\begin{equation}\label{eq:aux-fact}
 P(\{\omega\})=\sum_{x'\in\values} P(\{(\omega,x')\})=\sum_{x'\in\values} k_{x,x'}
P(\{(\omega,x)\})=P(\{(\omega,x)\})(\sum_{x'\in\values} k_{x,x'})
\end{equation}
and
\begin{equation*}
P(\{x\}) k_{x,x'}=(\sum_{\omega'\in\pspace} P(\{(\omega',x)\})) k_{x,x'}=\sum_{\omega'\in\pspace}
(k_{x,x'} P(\{(\omega',x)\}))=\sum_{\omega'\in\pspace} P(\{(\omega',x')\})=P(\{x'\}),
\end{equation*}
whence
\begin{equation*}
\pr(\{x\})(\sum_{x'\in\values} k_{x,x'})=\sum_{x'\in\values} k_{x,x'} \pr(\{x\})=\sum_{x'\in\values}
\pr(\{x'\})=1.
\end{equation*}
This implies that $P(\{x\})>0$ and as a consequence that $(\sum_{x'\in\values} k_{x,x'})=\frac{1}{\pr(\{x\})}$; using
Eq.~\eqref{eq:aux-fact}, we conclude that $\pr(\{(\omega,x)\})=\pr(\{\omega\})\pr(\{x\})$.

Now, assume ex-absurdo that there is an exposed point $\pr$ of
$\solp(\lpr)$ that does not satisfy
Eq.~\eqref{eq:factorising-linear}. Then, the reasoning above implies
that we can find two elements $x_1,x_2$ in $\values$ for which there
is no real $k$ such that
$\pr(\{(\omega,x_1)\})=k\pr(\{(\omega,x_2)\})$ for all $\omega$.

This in turn implies that there are $\omega_1,\omega_2$ in $\pspace$
such that $\pr(\{(\omega_1,x_1)\})=k_1 \pr(\{(\omega_1,x_2)\})$ and
$\pr(\{(\omega_2,x_1)\})=k_2 \pr(\{(\omega_2,x_2)\})$ for different
non-negative real numbers $k_1\neq k_2$. Note that the inequalities
$\pr(\{(\omega_1,x_1)\})\neq k_2 \pr(\{(\omega_1,x_2)\})$ and
$\pr(\{(\omega_2,x_1)\}) \neq k_1 \pr(\{(\omega_2,x_2)\})$ imply
that $\pr(\{(\omega_1,x_2)\}),\pr(\{(\omega_2,x_2)\})\neq 0$ and
that $k_1,k_2>0$. Assume without loss of generality that $k_1>k_2$.
Then for every $\lambda\in\left(\frac{1}{k_1},\frac{1}{k_2}\right)$,
\begin{equation*}
 \pr(\lambda I_{\{x_1\}} |\{\omega_1\})> \pr(I_{\{x_2\}} |\{\omega_1\}) \text{ and } \pr(\lambda I_{\{x_1\}} |\{\omega_2\})< \pr(I_{\{x_2\}} | \{\omega_2\}).
\end{equation*}
In particular, we can choose some $\lambda^*\in
(\frac{1}{k_1},\frac{1}{k_2})$ such that $\pr(\lambda^* I_{\{x_1\}})\neq
\pr(I_{\{x_2\}})$. Consider then the gambles $h_1\coloneqq \lambda^* I_{\{x_1\}}$ and
$h_2\coloneqq I_{\{x_2\}}$.

Assume for instance that $\pr(h_1)>\pr(h_2)$ (if $\pr(h_1)<\pr(h_2)$
then it suffices to reverse the argument below), and let us define
the gamble $f\in\gambles(\pspace\times\values)$ by
\begin{eqnarray*}
 f:\pspace\times\values &\rightarrow &\reals \\
  (\omega,x) &\mapsto & \begin{cases}
  h_1(\omega,x) &\text{ if } \pr(h_1|\{\omega\})\geq \pr(h_2|\{\omega\}), \\
  h_2(\omega,x) &\text{ otherwise}.
  \end{cases}
\end{eqnarray*}
It follows from this definition that
\begin{equation*}
 f^{\omega}=\begin{cases}
  h_1^{\omega} \quad \text{ if } \pr(h_1|\{\omega\})\geq\pr(h_2|\{\omega\}), \\
  h_2^{\omega} \quad \text{ otherwise},
  \end{cases}
\end{equation*}
whence $f^{\omega_1}=h_1^{\omega_1}, f^{\omega_2}=h^{\omega_2}_2$.

Now, since $\pr$ is an exposed point, there is a gamble $g\in\gambles(\values\times\pspace)$ such that
$\pr(g)<\pr'(g)$ for any $\pr'\in \solp(\lpr), \pr'\neq\pr$. Note
that we can assume without loss of generality that $\pr(g)=0$,
simply by redefining $g\coloneqq g-\pr(g)$. In particular, for every
$\lambda>0$ it holds that $\lpr(\lambda g)=\lambda \lpr(g)=0$, and
therefore
\begin{equation*}
(\forall \lambda>0) \lpr(\lambda g+h_1-h_1)=\lpr(\lambda g+h_1-f^{\omega_1})=0,
\end{equation*}
taking into account that $h_1^{\omega_1}=h_1$ because it is an
$\values$-measurable gamble.

Let us show that there is some $\lambda>0$ such that $\lpr(\lambda
g+h_1-h_2)\geq 0$. If this were not the case, then for every natural
number $n$ there would be some $\pr_n\in\solp(\lpr)$ such that
$\pr_n(ng+h_1-h_2)<0$, whence $\pr_n(h_2-h_1)>n \pr_n(g)$. Since
$\solp(\lpr)$ is a compact subset of
$\reals^{|\pspace\times\values|}$, which is a metric space, it is
sequentially compact, so there is a subsequence of $(\pr_n)_{n\in\nats}$ that
converges towards some $\pr'\in\solp(\lpr)$. If $\pr'(g)>0$, then we
have that $\pr'(h_2-h_1)\geq m\pr'(g)$ for all $m\in\nats$, which means that $\pr'(h_2-h_1)=+\infty$, a contradiction. Thus, we must have
$\pr'(g)=0$ or, equivalently, $\pr'=\pr$. But
$\pr(h_2-h_1)=\pr'(h_2-h_1)\geq 0$ by construction, while we have
assumed before that $\pr(h_1)>\pr(h_2)$. This is a contradiction.

We conclude that there is some $\lambda^\star>0$ such that
$\lpr(\lambda^\star g+h_1-h_2)\geq 0$. Then $\lpr(\lambda^\star g+h_1
-f^{\omega})\geq 0$ for all $\omega\in\pspace$, since we have established that
\begin{equation*}
 \lpr(\lambda^\star g+h_1-h_1)\geq 0 \text{ and } \lpr(\lambda^\star g+h_1-h_2)\geq 0.
\end{equation*}
On the other hand,
\begin{align*}
 \lpr(\lambda^\star g+h_1-f)&\leq \pr(\lambda^\star g+h_1-f)=\pr(h_1-f)=\sum_{\omega\in\pspace} \pr(\{\omega\})[\pr(h_1|\{\omega\})-\pr(f^{\omega}|\{\omega\})]\\ &=
 \sum_{\omega\in\pspace: f^{\omega}=h_2} \pr(\{\omega\})[\pr(h_1|\{\omega\})-\pr(h_2|\{\omega\})]\leq
 \pr(\{\omega_2\})[\pr(h_1|\{\omega_2\})-\pr(h_2|\{\omega_2\})]<0,
\end{align*}
since $\pr(\{\omega_2\})>0$ and
$\pr(h_1|\{\omega_2\})<\pr(h_2|\{\omega_2\})$. This is a
contradiction with~\eqref{eq:A4}.
\end{proof}

The next one is a technical result needed in the proof of Theorem~\ref{thm:dom-sp} below.

\begin{lemma}\label{le:closure-finite}
Let $\lpr_1,\lpr_2$ be two coherent lower previsions with a linear
domain $\domain\subseteq\gambles$, and such that
$\lpr_1(f)\geq\lpr_2(f)$ for every $f\in\domain$. Let $\lpr_2'$ be a
coherent lower prevision on $\gambles$ that extends $\lpr_2$. Then
there is a coherent lower prevision $\lpr_1'$ on $\gambles$ that
extends $\lpr_1$ such that $\lpr_1'\geq\lpr_2'$.
\end{lemma}

\begin{proof}
Consider the credal set $\solp(\lpr_1)\coloneqq\{\pr \text{ linear
}: (\forall f\in\domain)\pr(f)\geq\lpr_1(f)\}$, and let
$\solp(\lpr_2')$ be the credal set associated with $\lpr_2'$. Define
$\solp\coloneqq\solp(\lpr_1)\cap\solp(\lpr_2')$. This is the
intersection of two compact and convex sets, so it is also compact
and convex. We start proving that it is not empty.

Let us assume by contradiction that $\solp$ is empty. Then we can
apply the strong separation theorem in
\cite[Appendix~E3]{walley1991} and find a gamble $f$ and a real
number $\mu$ such that $\pr'(f)<\mu$ for every
$\pr'\in\solp(\lpr_1)$ and $\pr'(f)>\mu$ for every
$\pr'\in\solp(\lpr_2')$. This means that the upper envelope
$\unex_{\lpr_1}$ of $\solp(\lpr_1)$ satisfies
$\unex_{\lpr_1}(f)<\lpr_2'(f)$.

Since $\domain$ is a linear set, we can regard it as a convex cone
with the additional property that $\domain=-\domain$. This allows us
to deduce from \cite[Theorem~3.1.4]{walley1991} that
$\unex_{\lpr_1}(f)=\inf\{\upr_1(g)+\mu:f\le
g+\mu,g\in\domain\}<\lpr_2'(f)$; whence there is some $\eps>0$, some
real $\mu$ and some $g\in\domain$ such that $f\leq g+\mu$ and
$\upr_1(g)+\mu\leq \lpr_2'(f)-\eps$. It follows that
\begin{equation*}
 f-g-\mu+\upr_1(g)+\mu\leq \lpr_2'(f)-\eps \Rightarrow f-\lpr_2'(f)-g+\upr_1(g) \leq -\eps \Rightarrow f-\lpr_2'(f)-g+\lpr_2'(g) \leq -\eps,
\end{equation*}
taking into account that
$\upr_1(g)\ge\lpr_1(g)\ge\lpr_2(g)=\lpr_2'(g)$, given that
$g\in\domain$ and that $\lpr_2'$ is an extension of $\lpr_2$. This
contradicts the coherence of $\lpr_2'$. We conclude that $\solp$ is
not empty.

By \cite[Theorem~3.6.1]{walley1991}, $\solp$ induces a coherent
lower prevision $\lpr_1'$ on $\gambles$, and
$\solp(\lpr_1')=\solp=\solp(\lpr_1)\cap\solp(\lpr_2')$. This implies
that $\lpr_1'$ dominates $\lpr_2'$ as well as $\lpr_1$. In order to
show that it is an extension of $\lpr_1$, it remains to prove that
$\lpr_1'$ coincides with $\lpr_1$ on $\domain$. Consider a gamble
$f\in\domain$, and let $\pr_1\in\solp(\lpr_1)$ satisfy
$\pr_1(f)=\lpr_1(f)$. The restriction of $\pr_1$ to $\domain$
dominates $\lpr_2$, so we can reason as before to conclude that
$\solp(\pr_1)\cap\solp(\lpr_2')\neq\emptyset$.

Note that any linear prevision $\pr_1'$ in $\solp(\pr_1)$ must agree
with $\pr_1$ on $\domain$, since this set is negation invariant and
we have $\pr_1'(f)\geq \pr_1(f)$ and also $\pr_1(f)=-\pr_1(-f) \geq
-\pr_1'(-f)=\pr_1'(f)$ for any $f\in\domain$. Thus, we conclude that
there is a linear prevision $\pr_1'$ that coincides with $\pr_1$ in
$\domain$ and dominates $\lpr_2'$ on all gambles. It follows that
$\pr_1'\in\solp(\lpr_1)\cap\solp(\lpr_2')=\solp(\lpr_1')$. As a
consequence, $\lpr_1'(f)\leq \pr_1'(f)=\lpr_1(f)$, and since the
converse inequality holds trivially we conclude that
$\lpr_1'=\lpr_1$ on $\domain$.
\end{proof}

\begin{proof}[{\bf Proof of Theorem~\ref{thm:dom-sp}}]
If $\pspace$ is finite, the result follows from
Proposition~\ref{pr:ind-selection}. Consider next the case of
$\pspace$ infinite, and assume ex-absurdo that there is a gamble $f$
on $\pspace\times\values$ such that
$\lpr(f)<(\lpr_{\pspace}\boxtimes\lpr_{\values})(f)$. By
Lemma~\ref{le:linear-env-A5}, Eq.~\eqref{eq:A4} is preserved by
taking lower envelopes, so we can redefine $\lpr\coloneqq
\min\{\lpr,\lpr_{\pspace}\boxtimes\lpr_{\values}\}$ and obtain a
coherent lower prevision on $\gambles(\pspace\times\values)$ that is
dominated by $\lpr_{\pspace}\boxtimes\lpr_{\values}$ and satisfies
Eq.~\eqref{eq:A4}.

Let us prove that for every natural number $n$ there exists a simple
gamble $f_n$ on $\pspace\times\values$ such that $\norm{f_n-f}\leq
\frac{1}{n}$.

Since $f$ is bounded, there exists some natural number $k$ such that
$-k \leq \inf f\leq \sup f \leq k$. For every $n$, there exists a
finite partition $\partit_n$ of $[-k,k]$ with intervals $I^n_j$ of
length smaller than $\frac{1}{n}$. Let us define the gamble $f_n$ on
$\pspace\times\values$ by
\begin{equation*}
 f_n(\omega,x)\coloneqq\inf I^n_j \quad \text{ if } f(\omega,x)\in I^n_j,
\end{equation*}
for every $(\omega,x)\in\pspace\times\values$. Then since the
partition $\partit_n$ is finite, we deduce that $f_n$ is simple.
Moreover, by construction $|f_n(\omega,x)-f(\omega,x)|\leq
\frac{1}{n}$ for every $(\omega,x)\in\pspace\times\values$.

Let us define next the relation
\begin{equation*}
 \omega \sim \omega' \Leftrightarrow (\forall
 x\in\values)f_n(\omega,x)=f_n(\omega',x).
\end{equation*}
This is trivially an equivalence relation. It has a finite
number of equivalence classes: if $\{z_1,\dots,z_k\}$ denotes the
finite range of $f_n$ and $\values=\{x_1,\dots,x_m\}$, then the
equivalence classes are given by
\begin{equation*}
 \{\omega: (\forall i=1,\dots,m)f_n(\omega,x_i)=z_{t_i}\},
\end{equation*}
where $(t_1,\dots,t_m) \in \{1,\dots,k\}^{m}$ is finite. In other
words, the number of equivalence classes is at most
$\{1,\dots,k\}^{m}$. Let us denote these classes by
$\{B_1,\dots,B_{\ell}\}$.

Next, let us define the set of gambles
$$\domain\coloneqq \{g\in\gambles(\pspace\times\values):(\forall \omega\sim \omega')
g(\omega,x)=g(\omega',x)\}.$$ This is a
linear set of gambles: given $g_1,g_2\in\domain$ and reals
$\lambda_1,\lambda_2$, it holds that
\begin{equation*}
 (\lambda_1 g_1+\lambda_2 g_2)(\omega,x)=\lambda_1 g_1(\omega,x)+\lambda_2 g_2(\omega,x)=
 \lambda_1 g_1(\omega',x)+\lambda_2 g_2(\omega',x)=(\lambda_1 g_1+\lambda_2 g_2)(\omega',x)
\end{equation*}
for every $\omega\sim \omega'$.

The set $\domain$ corresponds to the gambles that are measurable
with respect to the finite partition $\{B_j \times \{x_i\}:
i=1,\dots,m, j=1,\dots,\ell\}$ of $\pspace\times\values$, and by
construction it includes the gamble $f_n$.

Let $\lpr'$ denote the restriction of $\lpr$ to $\domain$. Its
marginals are $\lpr_{\pspace}'$ and $\lpr_{\values}$, where
$\lpr_{\pspace}'$ is the restriction of $\lpr_{\pspace}$ to the
gambles that are measurable with respect to the partition
$\partit\coloneqq \{B_1,\dots,B_{\ell}\}$ of $\pspace$. Moreover,
$\domain$ is in a one-to-one correspondence with
$\gambles(\{1,\dots,\ell\}\times\values)$, so we can regard $\lpr'$
as a coherent lower prevision defined on the set of gambles on a
finite possibility space. Since $\lpr$ satisfies~\eqref{eq:A4}, so
does its restriction $\lpr'$. By Proposition~\ref{pr:ind-selection},
$\lpr'$ is a lower envelope of linear previsions satisfying
Eq.~\eqref{eq:factorising-linear}, so there are two linear
previsions
$\pr_{\pspace}'\geq\lpr_{\pspace}',\pr_{\values}\geq\lpr_{\values}$
such that
\begin{equation*}
 \lpr(f_n)=\lpr'(f_n)=(\pr_{\pspace}'\times\pr_{\values})(f_n).
\end{equation*}

By construction, we can regard $\pr_{\pspace}'$ as a linear
prevision defined on the set
\begin{equation*}
 \{f\in\gambles(\pspace): (\forall
 \omega \sim \omega')f(\omega)=f(\omega')\}
\end{equation*}
that dominates $\lpr_{\pspace}$ on this domain. By
Lemma~\ref{le:closure-finite}, there is a linear prevision
$\pr_{\pspace}\in\solp(\lpr_{\pspace})$ that extends
$\pr_{\pspace}'$. This means that
$\lpr(f_n)=(\pr_{\pspace}'\times\pr_{\values})(f_n)\geq
(\lpr_{\pspace}\boxtimes\lpr_{\values})(f_n)$, and applying
coherence, that $\lpr(f)=\lim_{n\rightarrow\infty} \lpr(f_n)\geq
\lim_{n\rightarrow\infty}
(\lpr_{\pspace}\boxtimes\lpr_{\values})(f_n)=(\lpr_{\pspace}\boxtimes\lpr_{\values})(f)$.
Since $\lpr$ is dominated by the strong product
$\lpr_{\pspace}\boxtimes\lpr_{\values}$, we conclude that
$\lpr(f)=(\lpr_{\pspace}\boxtimes\lpr_{\values})(f)$, a
contradiction.
\end{proof}

\begin{proof}[{\bf Proof of Proposition~\ref{prop:fsd-char}}]
For the direct implication, it suffices to show that
$\widetilde\rdesirs$ satisfies Eq.~\eqref{eq:cA}.

Consider a set $B$ and a gamble $f\in\widetilde\rdesirs|B$. Then
there are $f_1\in\rdesirs_1|B, f_2\in\rdesirs_2|B$ such that
$f=f_1+f_2$. Note that by construction $S(f_1), S(f_2),S(f)\subseteq
B$ and also there is some $\eps>0$ such that
$B(f_2-\eps)\in\rdesirs$. Let us prove that
$I_{S(f)}(f-\frac{\eps}{2})\in\widetilde\rdesirs|B$. Since
$S(f)\subseteq B$, it holds that, for any $\delta>0$,
\begin{multline*}
I_{S(f)}(f-\delta)=B(f-\delta)+ \delta I_{B\setminus S(f)}=B(f_1+f_2-\delta) + \delta I_{B\setminus S(f)}=B(f_2-\delta) +
B f_1+\delta I_{B\setminus S(f)};
\end{multline*}
the gamble $Bf_1+\delta I_{B\setminus S(f)}$ belongs to
$\rdesirs_1|B$ for any $\delta>0$; on the other hand,
$B(f_2-\delta)$ belongs to $\rdesirs_2|B$ for any
$\delta\in(0,\eps)$, and in particular for $\delta=\frac{\eps}{2}$.
As a consequence,
$I_{S(f)}(f-\frac{\eps}{2})\in\widetilde\rdesirs|B$.

This implies that $\widetilde\rdesirs|B$ satisfies
Eq.~\eqref{eq:cA}, and as a consequence so does
$\widetilde\rdesirs$. Therefore, its natural extension $\rdesirs$ is
a fully strictly desirable set of gambles.

As for the converse implication, by construction,
$\widetilde\rdesirs|B\subseteq\rdesirs$ for every set $B$, whence
$\rdesirs$ includes the natural extension of $\widetilde\rdesirs$.
If $\rdesirs$ is a fully strictly desirable set, then it is the
natural extension of a set $\rdesirs'$ satisfying Eq.~\eqref{eq:cA}.
Given a gamble $f\in\rdesirs'$, it follows from~\eqref{eq:cA} that
$f\in\rdesirs_2|S(f)\subseteq\widetilde\rdesirs|S(f)$, and hence
$\rdesirs'\subseteq\widetilde\rdesirs\subseteq\rdesirs$. As a
consequence, $\rdesirs$ is the natural extension of
$\widetilde\rdesirs$.
\end{proof}

\begin{proof}[{\bf Proof of Theorem~\ref{pr:equal-expressivity}}]
First of all, we are going to show that $\rdesirs$ is
coherent. For this, it suffices to show that the set $\cup_{i\in
I}\rdesirs_2|B_i$ avoids partial loss. Since for every $i\in I$ the
set $\rdesirs_2|B_i$ is a convex cone, given that
$\lpr_i(\cdot|B_i)$ is coherent, the set $\cup_{i\in
I}\rdesirs_2|B_i$ incurs partial loss if and only if there is some
finite $J\subseteq I$ and $g_j \in\rdesirs_2|B_j$ for every $j\in J$
such that $\sum_{j\in J} g_j\leq 0$. In that case,
\begin{equation*}
 0\geq \sum_{j\in J} g_j=\sum_{j \in J} B_j(f_j-(\lpr_j(f_j|B_j)-\eps_j)),
\end{equation*}
for some $f_j \in\gambles$ and some $\eps_j>0$. As a consequence,
\begin{equation*}
 \sup_{\omega \in \cup_{j\in J} B_{j}} \left[\sum_{j\in
J}B_j(f_j-\lpr(f_j|B_j))\right] (\omega)<0,
\end{equation*}
a contradiction with Eq.~\eqref{eq:williams-apl}.

Let us show next that $\widetilde\rdesirs|B_i$ satisfies
Eq.~\eqref{eq:cA}. Consider $f\in\widetilde\rdesirs|B_i$. Then there
are $f_1\in\rdesirs_1|B_i, f_2\in\rdesirs_2|B_i$ such that
$f=f_1+f_2$. Note that by construction $S(f_1), S(f_2),S(f)\subseteq
B_i$. Thus, for any $\delta>0$,
\begin{multline*}
I_{S(f)}(f-\delta)=B_i(f-\delta)+ \delta I_{B_i\setminus S(f)}=B_i(f_1+f_2-\delta) + \delta I_{B_i\setminus S(f)}=B_i(f_2-\delta) +
B_i f_1+\delta I_{B_i\setminus S(f)};
\end{multline*}
the gamble $B_if_1+\delta I_{B_i\setminus S(f)}$ belongs to
$\rdesirs_1|B_i$ for any $\delta>0$; on the other hand, for $f_2$
there exists some $g_2\in\gambles$ and some $\eps>0$ such that
$f_2=B_i(g_2-(\lpr_i(g_2|B_i)-\eps))$, whence if we take $\delta \in
(0,\eps)$ we obtain that
$B_i(f_2-\delta)=B_i(g_2-(\lpr_i(g_2|B_i)-(\eps-\delta)))\in\rdesirs_2|B_i$.
As a consequence, $I_{S(f)}(f-\delta)\in\widetilde\rdesirs|B_i$.

We have obtained that the set $\widetilde\rdesirs|B_i$ satisfies
Eq.~\eqref{eq:cA} for all $i\in I$, and as a consequence so does the
union $\cup_{i\in I} \widetilde\rdesirs|B_i$. Since $\rdesirs$ is
the natural extension of $\cup_{i \in I} \widetilde\rdesirs|B_i$, we
deduce that it is a fully strictly desirable set.
\end{proof}

\begin{proof}[{\bf Proof of Proposition~\ref{pr:CA-finite}}]
Assume that a set of gambles $\rdesirs'$ satisfies~\eqref{eq:cA},
and let $\rdesirs$ be its natural extension. Then by definition of
the natural extension, for every $f\in\rdesirs$ there are gambles
$f_1,\dots,f_n\in\rdesirs'$ and positive real numbers
$\lambda_1,\dots,\lambda_n$ such that
\begin{equation*}
 f=I_{S(f)} f\geq \sum_{j=1}^{n} \lambda_j f_j.\label{eq:If_natex}
\end{equation*}
By~\eqref{eq:cA}, for every $j\in\{1,\dots,n\}$ there is some $\epsilon_j$ such that
$I_{S(f_j)}(f_j-\epsilon_j)\in\rdesirs'$. Let
$$\epsilon\coloneqq\min\{\min_j\lambda_j \epsilon_j, \min_{\omega \in
S(f)\setminus \cup_j S(f_j)} f(\omega)\}>0.$$ Note that $f>0$ in
$I_{S(f)\setminus \cup_j S(f_j)}$, provided that $S(f)\setminus
\cup_j S(f_j)\neq\emptyset$, since $f\neq0$ in $S(f)$, while $f\geq
\sum_{j=1}^{n} \lambda_j f_j=0$ outside $\cup_j S(f_j)$; moreover,
since $S(f)$ is finite we deduce that $\min_{\omega \in
S(f)\setminus \cup_j S(f_j)} f(\omega)>0$. Taking into account that
$\lambda_j\eps_j>0$ for all $j$, we conclude that indeed $\eps>0$.
Moreover, we deduce also that
\begin{equation}
I_{S(f)\setminus \cup_j S(f_j)}(f-\eps)\geq0;\label{eq:case1}
\end{equation}
to prove this, it is enough to reconsider that $f>0$ in $S(f)\setminus \cup_j S(f_j)$, and that $\eps\leq\min_{\omega \in S(f)\setminus \cup_j S(f_j)} f(\omega)$.

Now,
\begin{align}
 I_{S(f)}(f-\epsilon) &\geq (\sum_{j=1}^{n} \lambda_j f_j)- \epsilon I_{S(f)} \geq
 \sum_{j=1}^{n} \lambda_j I_{S(f_j)} f_j -\epsilon I_{\cup_j
 S(f_j)}-\epsilon I_{S(f)\setminus \cup_j S(f_j)} \notag\\ & \geq
 \sum_{j=1}^{n} (\lambda_j I_{S(f_j)} f_j -\epsilon I_{S(f_j)})-\epsilon I_{S(f)\setminus \cup_j S(f_j)}\notag\\ &\geq
 \sum_{j=1}^{n} \lambda_j I_{S(f_j)}(f_j-\epsilon_j)-\epsilon I_{S(f)\setminus \cup_j S(f_j)},\label{eq:case2}
\end{align}
where in the last passage we have used that $\eps\leq\min_j\lambda_j \epsilon_j$.

Consider two cases:
\begin{itemize}
\item If $\omega \in S(f)\setminus \cup_j S(f_j)$, then $I_{S(f)}(f-\eps)\geq0=\sum_{j=1}^{n} \lambda_j I_{S(f_j)}(f_j-\epsilon_j)$, where the inequality follows from~\eqref{eq:case1}.
\item If $\omega\in\cup_j S(f_j)$, then $I_{S(f)}(f-\eps)\geq\sum_{j=1}^{n} \lambda_j I_{S(f_j)}(f_j-\epsilon_j)$, as it follows from~\eqref{eq:case2}.
\end{itemize}
As a consequence, $I_{S(f)}(f-\eps)\geq\sum_{j=1}^{n} \lambda_j
I_{S(f_j)}(f_j-\epsilon_j)$, and since $I_{S(f_j)}
(f_j-\epsilon_j)\in\rdesirs'$, we deduce that $I_{S(f)}(f-\epsilon)$
belongs to $\rdesirs$. Thus, $\rdesirs$ also
satisfies~\eqref{eq:cA}.

We conclude that $\rdesirs$ is fully strictly desirable if and only if it satisfies~\eqref{eq:cA}.
\end{proof}

\bibliographystyle{plain}

\end{document}